\documentclass{article}

\usepackage{times}
\usepackage{graphicx} % more modern
\usepackage{subfigure} 

\usepackage{algorithm}
\usepackage{algorithmic}

\usepackage{hyperref}

\usepackage{qiangstyle}
\usepackage[accepted]{icml2016}

\icmltitlerunning{A Kernelized Stein Discrepancy for Goodness-of-fit Tests}% and Model Evaluation}

\begin{document} 

\twocolumn[
\icmltitle{A Kernelized Stein Discrepancy for Goodness-of-fit Tests}

% unless you've provided the [accepted] option to the icml2015
\icmlauthor{Qiang Liu}{qliu@cs.dartmouth.edu}
\icmladdress{Computer Science, Dartmouth College, NH, 03755}
\icmlauthor{Jason D. Lee}{jasondlee88@eecs.berkeley.edu}
%\icmladdress{Department of Electrical Engineering and Computer Science  University of California, Berkeley, CA 94709}
\icmlauthor{Michael Jordan}{jordan@cs.berkeley.edu}
\icmladdress{Department of Electrical Engineering and Computer Science  University of California, Berkeley, CA 94709}

\icmlkeywords{RKHS, Stein's method, goodness-of-fit tests}

\vskip 0.3in
]

%\documentclass{article} % For LaTeX2e
%\usepackage{nips14submit_e,times}
%\documentstyle[nips13submit_09,times,art10]{article} % For LaTeX 2.09

%\usepackage{qiangstyle}
%\input{newcommands.tex}

%\title{A Kernelized Stein Discrepancy for Goodness-of-fit  tests and Model Evaluation}
%in Complex Probabilistic Models}
% The \author macro works with any number of authors. There are two commands
% used to separate the names and addresses of multiple authors: \And and \AND.
%
% Using \And between authors leaves it to \LaTeX{} to determine where to break
% the lines. Using \AND forces a linebreak at that point. So, if \LaTeX{}
% puts 3 of 4 authors names on the first line, and the last on the second
% line, try using \AND instead of \And before the third author name.

\newcommand{\fix}{\marginpar{FIX}}
\newcommand{\new}{\marginpar{NEW}}
\newcommand{\jnote}[1]{{[\color{red}Jason: #1]}}
%\nipsfinalcopy % Uncomment for camera-ready version

%\begin{document}

%\maketitle

%\hspace{-.2\baselineskip}
\begin{abstract}
%We propose a new discrepancy measure between two probability distributions, and apply it for goodness-of-fit tests. 
We derive a new discrepancy statistic for measuring differences between two probability distributions based on combining Stein's identity with the reproducing kernel Hilbert space theory.  
We apply our result to test how well a probabilistic model fits a set of observations, %the goodness-of-fit of probabilistic models over observed data, 
and derive a new class of powerful goodness-of-fit tests that are widely applicable for complex and high dimensional distributions, even for those with computationally intractable normalization constants. 
Both theoretical and empirical properties of our methods are studied thoroughly. 
%The effective of 
%general multivariate distributions, even for complex and high dimensional distributions with intractable normalization constants. 
%Our new method is widely applicable,   
%Our new method 
%by applying Stein's method on reproducing kernel Hilbert spaces. 
%Our new discrepancy measure can be empirical estimated using $U$-statistics, and can be efficiently computed even for complex models with intractable likelihood. %normalization constants.  We apply our measure for goodness-of-fit tests and demonstrate its empirical and theoretical properties. 
%\red{[]}
%and does not dependent on the normalization constant of which 
%based on a novel combination of Stein's method and 
%We propose a new kernelized version of Fisher divergence based on Stein's identity. It can be conveniently approximated and as as result provide a generic goodness-of-fit test.  Our method is particularly attractive for very complicated models that has been widely used in machine learning, such as deep generative models and hidden variable models, where evaluation of the likelihood is computationally intractable.
%as these used in deep generative learning, for which the model evaluation has been a major problem due to the computational intractability of calculating the testing likelihood. 
\end{abstract}

\section{Introduction}
\myempty{
 Evaluating the goodness-of-fit of models over observed data is a fundamental task in machine learning and statistics. %, but remains large unexplored. 
 Although there has been a rich statistical literature in goodness-of-fit tests \citep[e.g.,][]{d1986goodness, lehmann2006testing}, the classical methods, such as 
chi-square test and Kolmogorov-Smirnov test, only work for univariate or simple distributions and are not applicable for complex and high dimensional models 
widely used in modern learning techniques. 
%that are widely used in practice. 
In machine learning practice, likelihood ratios, or Bayesian factors are widely used for evaluating or comparing models, 
%Unfortunately, a general, widely applicable goodness-of-fit test is still missing. 
%The classical goodness-of-fit tests, such as $\chi^2$ tests and Kolmogorov-Smirnov test, mostly only works for univariate or simple cases. 
%Likelihood ratio based approaches provides a generic principle, but is  often computationally intractable when applied on complex and high dimensional models  
%increasingly involve complex probabilistic models with computationally intractable likelihoods, 
%such as large graphical models, hidden variables models and deep generative models \citep{koller2009probabilistic, ruslan}. 
}

 Evaluating the goodness-of-fit of models over observed data is a fundamental task in machine learning and statistics. 
Traditional approaches often involve calculating or comparing the likelihoods or cumulative distribution functions (CDF) of the models. 
Unfortunately, modern learning techniques increasingly involve complex probabilistic models with computationally intractable likelihoods or CDFs, such as large graphical models, hidden variables models and deep generative models \citep{koller2009probabilistic, ruslan}. 
Although Markov chain Monte Carlo (MCMC) or variational methods can be used 
to approximate the likelihood, their approximation errors are often large and hard to estimate, making it difficult to give results with calibrated statistical significance. 
In fact, it is often a \#P-complete problem to calculate or even approximate likelihoods for graphical models \citep[e.g,][]{chandrasekaran2008complexity}, making  likelihood-based approaches fundamentally infeasible. 
%likelihood calculation is \#P-complete for graphical models, 
%given the fundamental hardness of likelihood calculation, e.g., when 
%evaluating the normalization $Z$ has been a major bottleneck for model selection using likelihood, or Bayesian factors \citep[e.g.,][]{gelfand1994bayesian, salakhutdinov2008quantitative}, and still remains to be an active research problem \citep{burda2015accurate, liuuai15, grosse2015sandwiching}. 

We propose a \emph{likelihood-free} approach for model evaluation with guaranteed statistical significance. %without calculating the likelihood. 
In particular, we consider the setting of goodness-of-fit testing, where we test whether a given sample $\{x_i\} \sim p(x)$ is drawn from a given distribution $q(x)$, meaning $H_0: p=q$.
Our method is based on a new discrepancy measure between distributions that can be empirically estimated using $U$-statistics, and depends on $q$ only through its score function $\score_q = \nabla_x \log q(x)$; 
this score function does %can often be calculated efficiently be
not depend on the normalization constant in $q(x)$, and can often be calculated efficiently even when the likelihood is intractable. 
This allows us to apply our methods to complex and high dimensional models on which the likelihood-based methods, or other traditional goodness-of-fit tests, such as 
 $\chi^2$-test and Kolmogorov-Smirnov test, can not be applied. 

\paragraph{Main Idea}
Our method is motivated by Stein's method and the reproducing kernel Hilbert space (RKHS) theory. 
Stein's method \citep{stein1972} is a general theoretical tool for obtaining bounds on distances between distributions. 
%Stein's identity is a foundation of Stein's method \citep{stein1972} which provides a general theoretical tool for obtaining bounds on distances between distributions. 
%Stein's method starts with
%It was originally motivated with the elementary observation that a distribution $p(x)$ is standard normal if and only if 
%\begin{align}
%\E_p[-x f(x) + \nabla_x f(x)] = 0,  %\text{for all continuously differentiable functions with $\int \phi(x) f'(x) < \infty$}.  
%\end{align}
%for all  continuously differentiable functions $f(x)$, and therefore the deviates of $\E_p[-x f(x) + \nabla f(x)]$ from zero quantifies the distance between $p(x)$ and standard normal distribution. 
%The identity can be generalized general continuous distributions with smooth densities is 
%Generally, for each distribution $q$ we need an operator $\stein_q$, called \emph{Stein operator}, which ``characterizes" distribution $q$ in that $\E_p[\stein_q f(x)] = 0$ for all $f$ if and only if $p = q$.  In particular, a general Stein operator for continuous distributions with smooth densities is 
%There are different ways to find Stein operators \citep{reinert2005three}.  
%We focus on the so called \emph{the density approach} \citep{stein2004use}, which works for general continuous distributions; it reElies on basic identities of the following type
%It relies on the Stein's identity \citep{stein2004use}, which claims that two smooth densities $p(x)$ and $q(x)$ on $\R$ are identical if and only if 
Roughly speaking, it relies on the basic fact that two smooth densities $p(x)$ and $q(x)$ supported on $\R$ are identical if and only if 
\begin{align}
\label{equ:steq1}
\E_p[\score_q(x) f(x)  + \nabla_x f(x)] = 0
%\text{for all general smooth functions $f(x)$, }
%\makecell[l]{\text{for general smooth}  \\ \text{ functions $f(x)$}}%
%~~~ \score_p(x) = \nabla_x \log p(x) = \frac{\nabla_x p(x)}{p(x)}; 
%\E_p[f(x) \score_q(x) + \nabla_x f(x)] = 0, ~~~~~~~\text{where}~~~ \score_p(x) = \nabla_x \log p(x) = \frac{\nabla_x p(x)}{p(x)}; 
% ~~~~~~~~\iff ~~~~~~~~~ 
\end{align}
for smooth functions $f(x)$ with proper zero-boundary conditions,  
where $\score_q(x) = \nabla_x \log q(x) ={\nabla_x q(x)}/{q(x)}$ is called the (Stein) \emph{score function} of $q(x)$;
when $p=q$, \eqref{equ:steq1} is known as Stein's identity \citep[e.g.,][]{stein2004use}, and can be proved using integration by parts.
% this identity can be proved by integration by parts. 
%The $\score_p(x)$ is called the \emph{score function} of $p(x)$. 
%One method that works for generic continuous distribution, called \emph{the density approach}, relies on following identity 
%$f(-\infty) p(-\infty) = f(+\infty) p(+\infty)  = 0$. 
%We first give a high level overview of the method, and develop the details later. 
%Give a distribution with positive density $q(x)$ on $\RR$, its score function is defined to be $\score_q(x)  = \frac{\nabla_x q(x)}{q(x)} = \nabla_x \log q(x)$. Stein's identity says  we have $p = q$ if and only 
%\begin{align}
%\E_p[f(x) \score_q(x) + f'(x)] = 0, ~~\forall ~~f(x)  % ~~~~~~~~\iff ~~~~~~~~~ 
%\end{align}
%where $f(x)$ is taken to be any set of functions that is dense in the set of continuous differentiable function that satisfies a zero boundary condition $\int_{x \to \pm \infty}f(x)p(x) = 0$. 
%This motivates 
As a result, one can define a Stein discrepancy measure\footnote{Our definition is the square of the typical definition of Stein discrepancy, as such that in \citet{gorham2015measuring}.} between $p$ and $q$ via 
\begin{align}\label{equ:sf}
\S(p,q)  = \max_{f\in \mathcal F} \big(\E_p[ \score_q(x) f(x) + 
\nabla_x f(x)]\big)^2, 
\end{align}
where $\mathcal F$ is a set of smooth functions that satisfies \eqref{equ:steq1} and is also rich enough to ensure $\S(p,q) >0$ whenever $p \neq q$. 
%Note that our definition is sligdifferent
The problem, however, is that $\S(p,q)$ is often computationally intractable because it requires a difficult variational optimization. % in the function space. 
As a result, $\S(p,q)$ is rarely used in practical machine learning;
perhaps the only exception is the recent work of \citet{gorham2015measuring} who obtained a computationally tractable form by enforcing smoothness constraints only on a finite number of points, turning the  optimization into a linear programming. 
% and is generally intractable.  
%that is dense %is any set of functions that is dense in the set of continuous differentiable functions in terms of $L_{\infty}$ norm. 
%Note that $\S(p,q)$ is always non-negative, and equals zero iff $p = q$. 

We propose a simpler method for obtaining computational tractable Stein discrepancy $\S(p,q) $ by taking $\mathcal F$ to be a ball in a reproducing kernel Hilbert space (RKHS)
%Our key observation is that a computational tractable $\S(p,q) $ can be found by taking $\HH$ to the unit ball of a reproducing kernel Hilbert space (RKHS)
associated with a smooth positive definite kernel $k(x,x')$. 
In particular, we show that in this case
\begin{align}
 \S(p,q)  =  \E_{x,x'\sim p } \big [  u_q(x,x') \big], 
%& = \max_{f\in \HH} \E_p[f(x) \score_q(x) + f'(x)]  \\
%& = \E_{x,x'} \bigg [   \score_q(x) k(x,x') \score_q(x') + \frac{1}{2}(\score_q(x)+\score_q(x'))\nabla_{(x+x')} k(x, x')  + \nabla_{x,x'} k(x, x')  \bigg] \\
%& \overset{def}{=} \E_{x,x'} \bigg [  u(x,x') \bigg]. 
\label{equ:Dhh}
\end{align}
where $x,x'$ are i.i.d. random variables drawn from $p$ and
$u_q(x,x')$ is a function (defined in Theorem~\ref{thm:kxx}) that depends on $q$ only through the score function $\nabla_x \log q(x)$ which can be calculated efficiently even when $q$ has an intractable normalization constant. 
%\todo{earlier?} A key advantage here is that the score function does not depend on the normalization constant of $q$, 
Specifically, assuming $q(x) = f(x)/Z$ with $Z = \int f(x) dx $ being the normalization constant, we have $\score_q = \nabla_x \log f (x)$, independent of $Z$; calculating $Z$ involves a high dimension integration, and has been the major challenge for likelihood-based and Bayesian methods for model evaluation. %\citep[e.g.,][]{chen2012monte}. %or Bayesian methods based on Bayesian factors. 
%and can often be calculated efficiently even when the likelihood is intractable.

%is specified using an unnormalized probability. 
%$u_q(x,x') = \score_q(x)  k(x,x') \score_q(x')+  \score_q(x) \nabla_{x'} k(x,x')  +\score_q(x') \nabla_{x} k(x, x')  + \nabla_{x,x'} k(x, x') $. 
%where $k(x,x')$ is the kernel of the RKHS $\HH$, and $x,x'$ are i.i.d. random variables drawn from $p(x)$. 
%We do want the RKHS $\H$ to be rich enough to ensure that $\S(p,q) \neq 0$ whenever $p\neq q$, and this can be achieved by taking $k(x,x')$ to be universal kernel, for which the RKHS is a dense subset of the set of continuously differentiable functions. 
%We call $ \S(p,q)$ a \emph{kernelized Stein discrepancy} and present its theoretical properties in Section~\ref{sec:XXX}. 
%its theoretical properties and connections with other probability metrics are discussed in Section \red{[]}. 
With an i.i.d. sample $\{x_i\}$ drawn from the (unknown) $p(x)$, 
the form \eqref{equ:Dhh} also enables efficient empirical estimation of $ \S(p,q)$
via a $U$-statistic,% $\hat \S(p,q)$, 
%; given score function $\score_q(x)$ and
%an i.i.d. sample $\{x_i\}$ drawn from $p(x)$, the $\S(p,q)$ can be approximated by a $U$-statistics. 
\begin{align}
\hat \S(p,q) 
 =\frac{1}{n(n-1)}\sum_{i\neq j} u_q(x_i, x_j). 
 \label{equ:hatDhh}
\end{align}
%where $\{x_i\}$ is an i.i.d. sample drawn from $$
The distribution of $\hat \S(p,q)$ can be well characterized using the theory of $U$-statistics \citep{hoeffding1948class, serfling2009approximation}, allowing 
%This is a form of $U$-statistics (``$U$" stands for unbiasedness), which provides an minimum variance unbiased estimator for $\S(p,q)$ \citep{hoeffding1948class}. 
%The estimation error $U$-statistics has been well studied, and can be accessed using either asymptotic theory \citep[e.g.,][]{serfling2009approximation}, or concentration bounds \citep[e.g.,][]{hoeffding1963probability, peel2010empirical}. Our framework provides a natural approach for goodness-of-fit tests of whether a distribution $q(x)$ fits well with a sample $\{x_i\}$; 
%whether a sample $\{x_i\}$ follows a given distribution $q(x)$; 
%In addition, 
us to reduce the testing of $p = q$ to
 %to solve the goodness-of-fit test ($p = q$) by testing:
%it is turned into a zero-mean test problem under our framework: 
\begin{align*}
H_0:  \E_{p}[u_q(x,x')] = 0 &&vs.&& H_1:  \E_p[u_q(x,x')] > 0. 
%H_0:  \E_{x,x'}[u_q(x,x')] = 0 &&vs.&& H_1:  \E_{x,x'}[u_q(x,x')] > 0. 
%H_0: ~~ \E_{x,x'\sim p}[u_q(x,x')] = 0 ~~~~&&vs.&& ~~~ H_1: ~~ \E_{x,x'\sim p}[u_q(x,x')] > 0. 
\end{align*}
%A critical advantage of our method is that it dependent on $q(x)$ only through the score function $\score_q(x) = \nabla_x q(x) / q(x)$, which does not dependent on the normalization constant of $q(x)$; 
%$u_q(x,x')$ can be tractably calculated even when $q(x)$ is specified via an unnormalized probability $q(x) = \tilde q(x) /Z$, where $Z =\int_x q(x) dx$ is the normalization constant; 
%evaluating the normalization $Z$ has been a major bottleneck for model selection using likelihood, or Bayesian factors \citep[e.g.,][]{gelfand1994bayesian, salakhutdinov2008quantitative}, and still remains to be an active research problem \citep{burda2015accurate, liuuai15, grosse2015sandwiching}. 

%\paragraph{Related Work}Although there is a rich  literature on goodness-of-fit tests \citep[e.g.,][]{lehmann2006testing},  Most classical methods, such as $\chi^2$ test, Kolmogorov-Smirnov test, only work for univariate or simple distributions. 
\paragraph{Related Work}
The same idea was independently proposed by \citet{chwialkowski2016kernel} that appears simultaneously in this proceeding. %propose essentially the same test. 
The technique of combining Stein's identity with RKHS was first developed by \citet{oates2014control, oates2017control, oates2016convergence} for variance reduction. 
Reviews of classical goodness-of-fit tests can be found in e.g., \citet{lehmann2006testing}, where most methods 
%Most classical  goodness-of-fit tests 
have computational difficulty for unnormalized distributions. % \citep[e.g.,][]{lehmann2006testing}. 
One exception is \citet{fan2012output}, which uses the identity $\E_q [\vv \score_q]=0$ without using RKHS, but can be inconsistent in power since there exists $q\neq p$ with $\E_p[\vv\score_q]=0$. 
 %constructed: 

% An identical test was obtained simultaneously in independent work by Liu et al. (2016), for uncorrelated samples.

\paragraph{Outline}
%The rest of the paper is organized as follows. 
Section~\ref{sec:background} introduces RKHS and Stein's identity. 
%We start with backgrounds on RKHS and Stein's method in, 
Section~\ref{sec:ksd} defines our KSD and studies its main properties, and 
 Section~\ref{sec:U} discusses the empirical estimation of KSD and its application in goodness-of-fit tests. 
%Our main result is introduced in Section~\ref{sec:ksd}. 
We discuss related methods in Section~\ref{sec:related},  present experiments in Section~\ref{sec:exp} and conclude the paper in Section~\ref{sec:con}. 
%Conclusion and future directions are 

\paragraph{Notations}
%Throughout this paper, we denote by $\X$ a subset of
We denote by $\X$ a subset of
$d$-dimensional 
real space
 $\RR^d$. 
For a vector-valued function $\vv f(x) = [f_1(x),\ldots, f_{d'}(x)]$, its derivative $\nabla_x \vv f(x )=[\frac{\partial f_j(x)}{\partial x_i}]_{ij}$ is a $d \times d'$ matrix-valued function. 
%We will assume all the functions and densities are smooth functions. 
For a two-variable function (kernel) $k(x,x')$, we use $k(\cdot,x') = k_{x'}(\cdot)$ to refer to a function of $x$ indexed by fixed $x'$. 
For technical simplicity, we will assume all the functions we encounter are absolutely integrable, so that the
 Fubini-Tonelli theorem can be used to exchange the orders of integrals and infinite sums. 
%integrals and infinite sums can be exchanged. 

\section{Backgrounds}
\label{sec:background}
We first introduce positive definite kernels and reproducing kernel Hilbert spaces (RKHS) in Section~\ref{sec:kernel}, 
and then Stein's identity and operator in Section~\ref{sec:stein}.
\label{sec:back}
\subsection{Kernels and Reproducing Kernel Hilbert Spaces}
\label{sec:kernel}
%We briefly introduce backgrounds on Kernels and the related Reproducing Kernel Hilbert Spaces (RKHS). 
%\citet{steinwart2008support} 
%
%\begin{mydef}
%\end{mydef}
%We now establish the connection with reproducing kernel Hilbert space (RKHS). We start with a brief overview of RKHS; see e.g., \citet{steinwart2008support} for throughout treatment of RKHS. % can be found in e.g., \citet{steinwart2008support}. 

Let $k(x,x')$ be a positive definite kernel. The spectral decomposition of $k(x,x')$, 
as implied by Mercer's theorem, is defined as % of $k(x,x')$ is 
\begin{align}\label{equ:spectrumK}
k(x,x') = \sum_{j} \lambda_j e_j(x) e_j(x'),
\end{align}
where $\{e_j\}$, $\{\lambda_j \}$ are the orthonormal eigenfunctions and positive eigenvalues of $k(x,x')$, respectively, satisfying $\int e_i(x) e_j(x)dx = \ind[i=j]$,  
for $\forall i, j$. 

For a positive definite kernel $k(x,x')$, its related RKHS $\H$ comprises of linear combinations of its eigenfunctions, i.e., $f(x) = \sum_j f_j e_j(x)$ with $\sum_j f_j^2/\lambda_j < \infty$, endowed with an inner product $\la f, g\ra_{\H} = \sum_j f_j g_j  / \lambda_j$ between $f(x)$ and $g(x) = \sum_j g_j  e_j(x)$. 
Thus this Hilbert space is equipped with a norm $|| f ||_\H $ where $||f ||_\H^2 = \la f, f\ra_\HH = \sum_j f_j ^2  / \lambda_j$. 
One can verify that $k(x, \cdot)$ is in $\H$ and satisfies the important \emph{``reproducing" property},% of $\H$, 
\begin{align*}
f(x) = \la f, k(\cdot, x) \ra_\H, && 
k(x,x') = \la k(\cdot, x), k(\cdot, x') \ra_\H. 
\end{align*}
%There is an one-to-one mapping between positive definite kernel $k(x,x)$ and the RKHS $\H$. 
Every positive definite kernel $k$ defines a unique RKHS 
 for which $k$ is a reproducing kernel.
%reproducing kernel Hilbert space 

%\begin{lem}
%If $k(x,x')$ is smooth and in the Stein class of $p$, so is any function $f(x)$ in its related RKHS $\H$. 
%\end{lem}
%\begin{proof}
%This is true for $e_j$ with $\lambda_j>0$, since $\lambda_j e_j = \int k(\cdot ,x') e_k(x') dx'$, and hence for their linear combinations. 
%\red{TODO: Need Fubini again.}
%\end{proof}

We denote by $\H^d = \H \times \cdots \H$ the Hilbert space of $d\times 1$ vector-valued functions $\vv f = \{ f_\ell \colon  f_\ell \in\H \}_{\ell\in [d]}$, equipped with an inner product $\la 
\vv f, \vv g \ra_{\H^d}= \sum_{\ell \in [d]} \la f_\ell, g_\ell \ra_\H$ for $\vv f$ and $\vv g = \{g_\ell\}_{\ell\in [d]}$, and norm $|| \vv f||_{\H} = \sqrt{ \sum_{\ell}  || f_\ell||_{\H}^2}$.

%\section{Kernelized Stein Discrepancy}
%We introduce our Kernelized Stein Discrepancy and study its main properties. We start with introducing Stein's operator. 
%We make the following definition. 
\subsection{Stein's Identity and Operator}
\label{sec:stein}
%We introduce the basics of Stein's identity and operator, which play a central role in our method. %We first introduce the definition and several assumptions. 
%
\begin{mydef}
\label{def:def11}
Assume that $\X$ is a subset of $\R^d$ and $p(x)$ a continuous differentiable (also called smooth) density whose support is $\X$. 
% with piecewise smooth boundary $\partial X$. 
%$p$ and $q$ are two probability distributions on $\X$, with positive and continuous differentiable (smooth) densities $p(x)$, $q(x)$ in $\X$. 
The (Stein) score function of $p$ is defined as 
$$
\score_p = \nabla_x \log p(x) = %{\nabla_x p(x)} / {p(x)}. 
\frac{\nabla_x p(x)}{p(x)}. 
$$
We say that a function $f \colon \X \to \R$ is in the \textbf{Stein class} of $p$ if $f$ is smooth and satisfies %$\E_p[|\nabla_x f|] < \infty$ and 
\begin{align}
\label{equ:zeroD}
\int_{x\in \X} \nabla_{x} (f(x) p(x)) dx =0. 
\end{align}
The \textbf{Stein's operator} of $p$ is a linear operator acting on the {Stein class} of $p$, defined as
$$
\stein_p f(x) =  \score_p(x) f(x) + \nabla_x f(x). 
$$ 
Note that both $\score_p$ and $\stein_p f$ are $d\times 1$ vector-valued functions mapping from $\X$ to $\R^d$. 
A vector-valued function $\vv f(x) = [f_1(x),\ldots, f_{d'}(x)]$ is said to be in the Stein class of $p$ if all $f_i$, $\forall i\in [d']$ is in the Stein class of $p$. 
%We can also 
Applying $\stein_p$ on a vector-valued $\vv f(x)$ results a $d\times d'$ matrix-valued function, $\stein_p \vv f(x) =  \score_p(x) \vv f(x)^\top + \nabla_x \vv f(x). $
%We can also apply $\stein_p$ on a vector-valued function $\vv f(x) = [f_1(x),\ldots, f_{d'}(x)]$ in which case $\stein_p \vv f(x)$ is a $d\times d'$ matrix-valued function, $\stein_p \vv f(x) =  \score_p(x) \vv f(x)^\top + \nabla_x \vv f(x). $
\end{mydef}

\remark The condition \eqref{equ:zeroD} can be easily checked using integration by parts or divergence theorem; in particular, %we have 
when $\X = \RR^d$, \eqref{equ:zeroD} holds if 
$$
\lim_{||x||\to \infty} f(x) p(x) = 0,
$$ 
which holds, for example,  if $p(x)$ is bounded and $\lim_{||x||\to \infty} f(x)  =0$.
%and $f(x)$ is in the Stein class of any $p(x)$ if $\lim_{x\to \infty} f(x)  =0$. 
When $\X$ is a compact subset of $\RR^d$ with piecewise smooth boundary $\partial \X$, then by divergence theorem~\citep{marsden2003vector}, \eqref{equ:zeroD} holds if $f(x) p(x) = 0$ for $\forall x\in \partial 
\X$, or more generally if 
$
\oint_{\partial \X} p(x) f(x) \cdot \vec{n}(x) dS(x) = 0, 
$
where $\vec n(x)$ is the unit normal to the boundary $\partial \X$; $\oint_{\partial }  dS(x)$ denotes the surface integral over $\partial \X$. 
%Note that a special case of zero boundary condition is $f(x) = 0$, $\forall x\in \X$.  Our discussion includes the case when $\X = \R^d$ in which case the zero boundary condition is $\lim_{x\to \infty}p(x) f(x) =0$.
%We call $\hat\S_q(f(x)) =f(x) \score_q(x) + \nabla f(x)$ the stein operator of $q$.

\begin{lem}[Stein's Identity]%, \citet{stein2004use}]
\label{lem:basic1}
Assume $p(x)$ is a smooth density supported on $\X$, then
%Under Assumption~\ref{def:def11},  we have
$$
\E_p [\stein_p \vv f(x)] = \E_{p} [\score_p(x) \vv f(x)^\top  + \nabla \vv f(x)] = 0,
%\E_{p} [ f(x) \score_p(x) + \nabla f(x)]  = 0, ~~~~~ \score_p(x) = \frac{\nabla p(x)}{p(x)}, 
$$
 for any $\vv f$ that is in the Stein class of $p$.%, then %we have 
%where $\score_p(x) = \nabla_x \log p(x)$ is the score function of $p$. 
\end{lem}
\begin{proof}
By the definition of the Stein class, 
simply note that $  \score_p(x) \vv f(x)^\top + \nabla \vv f(x) 
= \nabla_x (\vv f(x)p(x))  / p(x).$
\end{proof}

The following  result gives a convenient tool for our derivation;
it relates the expectation under $p$ of Stein's operator $\stein_q f$ with the 
%score difference $(\score_q - \score_p)$ between $p$ and $q$.
difference of the score functions of $p$ and $q$. 
\begin{lem}[\citet{ley2013stein}]%[][Lemma 3.1]
\label{lem:basic2}
Assume $p(x)$ and $q(x)$ are smooth densities supported on $\X$
and $\vv f(x)$ is in the Stein class of $p$, 
we have 
$$
\E_p [\stein_q \vv f(x)] 
%=\E_{p} [ \score_q(x) f(x)  + \nabla f(x)]
= \E_p[(\score_q(x) - \score_p(x))\vv f(x)^\top]. 
%\E_{p} [ f(x) \score_p(x) + \nabla f(x)]  = 0, ~~~~~ \score_p(x) = \frac{\nabla p(x)}{p(x)}, 
$$
% for any $f(x)$ in the Stein class of $p(x)$. 
\end{lem}
\begin{proof}
Since $\E_p[\stein_p \vv f(x)] = 0$, %by Lemma~\ref{lem:basic1}, we have 
we have
$\E_{p} [ \stein_q \vv f(x)]  =\E_p[\stein_q \vv f(x) - \stein_p \vv f(x)] =  \E_p[(\score_q(x) - \score_p(x))\vv f(x)^\top].$
%Simply by integration by parts or divergence theorem. See Appendix for details. 
\end{proof}
Therefore, 
$\E_{p} [ \stein_q\vv f(x)] $ is 
%the expectation of the inner product of $f(x)$ and the score function difference $(\score_q(x)-\score_p(x))$. 
the $\vv f(x)$-weighted expectation of the score function difference $(\score_q(x)-\score_p(x))$ under $p$. 
%\red{
When $\vv f(x)$ is a $d\times1$ vector-valued function, $\E_{p} [ \stein_q \vv f(x)]$ is a $d\times d$ matrix; taking its trace gives a scalar
%Note that Lemma~\ref{lem:basic2} implies 
$$
\E_p [\trace(\stein_q \vv f(x))]
%=\E_{p} [ \score_q(x) f(x)  + \nabla f(x)]
= \E_p[(\score_q(x) - \score_p(x))^\top \vv f(x)],
%\E_{p} [ f(x) \score_p(x) + \nabla f(x)]  = 0, ~~~~~ \score_p(x) = \frac{\nabla p(x)}{p(x)}, 
$$%}
which was first derived in \citet{gorham2015measuring} using Langevin diffusion. 
%One can also use determinant or other matrix norms to replace the trace, and it is interesting direction to study their properties. 
It is an interesting direction to consider the possibility of  using determinant or other matrix norms instead of the trace. 
%One can also use determinant or other matrix norms to obtain scalars; 
%it is an interesting future direction to study their properties. 
%between the score functions. 
%This intuitive 
%Although our discrepancy measure has close connection with RKHS, it can be viewed as applying Stein operator twice over a symmetric function $k(x,x')$. We now build our framework from this elementary view. 
%We are ready to present our main definition. 

\section{Kernelized Stein Discrepancy}
\label{sec:ksd}
We introduce our kernelized Stein discrepancy (KSD) with an elementary definition motivated by Lemma~\ref{lem:basic2}, and then establish its connection with Stein's method and RKHS. 
%
%and establish its main properties. 

\begin{mydef}
A kernel $k(x, x')$ is said to be \emph{integrally strictly positive definition}, if for any function $g$ that satisfies $0<||g||_2^2 < \infty$, 
\begin{align}\label{equ:integral}
\int_{\X} g(x) k(x,x') g(x') dx dx'  > 0. 
&&
% \text{for ~ $\forall ||g||_2 < \infty, g\neq 0$}%  \in L^2(\X) \setminus \{0\}$.}% \text{for ~ $\forall g  \in L^2(\X) \setminus \{0\}$.}
\end{align}
\end{mydef}
%$k(x,x')$ is called 

\begin{mydef}
\label{def:ksdkxx}
%A function $k\colon \X \times \X \to \R$ is said to be symmetric if $k(x,x') = k(x', x)$ for $\forall x,x'\in \X$. We say that $k$ is in the Stein class of $p$ if $k(\cdot, x)$ is in the Stein class for any fixed $x$. 
%Given kernel $k(x,x'),  %two smooth densities $p$ and $q$, %kernel $k$
The {\bf kernelized Stein discrepancy} (KSD) $\S(p,q)$ between distribution $p$ and $q$ is defined as
\begin{align}
\label{equ:dpEdsq}
&\S(p,q) 
= \E_{x,x'\sim p} [ \vv \delta_{q,p}(x)^\top k(x,x')  \vv \delta_{q,p}(x') ] , 
%(\score_q(x') - \score_p(x')  )]. 
%&\S(p,q) = \notag\\
%&\!\!\!\!\!\!\E_{x,x'\sim p} [(\score_q(x) - \score_p(x))^\top k(x,x') (\score_q(x') - \score_p(x')  )]. 
\end{align}
where $\vv \delta_{q,p}(x) = \score_q(x) - \score_p(x)$ is the score difference between $p$ and $q$, and 
$x$, $x'$ are  i.i.d. draws from $p(x)$. 
\end{mydef}
%The following result shows that 
%$\S(p,q)$ is a valid discrepancy measure between $p$ and $q$ under mild conditions. % when $\vv \delta_{p,q} =  \score_q - \score_p \in \H^d$. 
%$\score_p - \score_q\in \H^d$. 
%for any positive definite kernel. 
\begin{pro}
\label{pro:pos}
Define $\vv g_{p,q}(x) = p(x)( \score_q(x)- \score_p(x))$. 
Assume $k(x,x')$ is integrally strictly positive definite, and $p$, $q$ are continuous densities with $||\vv g_{p,q}||_2^2 < \infty$, 
we have $\S(p,q) \geq 0 $ and $\S(p,q) = 0$ if and only if $p = q$.
%For any integrally strictly positive definite kernel $k(x,x')$ associated with a reproducing kernel Hilbert space (RHKS) $\H$, we have 
%$\S(p,q) \geq 0 $ for any $p$, $q$ with smooth densities. 
%In addition, when %$\vv \delta_{q,p} = p(x)( \score_q(x)- \score_p(x)) \in (L^2)^d$, 
%$||p(x)( \score_q(x)- \score_p(x))||_2 < \infty$
%we have $\S(p,q) = 0$ if and only if $p = q$. 
%Following Definition~\ref{def:ksdkxx}, we have 
%$$
%\S(p,q) = 0 for $p = q$ for any function $k(x,x')$. 
%$$
\end{pro}
\begin{proof}
Result directly follows the definition in \eqref{equ:integral}. 
%Result directly follows the definition of integrally strictly positive definite in \eqref{equ:integral}. 
%$\S(p,q)\geq 0$ is a direct result of the positive definiteness of $k(x,x')$. Because $\vv \delta_{q,p} = \{\vv \delta^\ell_{p,q}\}_{\ell\in [d]} \in \H^d$, we have $\vv \delta^\ell_{q,p}(x) = \sum_{j} \alpha_{j\ell} e_j(x)$, $\forall \ell\in [d]$, and $\S(p,q) = \sum_\ell \sum_j \lambda_j\alpha_{j\ell}^2$ by the spectral representation~\eqref{equ:spectrumK}. Therefore, $\S(p,q)=0$ implies $\alpha_{j\ell} = 0$ for $\forall j, \ell$, which then implies $\score_p = \score_q$ and hence $p = q$. 
%When $\log p, \log q \in \H$, we have $\score_p, \score_q\in \H^d$ by \citet{zhou2008derivative}.  
\end{proof}
This establishes $\S(p,q)$ as a valid discrepancy measure. % between $p$ and $q$. 
The requirement that $||\vv g_{p,q}||_2^2< \infty$ 
is a mild condition and can easily hold, e.g., when the tail of $p(x)$ decays exponentially, 
but it may not hold when $p(x)$ has a heavy tail.\footnote{One counterexample as proposed by an anonymous reviewer is when $p$ is a Cauchy distribution and $q$ is a Gaussian distribution.}
%holds easily when $p$ has exponential tails, but may not 

%\remark 
%Proposition~\ref{pro:pos} requires $(\score_q - \score_p) \in \H^d$ to make $\S(p,q)$ a valid discrepancy measure. 
%is in $\H^d$%establishes $\S(p,q)$ as a discrepancy measure between densities whose score difference $(\score_q - \score_p)$ is in $\H^d$. 
%Therefore, 
%To apply this result, 
%we would need to take $\H$ to be rich enough to capture a broad class of score differences. This can be achieved by taking $k(x,x')$ to be a universal kernel, whose related RKHS $\H$ is a dense subset of continuous functions on compact subsets of $\R^d$ under $L_{\infty}$ norm \citep{steinwart2008support}. 
%Many commonly used kernels, such as the RBF kernel, have been shown to be universal \citep{steinwart2008support}. 
%Therefore, $\S(p,q)$ captures the smooth deviation (in the sense of RKHS) over the score functions. 
%any deviation over $\s_p$ in $\H^d$. Note that this is true even when $\s_p$ itself is not in $\H^d$. 
%Most commonly used kernels on $\R^d$ are in Stein class for any density, 
%It is obvious that $\S(p,q) = 0$ for any $p = q$. In the following, we show that $\S(p,q)$ has increasingly nice properties as we impose more conditions on kernel $k(x,x')$. In particular, 1) $\S(p,q)$ yields the computationally convenient form \eqref{equ:Dhh} for any symmetric $k(x,x')$ in the Stein class of $p$; 2) $\S(p,q) \geq 0$ if $k(x,x')$ is also positive definite;  3) $\S(p,q) = 0$ iff $p = q$ if $\log p, \log q $ belong to the RHKS related to $k(x,x')$. 

The $\S(p,q)$ as defined in \eqref{equ:dpEdsq} requires to know both $\score_p$ and $\score_q$; 
we now apply Stein's identity to derive 
the more convenient form \eqref{equ:Dhh} that only requires $\score_q$. 
%a more computationally convenient form that only requires $\score_q$. 

\begin{mydef}
A kernel $k(x,x')$ is said to be in the Stein class of $p$ if $k(x,x')$ has continuous second order partial derivatives, and
both $k(x, \cdot)$ and $k(\cdot, x)$ are in the Stein class of $p$ for any fixed $x$. 
\end{mydef}
It is easy to check that the RBF kernel $k(x,x') =  \exp(-\frac{1}{2 h^2} {||x - x'||^2_2})$ is in the Stein class for smooth densities supported on $
\X = \R^d$. 
%for any smooth density $p(x)$ supported in $\X = \R^d$ with finite mean and variance. 

\begin{pro}\label{pro:steinrkhsIS}
If $k(x,x')$ is in the Stein class of $p$, so is any $f\in \H$. 
\end{pro}
%\begin{proof}See appendix. \end{proof}

\begin{thm}
\label{thm:kxx}
Assume $p$ and $q$ are smooth densities and 
$k(x,x')$ is in the Stein class of $p$. 
Define
\begin{align*}
&u_q(x,x') = \score_q(x)^\top k(x,x') \score_q(x')  +   \score_q(x)^\top \nabla_{x'}k(x,x') +  \\ 
&~~~~~~~~~~~~~~~~~~~  + \nabla_{x}k(x,x')^\top   \score_q(x') + \trace(\nabla_{x,x'}k(x,x')).
\end{align*}
\myempty{
Define $u_q(x,x') = \trace(U_q(x,x'))$ where 
%$$U_q(x,x') =k(x,x')  \score_q(x) \otimes \score_q(x') + \score_q(x) \otimes \nabla_{x'}k(x,x')  + \nabla_{x}k(x,x')  \otimes \score_q(x') + \nabla_{x,x'}k(x,x').$$ 
\begin{align*}
&U_q(x,x') = \score_q(x) k(x,x') \score_q(x')^\top  +   \score_q(x) \nabla_{x'}k(x,x')^\top +  \\ 
&~~~~~~~~~~~~~~~~~~~  + \nabla_{x}k(x,x')   \score_q(x')^\top + \nabla_{x,x'}k(x,x').
\end{align*}
}
%$$U_q(x,x') = \score_q(x) k(x,x') \score_q(x')^\top + \score_q(x) \nabla_{x'}k(x,x')^\top  + \nabla_{x}k(x,x')   \score_q(x')^\top + \nabla_{x,x'}k(x,x').$$ 
 \begin{align}
\hspace{-0em}\text{then} &&~~~~~~&& \S(p,q) = \E_{x,x'\sim p} [u_q(x,x')]. && ~~~~~~\hfill
\label{equ:dpEu}
\end{align}
\end{thm}
\begin{proof}
%\eqref{equ:dpEu} is obtained by 
Apply Lemma~\ref{lem:basic2} twice, first on $k(\cdot ,x')$ for fixed $x'$, and then with fixed $x$. See the Appendix. %for details. 
%Define $v(x,x') =  \score_q(x') k(x,x') + \nabla_{x'} k(x,x') =  \stein_q k_x(x') $, then 
%\begin{align*}
%\S(p,q) 
%& = \E_{x,x'\sim p} [\vv \delta_{q,p}(x)^\top k(x,x') \vv \delta_{q,p}(x')] \\
%&= \E_{x\sim p}[ \vv \delta_{q,p}(x)^\top \E_{x' \sim p} [ k(x,x') \vv \delta_{q,p}(x') ] ]\\
%&= \E_{x x'\sim p}[  \vv \delta_{q,p}(x)^\top  \stein_q k_x (x' ) ]\\
%&= \E_{ x,x'\sim p}[  \vv \delta_{q,p}(x)^\topv(x,x')  ]\\
%&=\E_{x,x' \sim p} [\trace( \stein_{q} ] \\
%&=\E_{x,x' \sim p}[ \trace(\score_q(x) k(x,x') \score_q(x')^\top + \score_q(x) \nabla_{x'}k(x,x')^\top  + \nabla_{x}k(x,x')   \score_q(x')^\top + \nabla_{x,x'}k(x,x'))]
%\end{align*}
%where the second equality applies Lemmaa \ref{lem:basic2} on $k_{x} (x')$, and the second to last equality applies Lemma \ref{lem:basic} to the function $A_q k_{x'} (x)$.
\end{proof}
%
%\remark  The definition of $\S(p,q)$ admits some connection with the method of exchangeable pairs in Stein's method, but we leave an in-depth study as a future work.\red{[]}. 
%
%\remark  Note that $U_q(x,x')$ is a matrix, and we can define $u_q$ to be any matrix norm of $U_q$; we use trace for convenience. \red{[]}
%
The representation in \eqref{equ:dpEu} is of central importance for our framework, since it provides a tractable formula for empirical evaluation of $\S(p,q)$ and its confidence interval based on the sample $\{x_i\} \sim p$ and score function $\score_q$; see Section~\ref{sec:U} for further discussion. 
An equivalent result of Theorem~\ref{thm:kxx} was first presented in Theorem 1 of \citet{oates2014control}.

%This reduces the test of $p=q$ to whether $\E_{x,x'}[u_q(x,x')] = 0$. 
%The test of whether $p=q$ reduces to an one-side zero mean test: 
%$$H_0\colon  \E_{x,x'}[u_q(x,x')] = 0 ~~~~~v.s. ~~~~~ H_1\colon \E_{x,x'}[u_q(x,x')] > 0. $$
%Equivalently, one may also calculate an confidence interval of $  \E_{x,x'}[u_q(x,x')]$ and check whether zero falls inside it. %the confidence interval. 

\myempty{
An alternative form of $U_q(x,x')$ is $$U_q(x,x') = k(x,x') \big([\score_q(x) + \nabla_{x}\log k(x,x')]  [\score_q(x') + \nabla_{x'}\log k(x,x')]^\top  + \nabla_{x,x'}\log k(x,x')\big). $$

\remark It is easy to find kernels in Stein classes of given distributions; in particular,
\begin{pro}
The RBF kernel $k(x,x') =  \exp(-\frac{||x - x'||^2_2}{2 h^2})$ is Stein class for any smooth density $p(x)$ supported in $\X = \R^d$ with finite mean and variance. 
\end{pro}
%\begin{proof}
%First, we should have $\lim_{x\to \infty} p(x) = 0$, and hence $\lim_{x\to }$
%First consider $k_{x'}(x)$ for fixed $x'$, we have $k_{x'}(x) p(x) \to 0$. 
%\end{proof}
\begin{exa}
Consider the RBF kernel $k(x,x') =  \exp(-\frac{||x - x'||^2_2}{2 h^2})$, we have %, $x,x'\in \R^d$, %and Gaussian distribution $q(x) = \normal(x; \mu, \sigma^2)$, 
then 
\begin{align*}
&\nabla_x k(x,x')   =  k(x,x') \frac{1}{h^2}(x' - x),  &&
&\nabla_{x,x'} k(x,x')   =  k(x,x') \big( -\frac{1}{h^4}(x' - x) (x' - x)^\top  + \frac{1}{h^2}  I \big), 
%&\score_q(x) =  \frac{\mu - x}{\sigma^2  }
\end{align*}
where $I$ is the $d\times d$ identity matrix. Therefore, we can show
\begin{align*}
u_q(x,x') 
%&= \trace( \score_q(x)  k(x,x')  \score_q(x')^\top + \frac{1}{h^2} \score_q(x) (x - x')^\top+ \frac{1}{h^2} (x' - x) \score_q(x')^\top  + \nabla_{x,x'} k(x,x') )\\
&= k(x,x') \big(  \score_q(x)^\top \score_q(x') 
+ \frac{1}{h^2} \score_q(x)^\top (x - x')
+ \frac{1}{h^2} (x' - x)^\top \score_q(x')
 - \frac{1}{h^4} ||x' - x ||_2^2  + \frac{d}{h^2}  \big) \\
 & = k(x,x') 
 \big(
[\score_q(x) + \frac{1}{h^2}(x' - x)]^\top [\score_q(x') + \frac{1}{h^2}(x-x')] + \frac{d}{h^2}\big)
%&= k(x,x') \big[  \frac{(\mu - x)(\mu - x')}{\sigma^4}  -\frac{(x' - x) ^2}{h^4}   + \frac{1}{h^2}  \big] 
\end{align*}
Taking $q(x)$ to be Gaussian $\normal(0, \mu, \Sigma)$, then $\score_q(x) = - \Sigma^{-1}(x-\mu)$ and
$$
u_q(x,x') = k(x,x') 
\big(
[-\Sigma^{-1}(x-\mu) + \frac{1}{h^2}(x' - x)]^\top [-\Sigma^{-1}(x'-\mu) + \frac{1}{h^2}(x-x')] + \frac{d}{h^2}\big). 
%\big[ (x-\mu)^\top\Sigma^{-2}(x'-\mu)  -\frac{1}{h^4}{||x' - x ||_2^2}   + \frac{d}{h^2}  \big].
$$
%Let $\delta = x-x'$, $y= x+ x'$, we have 
%\begin{align*}
%u(x,x') 
%&\propto\exp(-\frac{\delta^2}{2h^2}) \big[  \frac{ y^2/4   - \mu y}{\sigma^4} +  \frac{-\delta^2/4 +  \mu^2}{\sigma^4}  -\frac{\delta^2}{h^4}   + \frac{1}{h^2}  \big]. 
%\end{align*}
%We can fix $\delta$ and let $y \to \infty$, then $u(x,x')\to +\infty$. 
%\todo{how to fix this problem? Using some generalized score function instead?}
\end{exa}
}

%We can further establish that $\S(p,q) \geq 0$ and provide a spectrum representation using Mercer's Theorem. 
%
Using the spectral decomposition of $k(x,x')$, we can show that $\S(p,q)$ is effectively applying Stein's operator simultaneously on all the eigenfunctions $e_j(x)$ of $k(x,x')$.  
\begin{thm}
\label{thm:mercer22}
Assume $k(x,x')$ is a positive definite kernel in the Stein class of $p$, with positive eigenvalues $\{\lambda_j\}$ and eigenfunctions $\{e_j(x)\}$, %, and satisfies the condition in Theorem~\ref{thm:kxx}. 
then $u_q(x,x')$ is also a positive definite kernel, and can be rewritten into
%\red{Denote by $\stein_q e_j(x)  =\score_q(x)  e_j(x)  + \nabla_x e_j(x)$ the Stein's operator acted on $e_j$, we have}
%Further, assume $k(x,x')$ yields a spectrum decomposition, as implied by Mercer's theorem, of% of $k(x,x')$ is 
%$$k(x,x') = \sum_{j=1}^{\infty} \lambda_j e_j(x) e_j(x'),$$
%where $\{e_j\}$, $\{\lambda_j \}$ are the orthonormal eigenfunctions and nonnegative eigenvalues of $k(x,x')$, respectively, satisfying $\int e_i(x) e_j(x)dx = \ind(i=j)$,  
%then %we have $\S(p,q) \geq 0$ for any $p,q$, and in addition, 
\begin{align}
\label{equ:dpuq}
u_q(x,x') =  \sum_{j}  \lambda_j ~ [\stein_q e_j(x) ] ^\top ~ [\stein_q{e_j(x')}], 
%U_q(x,x')  &= \sum_{j}  \lambda_j ~ [\stein_q e_j(x) ]  ~ [\stein_q{e_j(x')}]^\top,    %(e_j(x) \score_q(x) + \nabla e_j(x)).   \notag
\end{align}
where $\stein_q e_j(x)  =\score_q(x)  e_j(x)  + \nabla_x e_j(x)$ is the Stein's operator acted on $e_j$. 
In addition, 
%where $\hat\S_q(e_j(x)) = e_j(x) \score_q(x) + \nabla e_j(x)$ is the Stein's operator of $q$. 
%and hence %Hence, we have
\begin{align}
\label{equ:dpEej2} 
\S(p,q) 
& = \sum_{j} \lambda_j  || \E_{x\sim p} [\stein_q e_j(x) ] ||^2_2. % \notag \\
%&= \sum_{j=1}^\infty  \lambda_j || \E_{x\sim p} [e_j(x) \score_q(x) + \nabla e_j(x)] ||^2_2  \notag  \\
%& = \sum_{j} \lambda_j  || \E_{x\sim p} [ e_j(x)(\score_q(x) - \score_p(x)) ] ||^2_2. 
\end{align}
Note that although $\{e_j\}$ are orthonormal, the $\{\stein_q e_j(x) \}$ are no longer orthonormal in general. % after the Stein operator is applied. 
%
%In addition,  $u_q(x,x') =  \sum_{j}  \lambda_j ~ [\stein_q e_j(x) ] ^\top ~ [\stein_q{e_j(x')}]$ and is also a positive definite kernel. 
\end{thm}
%Theorem~\ref{thm:mercer} suggests that $\S(p,q)$ is effectively applying Stein's operator simultaneously on all the eigenfunctions $e_j$ of kernel $k(x,x').$ 

Finally, we are ready to establish the variational interpretation of $\S(p,q)$ 
that motivated this work, 
%discussed in the introduction, 
that is, it can be treated as
%as a \emph{maximum} stein discrep
%The following result shows 
%that $\S(p,q)$ is equivalent to form ~\ref{equ:f00}, 
%can be treated as 
the maximum of $\E_{x\sim p}[\stein_q f(x)]$ when optimizing $f$ in the unit ball of RKHS $\H$ related to kernel $k(x,x')$.  

\begin{thm}
\label{thm:rkhs}
Let $\H$ be the RKHS related to a positive definite kernel $k(x,x')$ 
in the Stein class of $p$. 
%that satisfies the conditions in Theorem~\ref{thm:kxx}.  
%
Denote by $\vv \beta(x' ) =  \E_{x\sim p} [\stein_q k_{x'}(x)]$,
%= \E_x[ \score_q(x) k(x,x') + \nabla_x k(x,x')] 
% = \E_x[  (\score_q(x)   - \score_p(x))  k(x,x') ]$, then 
then
\begin{align}
\label{equ:beta2}
\S(p, q) = ||\vv \beta||_{\H^d}^2. 
\end{align}
Further, we have $\la \vv f, \vv \beta \ra_{\H^d}= \E_x [\trace(\stein_q  \vv f)]$ for $\vv f \in \H^d$, and hence
 %and \todo{does $\vv \beta \in \H^d?$}$
%Rewriting in a variational form under $\H^d$, we have%then we have %also have
\begin{align}
\sqrt{\S(p, q)}
%&  = \max_{\vv f\in \H^d} \bigg\{ \la \vv f, ~ \vv \beta \ra_{\H^d} ~~~s.t. ~~~ ||\vv f ||_{\H^d} \leq 1\bigg\} \\
& = \max_{\vv f \in \H^d} \bigg\{ \E_{x} [\trace(\stein_q \vv f)] ~~~~ s.t.~~~~ ||\vv f ||_{\H^d} \leq 1 \bigg \} 
\label{equ:good}
%&= \max_{\vv f_\ell \in \H} \bigg\{  \sum_{\ell=1}^d \E[f_\ell(x) ( \score_q^{\ell}(x) - \score_p^\ell(x))]   ~~~ s.t.~~~ \sum_\ell ||f_\ell||_\HH^2 \leq 1, ~~ \forall \ell \in [d]  \bigg\}  \label{equ:good} ,
\end{align}
where the maximum is achieved when $\vv f = \vv \beta/||\vv \beta||_{\H^d}$.
\end{thm}
%\citet{gorham2015measuring} also proposed the same variational 
Note that \eqref{equ:good} is slightly different from the definition in \citet{gorham2015measuring} which do not use the square root; 
we can take the square root off by optimizing within the ball of $|| \vv f ||_{\H^d}^2 \leq \S(p,q)$ instead.
%they are matched if we take $\mathcal F$ to be ball of $|| \vv f ||_\H^2 \leq \S(p,q).$

%\section{Empirical Estimation and Goodness-of-fit Testing}
%\section{KSD-based Goodness-of-fit Test}
\section{Goodness-of-fit Testing Based on KSD}
\label{sec:U}
The form in \eqref{equ:dpEu} allows efficient estimation of $\S(p,q)$ in practice. %We discuss empirical estimation based on $U$-statistics in this section. 
Given i.i.d. sample $\{x_i\}$ drawn from an unknown $p$ and the score function $\score_q(x)$, %Denote by $u_q(x,x') = \trace(U_q(x,x'))$; 
we can estimate $\S(p,q)$ by %a $U$-statistics of form, 
\begin{align}
%&& && 
\hat\S_u(p,q) = \frac{1}{n(n-1)}\sum_{1\leq i\neq j\leq n} u_q(x_i, x_j),   
%\\%&&\text{or}&& \hat\S_v(p,q) = \frac{1}{n^2} \sum_{i,j=1}^n u_q(x_i, x_j), 
\label{equ:SU}
\end{align}
where $\hat\S_u(p,q)$ is a form of $U$-statistics (``$U$" stands for unbiasedness), which provides a minimum-variance unbiased estimator for $\S(p,q)$ \citep{hoeffding1948class, serfling2009approximation}. 
%and $\hat\S_v(p,q)$ is a form of von Mises, or $V$- statistics, which provides an biased estimator, but has the advantage of always being nonnegative since $u_q(x,x')$ is positive definite. 
%The theory of $U$-statistics has been well established \citep[e.g,][]{serfling2009approximation}.
%We give more discussion on $U$-statistics, since 
%The unbiasedness of the $U$-statistics 
We can also estimate $\S(p,q)$ using a $V$-statistic of form $\frac{1}{n^2} \sum_{i,j=1}^n u_q(x_i, x_j)$, 
which provides a biased estimator, but has the advantage of always being nonnegative since $u_q(x,x')$ is positive definite. 
We will focus on the $U$-statistic in this work because of its unbiasedness. 
%which provides a biased, but always nonnegative 
%We can also construct a similar test based on the $V$-statistics, but the result is often equivalent to that based on $U$-statistics. 

\begin{thm}
\label{thm:uasym}
Let $k(x,x')$ be a positive definite kernel in the Stein class of $p$ and $q$. 
Assume the conditions in Proposition~\ref{pro:pos} holds, %, and $k(x,x')$ is positive definite kernel in the Stein class of $p$ and $q$, 
%\jnote{why is this q in red? Qiang: in most cases we only need k(x,x') in the Stein class of $p$, but i also need $q$ in this case to prove sigma_u \neq 0 when p \neq q, not sure if it is necessary}, 
and $\E_{x,x'\sim p}[u_q(x,x')^2] < \infty $, we have 

1) If $p \neq q$, then $ \hat\S_u(p,q) $ is asymptotically normal with
$$
\sqrt{n} (\hat\S_u(p,q) - \S(p,q)) \overset{d}{\to}  \normal(0, \sigma_u^2), 
$$
where $\sigma_u^2 = \var_{x\sim p}(\E_{x'\sim p}[u_q(x,x')])$ and $\sigma_u^2 \neq 0$.  

2) If $p = q$, then we have $\sigma_u^2 = 0$  %and $\hat\S_u(p,q) = O_p(1/n)$, 
(the $U$-statistics is degenerate) and %and we have
\begin{align}
\label{equ:chi2}
n \hat\S_u(p,q) \overset{d}{\to} \sum_{j=1}^{\infty} c_j (Z_j^2 - 1), 
\end{align}
where $\{Z_j\}$ are i.i.d. standard Gaussian random variables, and
$\{c_{j}\}$ are the eigenvalues of kernel $u_q(x,x')$ under $p(x)$, that is, they are the solutions of 
$
c_j \phi_j(x) = \int_{x'} u_q(x,x') \phi_j(x') p(x') dx' 
$ for non-zero $\phi_j$. 
\end{thm}
\begin{proof}
Using the standard asymptotic results of $U$-statistics in \citet[][Section 5.5]{serfling2009approximation}, we just need to check that $\sigma_u^2\neq 0$ when $p\neq q$ and $\sigma_u^2 = 0$ when $p =q$.  See Appendix for details. 
%
%Note that we have $\E_{x'} [u_q(x,x')] = \sum_{\ell}\score_q^\ell(x)\vv g_\ell(x) + \nabla_{x_\ell} \vv g_\ell(x)$, where $g(x) = \E_{x'} [\stein_q k_x(x')]$ and is in the Stein class of $p$. \red{[XX]}
%
%When $p = q$, we have $g(x) \equiv 0$, and hence $\sigma_u^2 = 0$. 
%
%When $p\neq q$, if $\sigma_u^2 =0$, we must have $\E_{x'} [u_q(x,x')] = c$, where $c$ is a constant; this is equivalent to $\sum_{\ell} \nabla_{x_\ell} [q(x)g_\ell(x)] = c q(x) $ and therefore $\int \sum_{\ell} \nabla_{x_\ell} [q(x)g_\ell(x)] dx = c \int q(x) dx$, since $\vv g(x)$ is in the Stein class of $p$, it reduces to $c = 0$ and hence $ \S(p,q)  = c =0$, which contradicts with $p \neq q$. 
\end{proof}
%Note that our
%
%
%\red{ Theorem~\ref{thm:uasym} suggests that under the alternative hypothesis $(p\neq q)$, $\hat \score_u(p,q)$ is asymptotic normal and converge at a typical $n^{-1/2}$ rate, while under the null hypothesis $(p=q)$, $\hat \score_u(p,q)$ is an degenerate $U$-statistics; it converges faster at a $n^{-1}$ rate, with a limit distribution of infinite weighted sum of $\chi^2$ random variables. %, which does not have analytic form unless $c_j = 0,$ or $1$.  }
 
Theorem~\ref{thm:uasym} suggests that $n\hat \S_u(p,q)$ has a well defined limit distribution under the null $p=q$, that is, $n\hat\S_u(p,q) <\infty $ with probability one, 
 but grows to $\infty$ at a $\sqrt{n}$-rate under any fixed alternative hypothesis $q\neq p$. 
 This suggests a straightforward goodness-of-fit 
 testing procedure: Denote by $F_{n\hat\S_u}$ the CDF of $n \hat\S_u$ under the null $p=q$, and set $\gamma_{1-\alpha}$ the $1-\alpha$ quantile of $F_{n\hat\S_u}$, i.e.,  $\gamma_{1-\alpha} = \inf \{s  \colon F_{n\hat\S_u} (s) \geq 1- \alpha\}$, then we reject the null 
 with significant level $\alpha$ if $n \hat\S_u \geq \gamma_{1-\alpha}$. 
%This suggests that the test based on  $n\hat \S_u(p,q)$ is consistent in level under any fixed alternative. 
 %
 \begin{pro}
 Assume the conditions in Theorem~\ref{thm:uasym}. 
 For any fixed $q\neq p$, the limiting power of the test that rejects the null $p=q$ when $n \hat\S_u(p,q) > \gamma_{1-\alpha}$ is one, that is, 
 %Hence, 
 the test is consistent in power against any fixed $q\neq p$. 
% Let $\gamma_{\alpha}$ be $\alpha$ quantile of asymptotic distribution in \eqref{equ:chi2}, then the test that rejects the null $p=q$ is consistent under any alternative $p$ that $p\neq q$. 
 \end{pro}
 
One difficulty in implementing this test is that the limit distribution in \eqref{equ:chi2} and its $\alpha$-quantile does not have analytic form unless $c_j = 0,$ or $1$. 
 Fortunately, the same type of asymptotics appears in many other classical goodness-of-fit tests, such as Cramer-von Mises test, Anderson-Darling test, as well as two-sample tests \citep{gretton2012kernel}.
%  similar limit distributions appear in many classical goodness-of-fit tests, such as Cramer-von Mises test, Anderson-Darling test.
%In addition, the same type of asymptotic \eqref{equ:chi2} also appears in the maximum mean discrepancy statistics for two sample tests \citep{gretton2012kernel}. 
% Fortunately, similar limit distributions appear in many classical goodness-of-fit tests, such as Cramer-von Mises test, Anderson-Darling test.
%, and as well as the weighted quadratic tests, 
% which can be treated an umbrella term for test statistics of with limit distribution \eqref{equ:chi2}; we will discuss the weighted quadratic tests in more depth in Section~\ref{sec:neyman}. 
%In addition, the same type of asymptotic \eqref{equ:chi2} also appears in the maximum mean discrepancy statistics for two sample tests \citep{gretton2012kernel}. 
As a consequence, a line of work has been devoted to approximating the critical values of \eqref{equ:chi2}, 
including bootstrap methods \citep{arcones1992bootstrap, huskova1993consistency, chwialkowski2014wild} and eigenvalue approximation \citep{gretton2009fast}. 

%Alternatively,  it is also possible to use other statistics with Gaussian limit. For example, a simpler estimator 
%or alternative statistics with Gaussian limit, mostly by using block-wise averaging \citep{ho2006two, zaremba2013b}. 
%Bootstrap procedures are developed to approximate the sample distribution of the $U$-statistics and can be shown to be consistent \citet{arcones1992bootstrap, huskova1993consistency}. 
%It is also possible to approximate the null distribution by estimating the eigenvalues $c_j$ of kernel $u_q(x,x')$; see e.g., \citet{gretton2009fast}.  
%
%Alternatively, one may consider other statistics that have a Gaussian limit under null, e.g., by using applying $U$-statistics on blocks of subsamples, and averaging across the blocks \citep{ho2006two, zaremba2013b}. 
%blocking averaging \citep{}
%\citet{ho2006two} uses an alternative blocking averaging based estimator to obtain asymptotic normal null distribution;  this idea is further developed in \citet{zaremba2013b}  to allow efficient trade-off between test power and computation time. 

\begin{algorithm}[tb] % 
\caption{Bootstrap Goodness-of-fit Test based on KSD}  \label{alg:test}
\begin{algorithmic}
\STATE \emph{Input:} Sample $\{x_i\}$ and score function $\score_q(x) = \nabla_x \log q(x)$. Bootstrap sample size $m$. % and significance level $\alpha$. 
\STATE \emph{Test:} $H_0$: $\{x_i\}$ is drawn from $q $  ~~~v.s ~~   $H_1$: $\{x_i\}$ is not drawn from $q$.  
\STATE 1. Compute $\hat \S_u$ by \eqref{equ:SU} and $u_q(x,x')$ as defined in Theorem~\ref{thm:kxx}. 
Generate $m$ bootstrap sample $\hat \S^*_u$ by \eqref{equ:bt}. % and $u_q(x,x')$ as defined in Theorem~\ref{thm:kxx}. 
%Generate sample $\{y_i\} $ via $y_i = b + A  x_i + A^{1/2} \epsilon_i$, where $\epsilon_i$ is i.i.d. standard Gaussian.  
\STATE 2. Reject $H_0$ with significance level $\alpha$ if the percentage of $\hat \S^*_u$ that satisfies $\hat \S^*_u > \hat \S_u$
is less than $\alpha$. 
%$\#(\hat S_u > \S^*_u ) >1-\alpha$. 
%2. Perform kernelized Stein test between $\{y_i \}$ and $\tilde q_{b,A}$ as defined in \eqref{equ:qbA}. 
\end{algorithmic}
\end{algorithm} 

%In this work, we use bootstrap for its simplicity. However, as discussed in \citet{arcones1992bootstrap}, the typical bootstrap may not work for degenerate $U$-statistics and %we instead use the  \citet{huskova1993consistency}
In this work, we adopt the bootstrap method suggested in  \citet{huskova1993consistency, arcones1992bootstrap}: %for approximating the null distribution of $ \hat\S_u(p,q)$. 
%We adopt the method in \citet{huskova1993consistency, arcones1992bootstrap}, that is, 
%we apply a simple bootstrap method for to estimate the distribution of ${n\hat\S_u(p,q)}$; specially,
We repeatedly draw multinomial random weights $(w_{1}, \ldots w_{n}) \sim ~ \mathrm{Mult}(n~; ~  \frac{1}{n}, \ldots, \frac{1}{n})$, and calculate bootstrap sample 
\begin{align}\label{equ:bt}
 \hat  \S^*_u(p, q)  =   \sum_{i\neq j} (w_{i}  - \frac{1}{n}) (w_j -\frac{1}{n})u_q(x_i, x_j), 
\end{align}
and then calculate the empirical quantile $\hat \gamma_{1-\alpha}$ of  $n \hat  \S^*_u(p, q)$. %the bootstrap sample, 
%and reject the null $p=q$ at significance level $\alpha$ if $\gamma_{\alpha} > 0$. %, where $\gamma_{\alpha}$ is the $\alpha$-quantile of the empirical distribution of $\hat S^*_u(p=q)$.
%which repeat this for a large number of times, and reject the null hypothesis $(p=q)$ at significance level $\alpha$ if $\gamma_{\alpha} > 0$, where $\gamma_{\alpha}$ is the $\alpha$-quantile of the empirical distribution of $\hat S^*_u(p=q)$.
The consistency of $\hat \gamma_{1-\alpha}$ for degenerate $U$-statistics has been established in \citet{arcones1992bootstrap, huskova1993consistency}. %that is, 
\begin{thm}[\citet{huskova1993consistency}]
 Assume the conditions in Theorem~\ref{thm:uasym}. % and $\E_{x,x'\sim p} [u_q(x,x')] = 0$. 
 If $p =q$, then as the bootstrap sample size $m\to \infty$,  
%nder $p = q$, 
$$
\sup_{s \in \RR} \big |~\prob( n \hat  \S^*_u   \leq s ~|~ \{x_i\}_{i=1}^n) ~ - ~\prob( n \hat\S_u \leq s) ~\big | \to 0, 
$$
that is, the bootstrap test attains the correct significance level asymptotically (consistent in level). 
\end{thm}
It is important to note, on the other hand, that the more usual bootstrap, such as $\sum_{i\neq j} w_i w_j u_q(x_i, x_j)$, may not work for degenerate $U$-statistics as discussed in \citet{arcones1992bootstrap}.
%and \red{[??]}. 
 %, and $\alpha$ is a
%The following results is from \citet{huskova1993consistency}. 

This bootstrap test is summarized in Algorithm~\ref{alg:test}; 
its cost is $O(mn^2)$ where $n$ is the size of the sample $\{x_i\}$ and $m$ the bootstrap sample size. 
A more computationally efficient, but less statistically powerful, method can be constructed based on the following linear estimator: 
%It is possible use alternative method to decrease $O(n^2)$ to $O(n)$ and also derive qualities with Gaussian limit without the bootstrap, with the cost of losing some statistical efficiency. The simplest example of this is the linear estimator, defined as
\begin{align}
\label{equ:linear}
\hat \S_{lin}  = \frac{1}{\lfloor n/2\rfloor} \sum_{i=1}^{\lfloor n/2\rfloor} u_q(x_{2i-1}, x_{2i}), 
\end{align}
which %this estimator has
%has a zero-mean Gaussian limit under the null and can give a test that does not require bootstrap. 
%Therefore, the overall complexity of the linear estimator based test is only $O(n)$. 
has a zero-mean Gaussian limit under the null. 
This gives a test with only $O(n)$ time complexity: 
reject the null if $\hat \S_{lin} > \hat\sigma z_{1-\alpha}$, where $z_{1-\alpha}$ is the $1-\alpha$ quantile of the standard Gaussian distribution, and $\hat \sigma$ the standard deviation of 
$\{ u_q(x_{2i-1}, x_{2i})\}$. 
This test, however, tends to perform much worse than the $U$-statistic based test as we show in our experiments. 
%. and can give a test that does not require bootstrap. 
Further computation-efficiency trade-off between the linear- and $U$-statistic can be obtained by block-wise averaging; see \citet{ho2006two, zaremba2013b} for details. 
%Although this estimator is less efficient than the $U$-statistics, there are trade-off between linear and quradatic complexity based on block-wise averaging \citep{ho2006two, zaremba2013b}.

\section{Related Methods}
\label{sec:related}
%%We review several related works, ~\ref{sec:neyman} and
%We also review several related existing goodness-of-fit tests in Section~\ref{sec:neyman}. 
%We discuss review and discuss the connection with several related methods and topics, including Fisher divergence (Section~\ref{sec:fisher}), some classical goodness-of-fit tests (Section~\ref{sec:neyman}) and maximum discrepancy measure (Section~\ref{sec:mmd}). 
% and discuss the connection with 
We discuss the connection with Fisher divergence and maximum discrepancy measure (MMD). 
% which we discuss in Section~\ref{sec:fisher} and Section~\ref{sec:mmd}, respectively. 

\subsection{Connection with Fisher Divergence} \label{sec:fisher}
Fisher divergence, also known as Fisher information distance \citep{johnson2004information}, is defined as 
%Our kernelized Stein discrepancy has a close connection with the classical Fisher divergence, defined as 
\begin{align}
\label{equ:fisherDef}
\mathbb F(p,q)  = \E_{x} \big [   ||\nabla_x \log p (x) - \nabla_x \log q(x) ||^2_2 ],   %k(x,x') (\score_q(x')- \score_p(x) \bigg]. 
% \E_{x} \bigg [   ||\score_q(x) - \score_p(x) ||^2_2 ].  %k(x,x') (\score_q(x')- \score_p(x) \bigg]. 
\end{align}
that is, it is the $\mathcal L^2(p)$ norm of $\score_q(x) - \score_p(x)$. %under $p(x)$.  
%Fisher divergence can be related to the KL divergence through deBruijn's identity \citep{johnson2004information}. 
%Moreover, the convergence in Fisher divergence is a stronger form of convergence than that in KL, total variation and Hellinger distances  \citep{}. 
% ([23, Lemmas E.2 & E.3], [25, Corollary 5.1]).
An immediate connection is made by noting that $\mathbb F(p,q)$ 
can be treated as a special case of $\S(p,q)$ defined in
 \eqref{equ:dpEdsq} 
%An immediate connection is made by comparing $\mathbb F(p,q)$ with \eqref{equ:dpEdsq}; 
%\begin{align*}
%\S(p,q) 
%& = || \E_x (\score_q(x) - \score_p(x))k(x,\cdot) ||_\HH^2 \\%/ (\score_q(x')- \score_p(x) \bigg] \\
%& = \E_{x,x'} \bigg [   (\score_q(x) - \score_p(x))k(x,x') (\score_q(x')- \score_p(x) \bigg]. 
%\end{align*}
%in fact, note that by taking 
with  $k(x,x') = \ind[ x = x']$, or a RBF kernel with bandwidth $h\to 0$; in this sense, we can also think KSD as a kernelized version of Fisher divergence. 
We can establish the follow inequalities between $\mathbb F(p,q)$ and $\S(p,q)$: % can be established: 
\begin{thm}\label{thm:FS}
1) Following Definition \eqref{equ:dpEdsq} and \eqref{equ:fisherDef}, we have
%Under the assumption of Theorem~\ref{thm:kxx}, we have 
\begin{align}\label{equ:SF}
|\S(p,q)|  \leq \sqrt{\E_{x,x'\sim p} [k(x,x')^2]} \cdot \mathbb F(p,q). 
\end{align}
2) In addition, if $k(x,x')$ is positive definite and in the Stein class of $p$, 
%the conditions in Theorem~\ref{thm:kxx} and Theorem~\ref{thm:rkhs} hold 
and $\score_q - \score_p\in \H^d$, we have, for $p\neq q$, 
%under the assumption of Theorem~\ref{thm:kxx}, we have 
\begin{align}
\label{equ:FS}
%{ \S(p,q)} \geq  \mathbb   (\frac{|| \score_q - \score_p ||_{\H^d} }{|| \score_q - \score_p ||_{\H^d} })  F(p,q )  ~/~   || \score_q - \score_p ||_{\H^d} . 
\sqrt{ \S(p,q)} \geq  \mathbb F(p,q )  ~/~   || \score_q - \score_p ||_{\H^d} . 
\end{align}
\end{thm}
\begin{proof}
\eqref{equ:SF} is a simple result of Cauchy-Schwarz inequality, and \eqref{equ:FS} can be obtained by taking 
$f = (\score_q - \score_p)/|| \score_q(x) - \score_p(x) ||_{\H^d}$ in \eqref{equ:good}. 
See Appendix. 
\end{proof}
%Note that \eqref{equ:FS} is eq
%the $\S(p,q)$ defined in \eqref{equ:dpEdsq} reduces to Fisher divergence. %In addition, Fisher 
\eqref{equ:SF} suggests that the convergence in Fisher divergence is stronger than that in KSD. 
In fact, using Stein's method, \citet{ley2013stein} showed that Fisher divergence is stronger than most other divergences, including KL, total variation and Hellinger distances.

In addition, we can also represent $\mathbb F(p,q)$ in a variational form similar to \eqref{equ:good} but with $\vv f$ optimized over the unit ball of the intersection of the unit ball in $\mathcal L^2(p)$ space and the Stein class of $p$, which is larger than the ball of $\H^d$ and includes discontinuous, non-smooth functions; see Proposition~A.1 in Appendix.  
%that is, 
%
%In additional, the variational form \eqref{equ:good} under RKHS does not hold for Fisher divergence neither, because kernel $\ind[ x = x']$ does not corresponding to a RKHS. 
%Instead, one can show that Fisher divergence can be viewed as a variational optimization in the union ball of %the intersection of $L^2(p)$ and the Stein class of $q$. 
 %$(L^2(p))^d = L^2(p) \times \cdots \times L^2(p)$, 
% \begin{align}
% \label{equ:fff}
%\sqrt{\mathbb F(p,q)} 
%= \max_{\vv f \in (L^2(p))^d } \bigg\{ \sum_{\ell = 1}^d  \E_p[f_{\ell}(x) (\score_q^\ell(x) - \score_p^\ell(x))] \notag \\
%~~~~ s.t.~~~ \E_p[||\vv f(x)||^2] \leq 1 \bigg\}. 
%\end{align}
 %equipped with norm, %$||\vv f ||_{(L^2(p))^d} = \sum_{\ell\in [d]} \E_p []$
%under function set $\mathcal F = \{f \in  L^2(p) \cap \mathcal{S}(q), ~~~ ||f||_{2,p} \leq 1\}$ space, that is, 
\fullversion{
\begin{pro}
\label{pro:FisherVar}
Let $\mathcal F(p) = \mathcal L^2(p)  \cap \mathcal S(p)$, where $ \mathcal S(p)$ represents the Stein class of $p$, then we have 
\begin{align*}
 \sqrt{\mathbb F(p,q)} \geq  & \max_{\vv f \in \mathcal F(p)^d} \bigg\{ \E_{p}[\trace(\stein_q \vv f(x))] \\
& ~~~~~~~~~~~~~~~~~~~~~~~~~~~~~~~~~ s.t.~~~ \E_p[ || \vv f(x)||^2_2 ] \leq 1 \bigg\}.  
\end{align*}
and the equality holds when $\score_q - \score_p \in \mathcal F(p)^d$. 
\end{pro}
Note that $\mathcal L^2(p)$ is larger than the Stein class and RKHS, and includes discontinuous, non-smooth functions, and hence we need to ensure $\vv f$ is in the Stein class explicitly. 
}
%\begin{proof}Note that \eqref{equ:fff} and apply Lemma~\ref{lem:basic2}. \end{proof}
%We remark

%Nevertheless, we can rewrite Fisher divergence into a maximum Stein discrepancy in the intersection of $(L^2(p))^d$ and the Stein class of $p$ if $\score_q - \score_p$ is in the Stein class of $p$. 

Despite the connections, the critical disadvantage of Fisher divergence compared to KSD 
is that the computationally convenient representation \eqref{equ:dpEu} no longer holds for Fisher divergence, because its corresponding kernel $\ind[ x = x']$ is not differentiable. 
%In 
%$\mathcal F(p,q)$ compared to $\S(p,q)$ is that it does not 
%is that 
%However, because the kernel $\ind[x=x']$ is not differentiable, 
%the computational convenient form \eqref{equ:dpEu} no longer holds for Fisher divergence, %because the kernel $\ind[ x = x']$ is not differentiable,
%making it impossible to estimate $$ using $U$-statistics like kernelized Stein discrepancy. 
Therefore, we can not estimate $\mathbb F(p,q)$ using the $U$-statistic in \eqref{equ:SU}. Instead, estimating $\mathbb F(p,q)$ seems to be substantially more difficult. % than estimating $\S(p,q)$. 
To see this, note that %this can be seen by noting that 
 \begin{align}
\mathbb F(p,q) 
&=  \E_{x\sim p} [ || \score_q(x)||^2_2  - 2 \score_p(x)^\top \score_q(x)  + ||\score_p(x)||_2^2 ] \notag \\
&= \E_{x\sim p}  [ \phi_q(x) ]  ~+ ~ \E_{x\sim p} ||\score_p(x)||_2^2, \label{equ:fishergood}
\end{align}
where $\phi_q(x) = ||\score_q(x)||^2_2  + 2 \trace(\nabla_x \score_q(x))$ and is obtained by applying Stein's identity on the cross term. 
Note that although the first term $\E_{x\sim p}  [ \phi_q(x) ]$ in \eqref{equ:fishergood} can be estimated by the empirical mean of $\phi_q(x)$ (which only depends on $\score_q$) under sample $\{x_i\}\sim p$, 
%empirically evaluated based on $\score_q(x)$ and sample $\{x_i\} \sim p$, 
the second term $\E_{x\sim p} ||\score_p(x)||_2^2$ is more difficult to estimate, since it depends on the score function $\score_p(x)$ of the unknown $p(x)$, and hence requires a kernel density estimator for $p(x)$; see \citet{hall1987estimation, birge1995estimation}. 
We should point out that similar difficult ``constant" terms appear in other common discrepancy measures such as KL divergence and $\alpha$-divergence \citep[e.g.,][]{Krishnamurthyicml15}. 
For this reason, KSD provides a much more convenient tool for goodness-of-fit tests than the other discrepancies. % measures. 
% $\mathbb{KL}(p,q)= \mathbb{H}(p) + \E_p \log q$, where the entropy term $\mathbb{H}(p)$ plays a similar 
 % (the first term requires a density estimator $$)
%Typical divergence evaluation requires certain form of density estimation, and often subject to non-parameteric convergence rate; see for example, \citet{Krishnamurthyicml15, Krishnamurthyaistats15, Krishnamurthynips15, kerkyacharian1996estimating, laurent1996efficient, birge1995estimation} for the discussion of $\alpha$ divergence and $f$ divergence. 

%and is more difficult to estimate. 
%In fact, it involves estimating the integrated density derivatives, for which it is often necessary to 
%estimating forms of such type requires 
%often requires certain 
%construct a  nonparametric density estimator on $p(x)$ and hence inefficient in high dimensional \citep{hall1987estimation}. 
 %Fisher divergence 
 
%it is critical to point out that the computational convenient form \eqref{equ:dpEu} no longer holds for Fisher divergence, because the kernel $\ind[ x = x']$ is not differentiable,
%making it impossible to estimate $$ using $U$-statistics like kernelized Stein discrepancy. 
%there exists no form of Fisher divergence analogues to \eqref{equ:dpEU} because the kernel $\ind[ x = x']$ is not differentiable.
% and hence we can not estimate $\mathbb F(p,q)$ using $U$-statistics like $\S(p,q)$. %kernelized Stein discrepancy.    For this reason, it is less convenient to use $\mathbb F(p,q)$ for goodness-of-fit tests than KSD $\S(p,q)$.

 %\red{[]}Therefore
% Because of form \eqref{equ:fishergood}, 
Meanwhile, Fisher divergence still has the advantage of being
independent of the normalization constants of $p$ and $q$, 
and provides a useful tool in cases when it does not require evaluating the term $
\score_p (x)$.
 %Fisher divergence can be used for
For example,  Fisher divergence has been widely used for parameter estimation, finding the optimal $q(x)$ that best fits a sample $\{x_i\}$ by minimizing $\mathbb F(p,q)$; 
this yields the score matching methods developed in both parametric \citep{hyvarinen2005estimation, lyu2009interpretation} and non-parametric \citep{sriperumbudur2013density} settings.  

\fullversion{
\subsection{Classical Goodness-of-fit Tests }
\label{sec:neyman}
%Likelihood based tests often explicitly assumes a set of parametric alternatives $p(x,\theta)$, such that $q(x) = p(x,\theta)$ if and only if $\theta = \theta_0$. This problem is then turn into testing whether $\theta = \theta_0$. Assume $\hat \theta$ be the maximum likelihood of $p(x, \theta)$ on sample $\{x_i\}$.
Likelihood based methods test whether $\{x_i\}$ is drawn from $q(x)$ by
considering a set of parametric alternatives $q(x,\theta),~~ \theta \in \Theta$, such that $q(x) = q(x,\theta)$ if and only if $\theta = \theta_0$. 
%This problem is then turn into testing whether $\theta = \theta_0$. 
Let $\hat \theta$ be the maximum likelihood estimator on sample $\{x_i\}$. 
Three typical test statistics are commonly used \citep[e.g.,][]{lehmann2006testing}: %constructed: 
\begin{align*}
& \text{Likelihood Ratio:} &&L_n = 2\log ( q(X | \hat \theta)/ q(X | \theta_0)) \\
& \text{Wald Score:}  &&   W_n = 
 (\hat\theta - \theta_0)^\top H(\hat\theta) (\hat \theta - \theta_0) \\
%n (\hat\theta - \theta_0)B(\hat\theta)^{-1} (\hat \theta - \theta_0) \\
& \text{Rao Score:}  &&   
R_n = S(\theta_0)^\top H(\theta_0)^{-1} S(\theta_0)%\nabla_\theta \log p(x|\theta_0)
\end{align*}
where $\log q(X|\theta)=\prod_i\log q(x_i|\theta)$, 
$S(\theta) =   \nabla_\theta \log q(X|\theta)$ and 
$H_n = - \nabla_{\theta,\theta}\log q(X | \theta)$, respectively. %Fisher information.  
%where $B_n(\theta) = -n^{-1} H_n^{-1} $ and $H$ is the Hessian matrix of log likelihood. 
All of the three statistics have a $\chi^2$ asymptotic distribution with the degree of freedom equaling the number of parameters $\theta$.  
Compared to our method, these methods require defining a parametric alternative, and more critically, need to estimate $\hat \theta$, or evaluate the likelihood or its gradient, which are computationally intractable in many practical cases. 
%http://people.hss.caltech.edu/~mshum/stats/lect8.pdf
%These methods requires to specify a set of parametric alternatives. But the major difficulty is computational. 
%Since it becomes difficult to estimate $\hat\theta$, or evaluate $\log p(x | \theta)$ and its gradient. 

\myempty{
The weighted quadratic tests can be treated as special cases of Rao score test; it considers the following set of alternatives, 
$$q(x, \theta) = q(x)\exp(\sum_{j=1}^J \theta_j \psi_j(x) ) /  Z(\theta)),$$
where $\{\phi_j(x)\}$ is a set of orthonormal ``perturbation" directions with $q(x)$, that is, $ \int \psi_j(x) \psi_{j'} (x) q(x) dx = \ind[j = j']$, and satisfies
$\E_q(\psi_j(x)) = 0$, 
and $Z(\theta)$ the normalization constant. %The Rao score 
It then consider the weighted quadratic statistics
%The null hypothesis assets that $\theta = 0$, and 
%Neyman's smooth test considers the following statistics 
$$
\hat T  =  n \sum_{j=1} a_j \hat \E(\psi_j(x))^2,  ~~~ \text{where}~~~ \hat \E(\psi_j(x)) = \frac{1}{n}\sum_i \psi_j(x_i),% -\E_q(\psi_j(x)), 
$$
which has an asymptotic distribution of form $\sum_j a_j Z_j^2$, where $Z_j$ are i.i.d. standard Gaussian random variables. 
Weighted quadratic tests include many methods as special cases, 
when $a_j=1$ for a finite set $j = 1,\ldots, J$, it reduces to Neyman's smooth test. 
It is 

%where $\hat \E(\cdot)$ represents the empirical expectation under sample $\{x_i\}\sim p$. 
To avoid the need of calculating $\E_q(\psi_j(x))$, Neyman's test takes $\psi_1 =1$ and restricts $\{\psi_j\}$ to be a series of orthonormal functions w.r.t. $q(x)$, that is, 
$ \int \psi_j(x) \psi_{j'} (x) q(x) dx = \ind[j = j']$. Then we have $\E_q(\psi_j(x)) = 0$, and hence the testing statistics reduces to $T_j =  \hat \E(\psi_j(x))$, and 
$n \hat T $ can be shown to the $\chi^2$ distribution with $J$ degrees of freedom. 
\citet{neyman1937smooth} showed that this test asymptotically maximize minimum (and average) power against the parametric alternative $\{q(x|\theta), ~\theta \neq 0\}$. 
Neyman's smooth test was originally considered under uniform null, under which Legendre polynomials are considered.  
However, Neyman's test only allows a finite number of $J$ and casts the importance problem of selecting the optimal $J$; 
in practice, $J$ is often restricted to small number such as $J = 2,3,4$, or selected by data driven approaches such as those based on BIC \citep[see e.g.,][]{fan1996test, ledwina1994data}. 
}

Neyman's smooth test  \citep{neyman1937smooth} can be treated as a special case of Rao score test; it considers the following set of alternatives, 
%explicitly consider a set of ``smooth" parametric alternatives, 
$$q(x, \theta) = q(x)\exp(\sum_{j=1}^J \theta_j \psi_j(x) ) /  Z(\theta)),$$
where $\{\psi_j(x)\}$ is a set of zero-mean, orthonormal ``perturbation" directions w.r.t. $q(x)$, that is, 
$\E_q(\psi_j(x)) = 0$ and
$ \int \psi_j(x) \psi_{j'} (x) q(x) dx = \ind[j = j']$, 
and $Z(\theta)$ the normalization constant. In this case, the Rao score reduces to
%where $\{\phi_j(x)\}$ are the ``perturbation" directions and $Z(\theta)$ the normalization constant. 
%The null hypothesis assets that $\theta = 0$, and 
%Neyman's smooth test considers the following statistics 
$$
\hat T  =  n \sum_{j=1}^J  \hat \psi_j ^2,  ~~~ \text{where}~~~ \hat  \psi_j = \frac{1}{n}\sum_{i=1}^n \psi_j(x_i),% -\E_q(\psi_j(x)), 
%\hat T  =  n \sum_{j=1} a_j \hat \E(\psi_j(x))^2,  ~~~ \text{where}~~~ \hat \E(\psi_j(x)) = \frac{1}{n}\sum_i \psi_j(x_i),% -\E_q(\psi_j(x)), 
%\hat T  =  \qrt{n}\sum_{j=1}^J \Psi_j^2, ~~~ \text{where}~~~ \Psi_j =  \hat \E(\psi_j(x)) -\E_q(\psi_j(x)), 
$$
which has an asymptotic $\chi^2$ distribution with $J$ degrees of freedom. 
%where $\hat \E(\cdot)$ represents the empirical expectation under sample $\{x_i\}\sim p$. 
%To avoid the need of calculating $\E_q(\psi_j(x))$, Neyman's test takes $\psi_1 =1$ and restricts $\{\psi_j\}$ to be a series of orthonormal functions w.r.t. $q(x)$, that is, 
%$ \int \psi_j(x) \psi_{j'} (x) q(x) dx = \ind[j = j']$. Then we have $\E_q(\psi_j(x)) = 0$, and hence the testing statistics reduces to $T_j =  \hat \E(\psi_j(x))$, and 
%$n \hat T $ can be shown to the $\chi^2$ distribution with $J$ degrees of freedom. 
\citet{neyman1937smooth} showed that this test asymptotically maximizes the minimum (and average) power against the parametric alternative $\{q(x|\theta), ~\theta \neq 0\}$; 
the disadvantage, however, is that it only allows a  finite number $J$ of perturbation directions;
in practice, $J$ is often restricted to a small number $J = 3$ or $4$, or selected by data driven approaches such as those based on BIC \citep[see e.g.,][]{fan1996test, ledwina1994data}. 
%Neyman's smooth test was originally considered under uniform null, under which Legendre polynomials are considered.  
%However, Neyman's test only allows a finite number of $J$ and casts the importance problem of selecting the optimal $J$; 
%in practice, $J$ is often restricted to small number such as $J = 2,3,4$, or selected by data driven approaches such as those based on BIC \citep[see e.g.,][]{fan1996test, ledwina1994data}. 

The weighted quadratic tests generalize Neyman's approach by assigning a nonnegative weight $a_j$ on each perturbation direction $\psi_j(x)$, making it possible to use an infinite number of directions, %that is, 
\begin{align*}
%\label{equ:wquadratic}
\hat T_w = n \sum_{j=1}^\infty a_j \hat \psi_j^2,   
\end{align*}
whose asymptotic distribution is $\sum_{j=1}^\infty a_j Z_j^2$ with $Z_j$ being i.i.d. standard Gaussian random variables. 
%where $a_j$ are non-negative numberers, and 
The challenge, however, is to choose $a_j$ such that the infinite sum can be explicitly evaluated. 
This can be done in simple cases such as when $q$ is uniform,
leading to empirical distribution function (EDF) based tests such as Cramer-von Mises test and Anderson-Darling statistics for difference choices of $\{a_j\}$ and $\{\psi_j\}$.
%It is sometimes possible to choose $a_j$ such that the infinite sum can be explicitly evaluated. 
%For example, \citet{lehmann2006testing} (page 611-612) showed that Cramer-von Mises test statistics $n \int_0^1 (\hat F_n - F_q(x)) dF_q(x)$, where $F_q$ is the c.d.f of $q$ and $\hat F_n$ the empirical c.d.f. of $\{x_n\}$ can be treated as a weighted quadratic statistics applied on uniform distribution ($F_q(x)$ is has a uniform distribution under $x\sim q$) with $a_j = \frac{1}{\pi^2 j^2}$ and $\psi_j(x) = \sqrt{\pi j x}$, and similarly 
%the Anderson-Darling  statistics  $n \int_0^1 \frac{[\hat F_n(x) - x]}{x(1-x)} dx $ can be treated as $\psi(x)$ being normalized Legendre polynomial and $a_j = 1/j(j+1)$. In fact, any statistics whose limit distribution is a (infinite) weighted sum of $\chi^2$ random variables, including our method, can be trivially viewed as a special case of weighted qudratic test \eqref{equ:wquadratic}. 
However, we are not aware of general principles like our method that construct infinite sums for generic distributions $q(x)$. %\red{[double check?]}. 

Our test can be treated using a set of \emph{non-orthonormal} perturbation directions $\stein_q e_j(x)$: 
$$
q(x, \theta) = q(x) \exp\big[\sum_j \theta_j^\top  \stein_q e_j(x) \big], 
$$
with test statistics $\hat\S_u = \sum_j \lambda_j (||\stein_q e_j(x)||^2-1)$. 
Note that orthonormalizating $\stein_q e_j(x)$ w.r.t. $q(x)$ gives   
the eigenfunctions $\phi_j(x)$ of $u_q(x,x')$, and derives the asymptotic distribution in \eqref{equ:chi2}. 
%However, it is much more computationally tractable to consider $ \stein_q e_j(x)$.% instead of 
%note that $\stein_q e_j(x)$ are not orthogonal, but one can alternatively view it through the orthornormal eigenfunctions of $u_q(x,x')$. 
%Our $V$-statistis has a form of $\hat\S_v = \sum_j \lambda_j ||\stein_q e_j(x)||^2$, but the difficulty is that it requires additional approximation to get the trace  $\sum_j \lambda_j c_j  = \trace(u_q(x,x))$
% can be viewed as  
%special choices of $a_j$ and $\psi$
%Data-driven approaches for
%and it is a major challenging to find 
}

\subsection{Maximum Mean Discrepancy \& Two-sample Tests} \label{sec:mmd}
Closely related to goodness-of-fit tests are  two sample tests, which test whether two i.i.d. samples $\{x_i\}$ and $\{y_i\}$ are drawn from the same distribution. In principle, one can turn a goodness-of-fit test into a two sample test by drawing $\{y_i\}$ from $q(x)$. However, it is often difficult to draw exact i.i.d. samples for practical models, and furthermore MCMC sampling may be computationally expensive, suffer from the convergence problems, and introduce undesired correlations. 
When the MCMC approximation is poor, the two sample test would reject the null even when $p = q$ (inconsistent in level). 
%may correlations in the sample. 
%when $q(x)$ is complicate, it may be difficult to drawn 
%it is often the case that it is difficult to 

Maximum Mean Discrepancy \citep{gretton2012kernel} is a nonparametric distance measure widely used for two sample tests, defined as
%between distributions, defined as % It is defined as
\begin{align*}
{\mathbb M}(p, q) = \max_{h \in \H } \big \{  \E_{p}[h(x)] - \E_q[h(x)]   ~~~s.t.~~~ || h||_\H \leq 1\big  \}, 
\end{align*}
where $\H$ is the RKHS of kernel $k(x,x')$. 
\citet{gretton2012kernel} showed that $\mathbb M (p,q)$ can be rewritten into %represented as 
\begin{align}\label{equ:mmd}
\mathbb M(p,q) = 
\E[k(x,x') + k(y,y') - 2 k(x,y') ],%- k(x', y)], 
%\E_{x,x'; y,y'} [k(x,x') + k(y,y') - k(x,y') - k(x', y)], 
%\E_{x,x'\sim p; y,y'\sim  q} [k(x,x') + k(y,y') - k(x,y') - k(x', y)], 
%\E_{x\sim p,x' \sim p} k(x,x') +  \E_{x\sim q,x' \sim q} k(x,x') -2  \E_{x\sim p,x' \sim q} k(x,x'). 
\end{align}
where $x,x'$ and $y, y'$ are i.i.d. draws from $p$ and $q$, respectively. 
Therefore, $\mathbb M(p,q)$ can be empirically estimated based on sample $x_i \sim p$ and $y_i \sim q$ using $U$- or $V$- statistics, making it a useful tool for two sample tests.  %that tests. %whether two samples are drawn from a same distribution. 
%This is different from KSD in that KSD is better estimated with 
Our KSD, on the other hand, is better estimated with sample $x_i \sim p$ and the score function $\score_q$ and hence suitable for goodness-of-fit tests. 
Finally, by comparing \eqref{equ:mmd} with \eqref{equ:dpEu} and noting that $\E_{x\sim q}[u_q(x,x')]=0$,
 we can consider KSD as a special MMD with kernel $u_q(x,x')$;
 the key difference is that kernel $u_q(x,x')$ depends on $q$, %-specific
making KSD asymmetric. 
%we remark that MMD is symmetric, while KSD is asymmetric.  
%We remark that it is also possible to use $\mathbb M(p,q)$ for goodness-of-fit tests by drawing sample $y_i$ from $q$, and turn goodness-of-fit into a two sample test. This, however, may be computationally expensive when $q$ is complicated, and only approximate sampling based on MCMC is possible. tests based on MCMC has to be risk of being inconsistent in size, because when MCMC does not converge and hence $y_i \sim q$ is not perfectly drawn from $q$, it will have high probability to reject the null even when $p = q$. 
%
%\red{[todoXXX]}
\myempty{
MMD and KSD can be related to each other using Stein's equation, or the inverse of the Stein's operator. For any $h\in \H$, let $f = \stein_q^{-1} h$ be the solution of 
$$
 h(x) - \E_q[h(x)] = \stein_q f_h(x). 
$$
Denote by $||\stein_q ||_\H $ and $||\stein_q^{-1}||_\H$ the operator norm of $\stein_q$ and $\stein_q^{-1}$, respectively, that is, 
$$
|| \stein_q || = \max_{f } || \stein_q f ||_\H / || f ||_\H, ~~~ 
|| \stein_q^{-1} || = \max_{f} || \stein_q^{-1} f ||_\H / || f ||_\H. 
$$
Then we have 
\begin{pro}
$$
\frac{1}{||\stein_q ||_\H }  \S(p,q)  \leq \mathbb M(p,q) \leq || \stein_q^{-1}||_\H \S(p,q) 
$$
\end{pro}
\begin{proof}
Let $\stein_\H $ be the set of $f_h$ for $h\in \H$ and $|| h||_\H \leq 1$,that is, 
$$
\stein_\H^{-1} = \big \{  f_h \colon  h(x) - \E_q[h(x)] = \stein_q f_h(x), ~~ h\in \H, ~ || h ||_\H \leq 1  \big \}, 
$$
then 
$$
\mathbb M(p,q) = \max_{ f \in \stein_\H^{-1}}   \big \{ \E_p [\stein_q f (x)]  \big \}. 
$$

Or equivalently,  
$$
\S(p,q) = \max_{h \in \stein_\H} \big\{  \E_p[h(x)] - \E_q[h(x)]  ~~~ s.t.~~~  ||\stein_q^{-1} h ||_\H \leq 1\big\}. 
$$
where $\stein_q^{-1}$ represents the inverse of the Stein's operator, and $\stein_q^{-1} h $ is the solution $f_h$ of 
$ h(x) - \E_q[h(x)] = \stein_q f_h(x).$
Denote by $|| \stein_q ||_\H$ the norm of linear operator $\stein_q$ that satisfy $|| \stein_q ||_\H =  \max_{f \in \H \colon || f||_\H \leq 1}\{ || \stein_q f  ||_\H  \}$
Therefore, 
$$
\S(p,q) \leq   ||\stein_q ^{-1}|| \mathbb   M(p,q)
$$

\begin{align*}
\S(p,q) 
= \max_{h \in \stein_\H} \big\{\frac{1}{||\stein_q^{-1} h ||_\H}  \E_p[h(x)] - \E_q[h(x)]  \big\}
\geq 
 \max_{h \in \stein_\H} \big\{\frac{1}{||\stein_q^{-1}||_{\H}  || h ||_\H}  \E_p[h(x)] - \E_q[h(x)]  \big\} 
 = \frac{1}{|| \stein^{-1}_q||_\H} \mathbb M(p,q). 
\end{align*}

\begin{align*}
\mathbb M(p,q) 
= \max_{h \in \H} \big\{\frac{1}{|| h ||_\H}  \E_p[h(x)] - \E_q[h(x)]  \big\}
\geq 
 \max_{h \in \stein_\H} \big\{\frac{1}{||\stein_q||_{\H}  || \stein_q^{-1}h ||_\H}  \E_p[h(x)] - \E_q[h(x)]  \big\} 
 = \frac{1}{|| \stein_q||_\H}  \S(p,q). 
\end{align*}
Therefore, 
$$
\frac{1}{||\stein_q ||_\H }  \S(p,q)  \leq \mathbb M(p,q) \leq || \stein_q^{-1}||_\H \S(p,q) 
$$
\end{proof}
}

%Therefore, we have
%$$\H (p, q) = \max_{f \in \H\colon || f||_\H \leq 1} $$

\myempty{
\section{Extensions to Discrete Models}
\label{sec:disc}
Our method requires densities to be continuously differentiable; this, however, is not a substantial restriction, since 
any density can be approximated arbitrarily well by a ``smoothed", infinitely differentiable function via the Weierstrass transforms, that is, by convolving with a Gaussian densities \citep[e.g.,][]{zayed1996handbook}. 
Specially, given that a un-smooth density $q(x)$, we consider its smoothed version $q_t(x) =  \int  \phi(x + t \epsilon )q(x) d\epsilon,$ and correspondingly $p_t(x)$ for $p(x)$. 
Then given a sample $\{x_i\} \sim p(x)$, we generated a new sample $\{y_i\}$ by $y_i = x_i + t \epsilon$, where $\epsilon$ is drawn from standard Gaussian. 
The testing between $q$ and $\{x_i\}$ is then transformed to that between $q_t(x)$ and $\{y_i\}$. We should use a small $t$ to ensure that $q_t(x)$ is close to $q(x)$ and hence there is no significant lose of power in the test.

Remarkable, this idea also applies even for discrete models in which case $q(x)$ denotes a probability table under a finite set $\X$. In this case we add Gaussian noise to the discrete variables to turn it into continuous variables
However, directly implementing this idea has two difficulties: 

1. The choice of noise variance is critical; a small variance causes poor estimation of the discrepancy, while a large variance causes lose of discriminant power (distributions are more similar when perturbed with large gaussian noise).  %power ()

2. We need to calculate the score function for $x$, with probability density $p(x) = \sum_h p(x, h) = \sum_{h} \phi(x -h 0, \Sigma) p(h)$, this can be difficult when $p(h)$ is intractable (e.g., when $p(h)$ is a Ising model). 
%One major difficulty for implementing this idea is to ensure that the score function of $q_t(x)$ is computational tractable when $q(x)$ is high dimension models such as discrete graphical models in which case $\X$.  

In the following, we show that by using a smart choice of Gaussian noise, one can apply this idea on Ising models to avoid computational intractability. 
%However, we show that it is possible to choose $\Sigma$ smartly to avoid these two problems. We give several examples here, including Ising model and restricted Boltzmann machines. 
%Our method does not directly applies to discrete models. However, one general approach is to add Gaussian noise to the discrete variables to turn it into continuous variables (effectively doing a convolution with the original distribution). To be concrete, assume $h$ is a discrete random variable with probability $p(h)$, we can construct a continuous variable by adding Gaussian noise $x = h + \normal(0, \Sigma)$. Directly applying this approach has two difficulties: 
%1. The choice of noise variance is critical; a small variance causes poor estimation of the discrepancy, while a large variance causes lose of discriminant power (distributions are more similar when perturbed with large gaussian noise).  %power ()
%2. We need to calculate the score function for $x$, with probability density $p(x) = \sum_h p(x, h) = \sum_{h} \phi(x -h 0, \Sigma) p(h)$, this can be difficult when $p(h)$ is intractable (e.g., when $p(h)$ is a Ising model). 
%However, we show that it is possible to choose $\Sigma$ smartly to avoid these two problems. We give several examples here, including Ising model and restricted Boltzmann machines. 

\begin{mydef}[Ising Model]
Assume $\X = \{-1, +1\}^d$, an Ising graphical model defines the following probability for $x\in \X$, 
\begin{align}
\label{equ:ising}
q(x)  = \exp(\frac{1}{2}x^\top A x + b^\top x - \log Z), ~~~ Z = \sum_{x\in \X}  \exp(\frac{1}{2}x^\top A x + b^\top x). 
\end{align}
where $Z$ is the normalization constant called partition function, and is generally intractable to calculate, making directly evaluating the likelihood of $q(x)$ infeasible. 
We will assume $A$ is positive definite without lose of generality, since modifying the diagonal elements of $A$ does not change the probability. 
\end{mydef}

\begin{pro}
Assume $A$ is positive definite. 
For any discrete probability $p(x)$ and $q(x)$, $x\in \X$, denote by $\tilde p_{b,A}(y)$ and $\tilde q_{b,A}(y)$ the continuous density of $y = b + Ax + A^{1/2} \epsilon$, $y\in \RR^d$, where $x$ is drawn from $p$ and $q$, respectively, and $\epsilon$ is a $d\times1$ i.i.d. standard Gaussian noise. 
Then 

1).  $\tilde p_{b,A}(y) = \tilde q_{b,A}(y), ~ \forall y \in \RR^{d}$ ~ if and only if ~ $p(x) = q(x), ~ \forall x \in \X$. 

2). If $q(x)$ is the Ising model as defined in \eqref{equ:ising}, we have 
\begin{align}
\tilde q_{b,A}(y) \propto \exp[-\frac{1}{2}(y - b) A^{-1} (y - b)] \prod_{\ell\in [d]}[\exp(y^\ell) + \exp(-y^\ell)], 
\label{equ:qbA}
\end{align}
and hence 
$$\score_{\tilde q_{b,A}}(y) =\nabla_y \log \tilde q_{b,A}(x)=  - A^{-1} (y - b) 
+ \sum_{\ell\in [d]} \frac{\exp(y^\ell)-\exp(-y^\ell))}{\exp(y^\ell)+\exp(-y^\ell)}. 
$$
%+ \sum_{\ell\in [d]}  \tanh(y^\ell),$$ 
%where $\tanh(y )= \frac{\exp(y)-\exp(-y))}{\exp(y)+\exp(-y)}$.
\end{pro}
\begin{proof}
1). %We just need to show that $\tilde p_{b,A}= \tilde q_{b,A}$ implies $p = q$; 
Note that $\tilde p_{b,A}$, $\tilde q_{b,A}$ are Gaussian mixtures, centered at $b+Ax$, $x\in \X$ and with mixture weights $p(x)$, $q(x)$, respectively. Note that 
there is a one to one map from $x\in \X$ to the mixture centers since $A$ is assumed to invertible. Therefore, $\tilde p_{b,A}$ and $\tilde q_{b,A}$ equal if and only if their mixture weights $p(x)$ and $q(x)$ equal. 

2). When $q(x)$ is the Ising model in \eqref{equ:ising}, then the joint distribution of $(x,y)$ is
\begin{align*}
q(x,y) 
& \propto \exp(-\frac{1}{2}(y -  b - Ax)^{T} A^{-1} ( y - b - Ax)) \cdot  \exp(\frac{1}{2}x^\top A x + b^\top x - \log Z)  \\
& \propto \exp( -\frac{1}{2} (y-b)^{T} A^{-1} (y -b) ) \cdot \exp(x^\top y). 
\end{align*}
We then get the result by noting that $\tilde q_{b,A}(y) = \sum_{x\in \X} q(x,y)$.  
\end{proof}

Therefore, Given sample $\{x_i\} \sim p$ and and Ising model $q(x)$, and can generate $\{y_i\} \sim \tilde p_{b,A}$ via $y_i = b + A  x_i + A^{1/2} \epsilon_i$ and test the equality of $\tilde p_{b,A}$ and $\tilde q_{b,A}$. 
The algorithm is summarized in Algorithm~\ref{alg:goodising}. 

\begin{algorithm}[tb] % 
\caption{Goodness-of-fit Test for Ising Model}  \label{alg:goodising}
\begin{algorithmic}
\STATE \emph{Input:} sample $\{x_i\} $ and Ising model $q(x) = \exp(-\frac{1}{2} x^\top A x + b^\top x)$.   
\STATE \emph{Goal:} Test: $H_0$: $\{x_i\}$ is drawn from $q $  ~~~v.s ~~   $H_1$: $\{x_i\}$ is not drawn from $q$.  
\STATE 1. Generate sample $\{y_i\} $ via $y_i = b + A  x_i + A^{1/2} \epsilon_i$, where $\epsilon_i$ is i.i.d. standard Gaussian.  
\STATE 2. Perform kernelized Stein test between $\{y_i \}$ and $\tilde q_{b,A}$ as defined in \eqref{equ:qbA}. 
\end{algorithmic}
\end{algorithm} 
}

\begin{figure*}[ht]
   \centering
   \begin{tabular}{cccc}%{llll}
   \hspace{-.5em}
      \includegraphics[width=.22\textwidth, trim={0 1cm 0 0},clip]{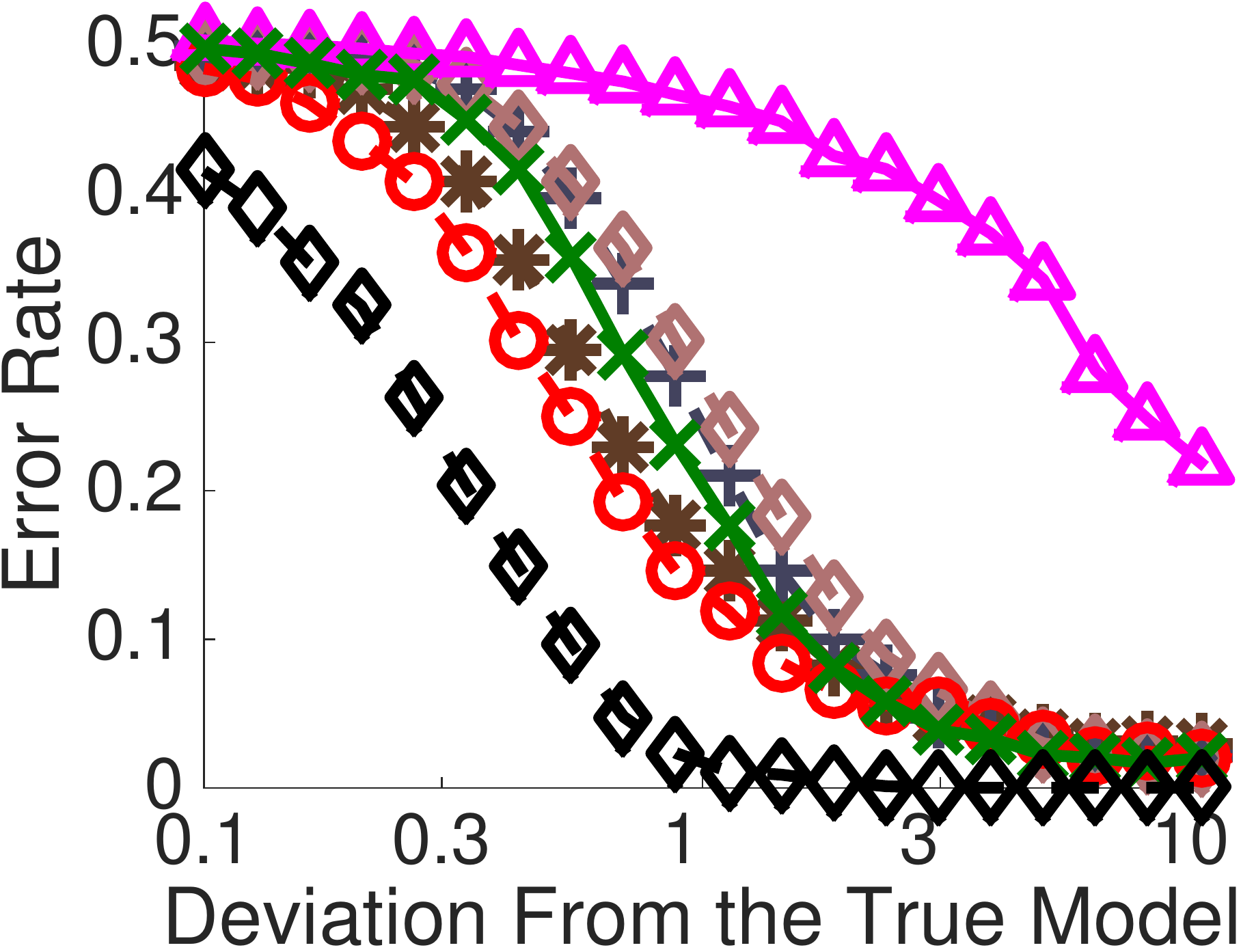} & \hspace{-1.2em}
   \includegraphics[width=.22\textwidth, trim={0 1cm 0 0},clip]{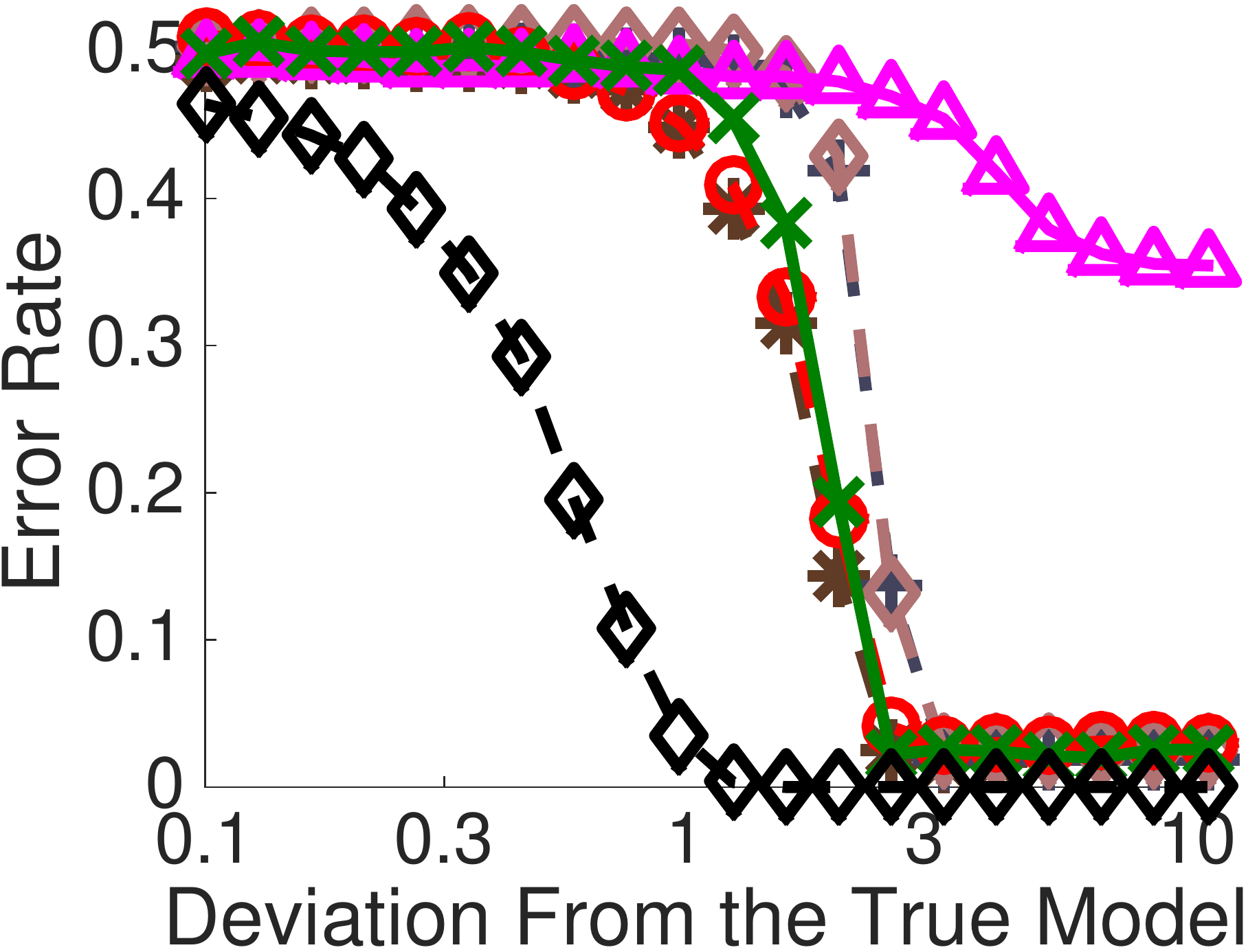}   &\hspace{-1.5em}
   \includegraphics[width=.22\textwidth, trim={0 1cm 0 0},clip]{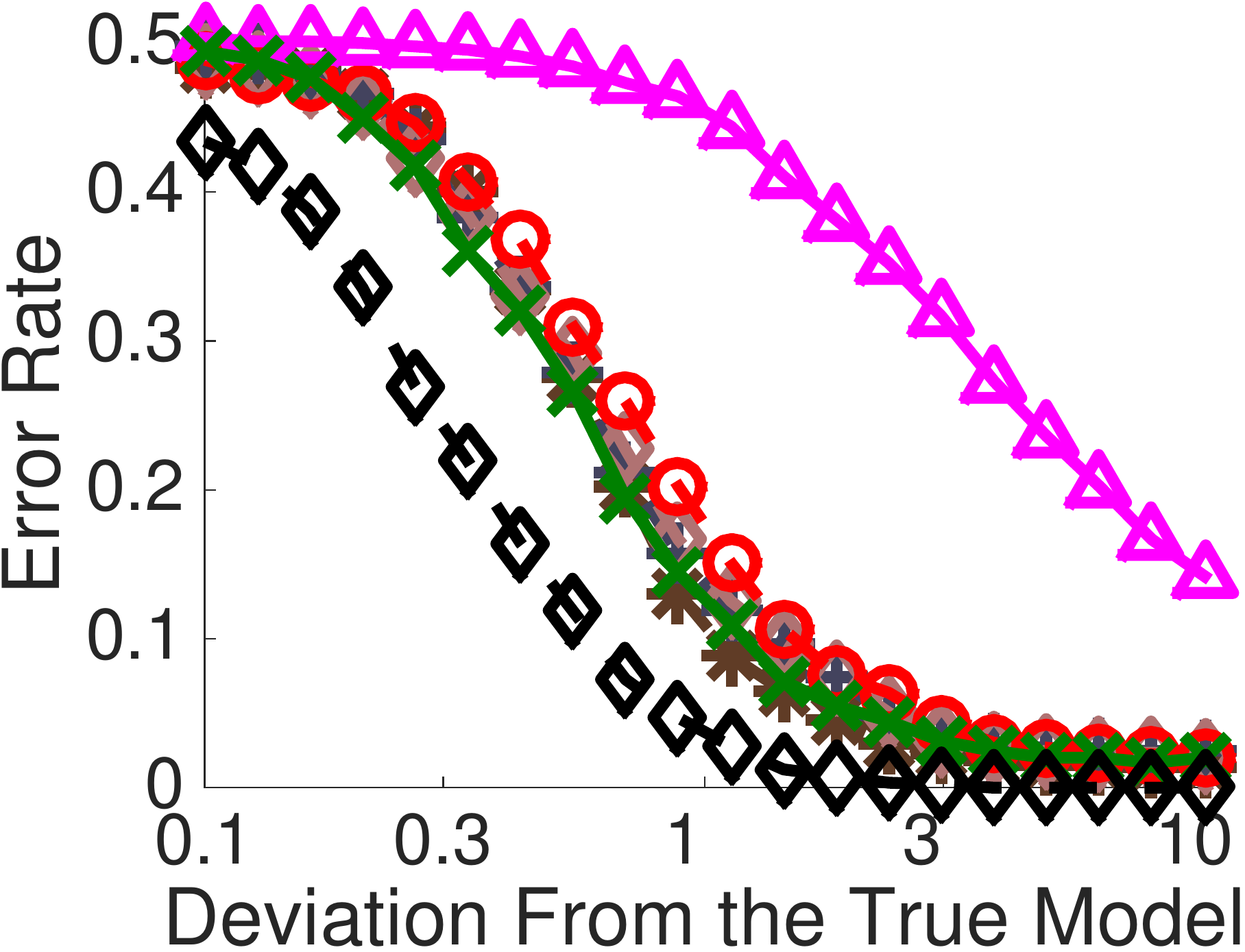}   &
   \hspace{-4em}
\raisebox{1em}{\includegraphics[width=.18\textwidth]{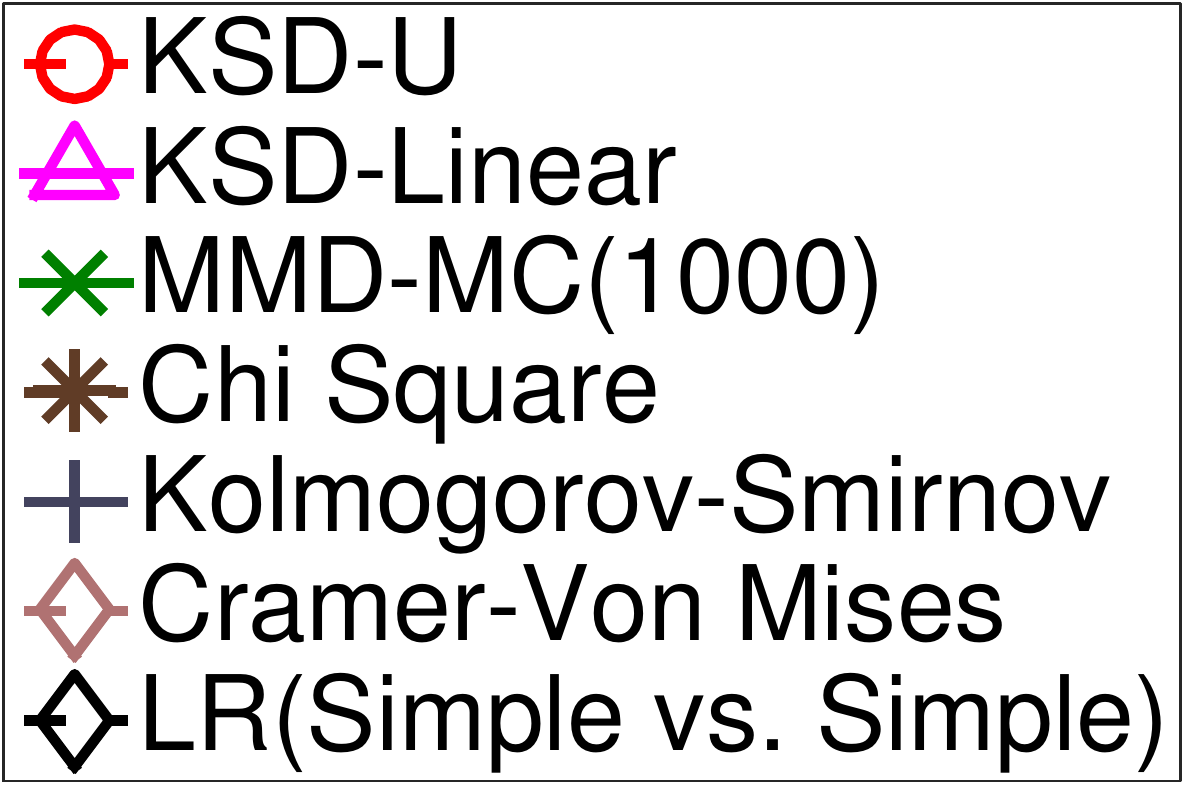}}   
   \\[-.7em]
     {\scalebox{1.1}{\fontfamily{phv}\selectfont \scriptsize Perturbation Magnitude $\sigma_{per}$ }}   &
{\scalebox{1.1}{\fontfamily{phv}\selectfont \scriptsize Perturbation Magnitude $\sigma_{per}$ }}  &
{\scalebox{1.1}{\fontfamily{phv}\selectfont \scriptsize Perturbation Magnitude $\sigma_{per}$ }}       &  \\ [.2em]  
     (a) Perturbation on Mean  & (b)  Perturbation on Variance & (c) Perturbation on Weights &  \\[0.2em]
%     (a) Mean Perturbed  & (b) Variance Perturbed  & (c) Weights Perturbed &  \\[0.2em]
   \hspace{-.5em}
   \includegraphics[width=.22\textwidth, trim={0 1.75cm 0 0},clip]{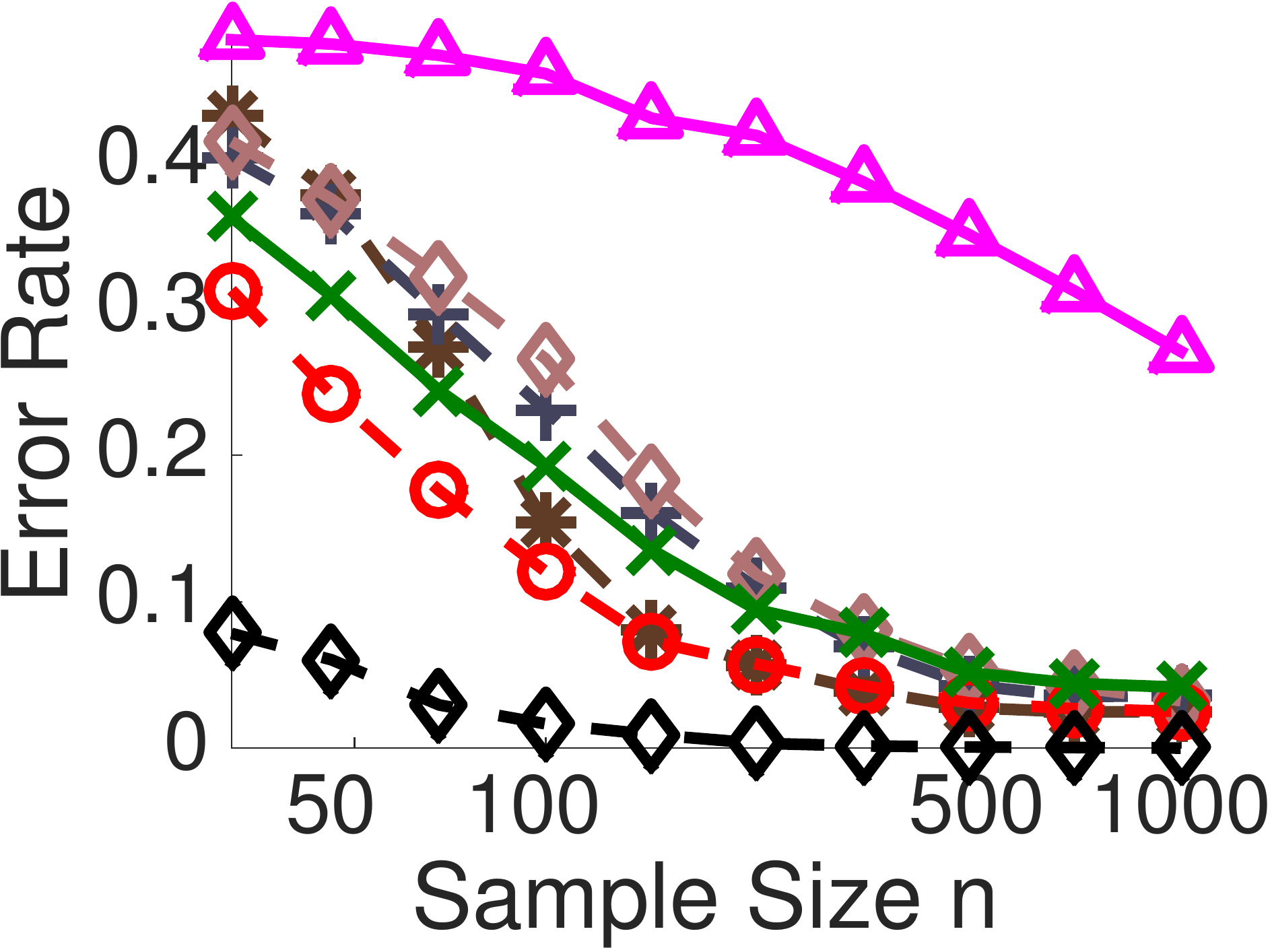}   &    \hspace{-1.2em}   
      \includegraphics[width=.22\textwidth, trim={0 1.75cm 0 0},clip]{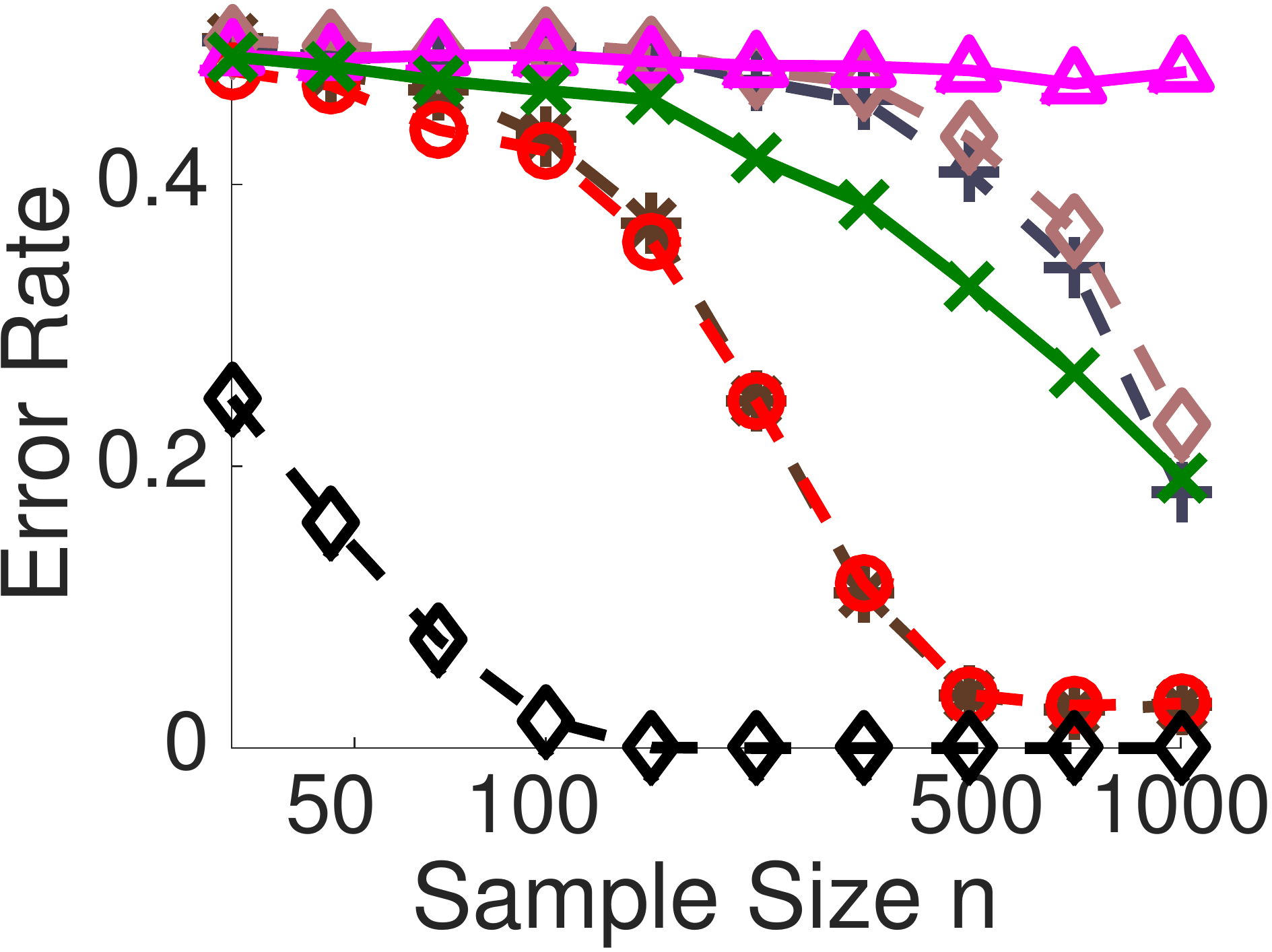}   &   \hspace{-1.5em}
   \includegraphics[width=.22\textwidth, trim={0 1.75cm 0 0},clip]{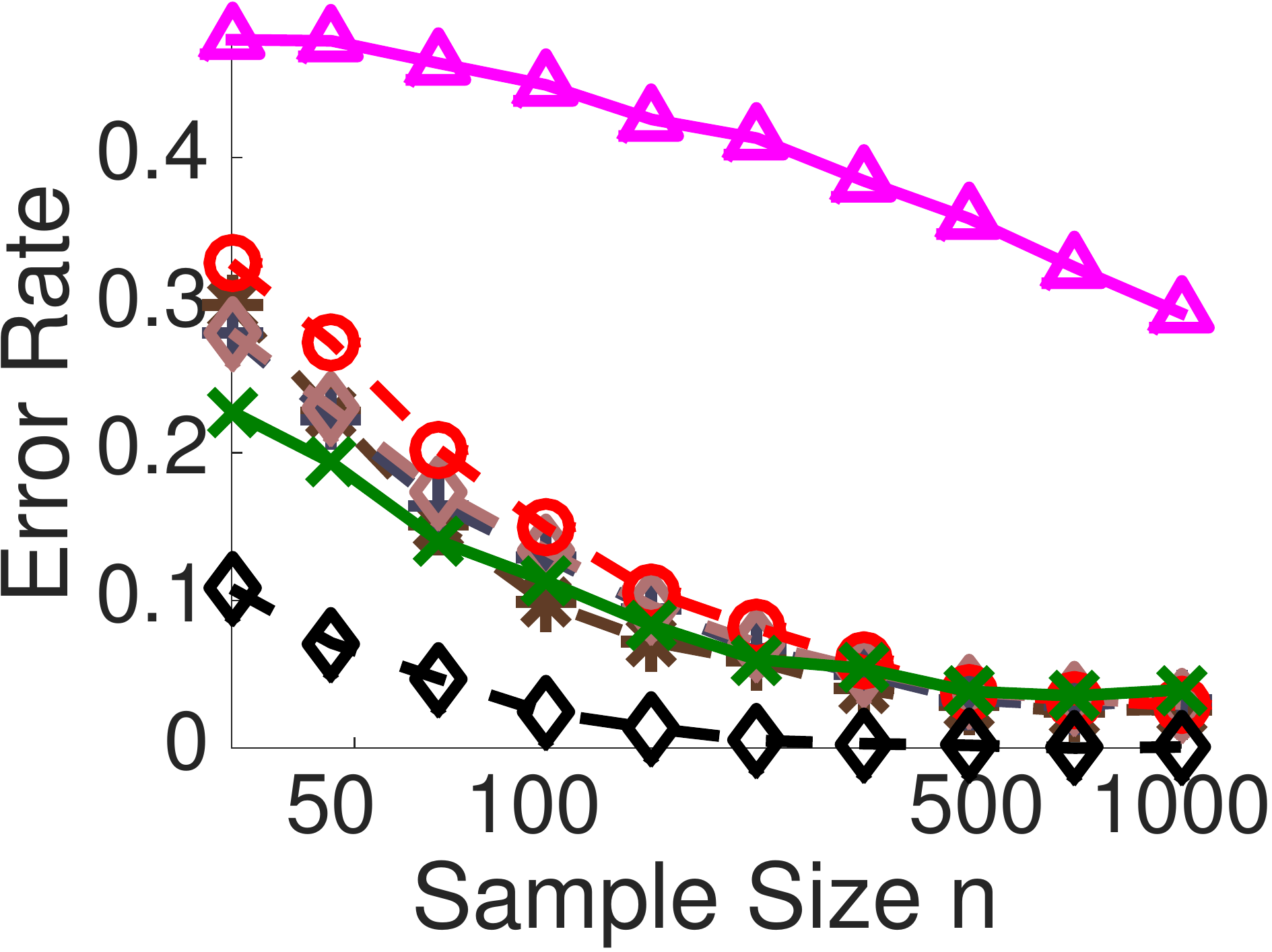} & \hspace{0em}
   \raisebox{0em}{\includegraphics[width=.23\textwidth, trim={0 0cm 0 0},clip]{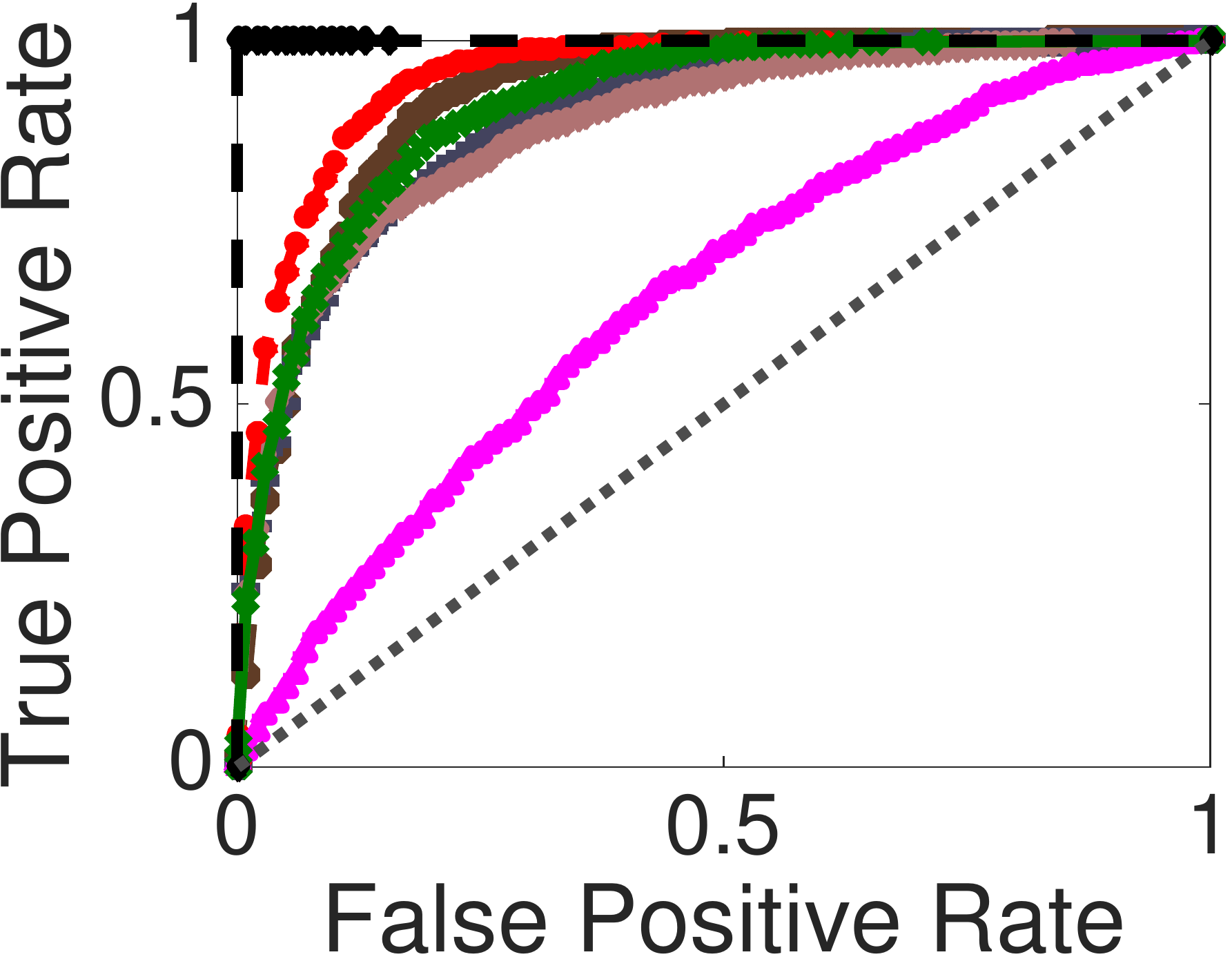} }    
    \\[-.5em]
        { \scalebox{1.1}{\fontfamily{phv}\selectfont \scriptsize Sample Size $n$ }}   &
{  \scalebox{1.1}{\fontfamily{phv}\selectfont \scriptsize  Sample Size $n$ }}  &
{  \scalebox{1.1}{\fontfamily{phv}\selectfont \scriptsize  Sample Size $n$ }}       &  
{~~~~~~~ ~~~ \scalebox{1.1}{\fontfamily{phv}\selectfont \scriptsize   ($n=100$, $\sigma_{per}=1$) }}        \\      
%   \includegraphics[width=.3\textwidth, trim={0 1cm 0 0},clip]{figures/logpAIS1000_gbrbmNhid10Nvis20sign1_LRrepeat1000_N100sVec0_16_1M210_7_1000.pdf}    \\
%     (d) Mean Perturbed  & (e) Variance Perturbed  & (f) Weights Perturbed &           ~~~~ (g) Mean Perturbed 
     (d) Perturbation on Mean  & (e)  Perturbation on Variance & (f) Perturbation on Weights & ~~~~ (g) Perturbation on Mean
%     \raisebox{.8em}{\begin{tabular}{l}
%     (g) Mean Perturbed \\
%     ($n=100$, $\sigma_{per}=1)$
%     \end{tabular}}     
%     \\%[0.2em]
%     (d) Perturbation on Mean  & (e)  Perturbation on Variance & (f) Perturbation on Weights &  \\
%  (d) & (e)& (f)  &
    \end{tabular}\\
%\setlength{\unitlength}{\textwidth}    
%   \begin{picture}(0,0)
%   \put(-.38,.15){\includegraphics[width=.18\textwidth]{figures/GenV2_gbrbmNhid10Nvis20sign1_LRrepeat1000_N100sVec0_16_1M210_7_1000.pdf} }      
%    \end{picture}
   \caption{Results on 1D Gaussian mixture. (a)-(c) The error rates of different methods vs. the perturbation magnitude $\sigma_{per}$ when perturbing the mean, variance and mixture weights, respectively; we use a fixed sample size of $n =100$. (d)-(f) the error rates vs. the sample size $n$, with fixed perturbation magnitude $\sigma_{per}= 1$.  
We find that the type I errors of all the methods are well controlled under $0.05$, and hence the reported error rates are essentially type II errors.    
%   The error rates include both 
   (g) The ROC curve with mean perturbation, $n=100$, $\sigma_{per}=1$. 
    }
%    Results are averaged on $100$ random trials.}
   \label{fig:gmm5}
\end{figure*}

\section{Experiments}
\label{sec:exp}
We present empirical results in this section. We start with a toy case of 1D Gaussian mixture on which we can compare with the classical goodness-of-fit tests that only work for univariate distributions, and then proceed to 
Gaussian-Bernoulli restricted Boltzmann machine (RBM), a graphical model widely used in deep learning \citep{welling2004exponential, hinton2006reducing}. 
The following methods are evaluated, all with a significance level of $0.05$: 

1) {\tt KSD-U}. The KSD-based bootstrap test using $U$-statistic in Algorithm~\ref{alg:test} (bootstrap size is $1000$), using
RBF kernel with bandwidth chosen to be median of the data distances.  
%The kernel bandwidth is choose 

2) {\tt KSD-Linear}. The KSD test based on the linear estimator in \eqref{equ:linear} with asymptotically normal null distribution.  %[intro earilier]. 

3) Classical goodness-of-fit tests, including $\chi^2$ test, Kolmogorov-Smirnov test and Cramer-von Mises test \citep{lehmann2006testing}; they are evaluated on only the 1D Gaussian mixture. 

4) {\tt MMD-MC(${n'}$)}. Draw exact sample $\{y_i \}$ of size $n'$ from $q(x)$ and perform two sample MMD test of \citet{gretton2012kernel} over $\{x_i\}$ and $\{y_i\}$ using bootstrap\footnote{We use the mmdTestBoot.m under \url{http://www.gatsby.ucl.ac.uk/\%7Egretton/mmd/mmd.htm}}, with 1000 bootstrap replicates.   

5) {\tt MMD-MCMC(${n'}$)}. Draw approximate sample $\{y_i \}$ of size $n'$ from $q(x)$ using Gibbs sampler and perform MMD test on $\{x_i\}$ and $\{y_i\}$; we use $1000$ burn-in steps. 

6) {\tt LR (simple vs. simple)}. We evaluate the  exact log-likelihood ratio $2\log (q(x) /p(x))$ and use it to test whether $\{x_i\}$ is drawn from $p(x)$ or $q(x)$. 
This approach is an oracle test in that it knows it exactly calculates the likelihood, and assumes we know $p(x)$ and tests a much easier null hypothesis of simple vs. simple. 
%This approach is oracle in that it exactly calculates the likelihood, and also assumes that we know $p(x)$ and tests a much easier null hypothesis of simple v.s. simple. 

7) {\tt Likelihood Ratio (AIS)}. 
We  approximately evaluate the likelihood ratio using annealed importance sampling (AIS), which is one of the most widely used algorithm for approximating likelihood \citep{neal2001annealed, salakhutdinov2008quantitative}. 
Our AIS implementation uses a Gibbs sampler transition with a linear temperature grid of size $1000$.
We do not perform a test based on the AIS result because it is hard to know the approximation error. 
%due to its approximation nature. 
%since it would requires the maximum likelihood estimator that is hard to calculate. 

%We also evaluate the likelihood of $\{x_i\} \sim p$ under $q(x)$, both exactly using  the brue-force algorithm, and approximately using annealed importance sampling using Gibbs sampler transition with $1000$ linear temperature grid. We do not perform the standard likelihood ratio test, since it would requires the maximum likelihood estimator that is hard to calculate. 

\paragraph{1D Gaussian Mixture} 
We draw i.i.d. sample $\{x_i\}_{i=1}^n$ from $p(x) =\sum_{k=1}^5 w_k \normal(x~; \mu_k, \sigma^2)$ with $w_k = 1/5$, $\sigma = 1$ and $\mu_k$ randomly drawn from $\mathrm{Uniform}[0,10]$.
We then generate  $q(x)$ by adding Gaussian noise on $\mu_k$, $\log w_k$, or $\log \sigma^2$, leading to three different ways for perturbation; 
the perturbation magnitude is controlled by the variance $\sigma_{per}^2$ of Gaussian noise. 
In our experiment, we set $q(x)$ randomly with equal probability to be either the true model $p(x)$ ($H_0: p=q$), or the perturbed version ($H_1: p \neq q$), and use different methods to test $H_0$ vs. $H_1$. 
We repeat $1000$ trials, and report the average error rate in Figure~\ref{fig:gmm5}. 
%Therefore, a random guess would give 
% We then measure the error rate of different tests; note that the maximum error rate should be $0.5$, giving by random guessing. 
%Therefore, 
% version of $$

%Figure~\ref{fig:gmm5} shows the error rates of different algorithms when the deviation magnitudes $\sigma_{per}$ and the sample size $n$ varies, under different types of perturbations. 
We find from Figure~\ref{fig:gmm5} that the oracle {\tt LR (simple vs. simple)} performs the best as expected. Otherwise, 
our {\tt KSD-U} performs comparably with, or better than, the classical tests ($\chi^2$, Kolmogorov-Smirnov and Cramer-Von Mises) as well as {\tt MMD-MC(1000)}. 
{\tt KSD-Linear} tends to perform the worst, suggesting it is not useful in this simple setting.  
However, it can serve as a computationally efficient alternative of {\tt KSD-U} for more complex models on which the other tests are not practical. 
Note that because both the cases of $p=q$ and $p\neq q$ happen with $0.5$ probability in our simulation, the error rate in the hardest case when $p$ is close $q$ is $0.5$. 
%the best or close to the best. 
%without surprise. 
%, (a)-(c) shows the error rates of different algorithms vs. the deviation magnitude $\sigma_p$ with fixed  nd (d)(f) shows the error rates vs. the sample size 
% We see that 
 
% {\fontfamily{phv}\selectfont This text uses a different font typeface }

\begin{figure*}[htb]
   \centering
   \scalebox{.95}{
   \begin{tabular}{cccc}
      \includegraphics[width=.275\textwidth, trim={0 1cm 0 0},clip]{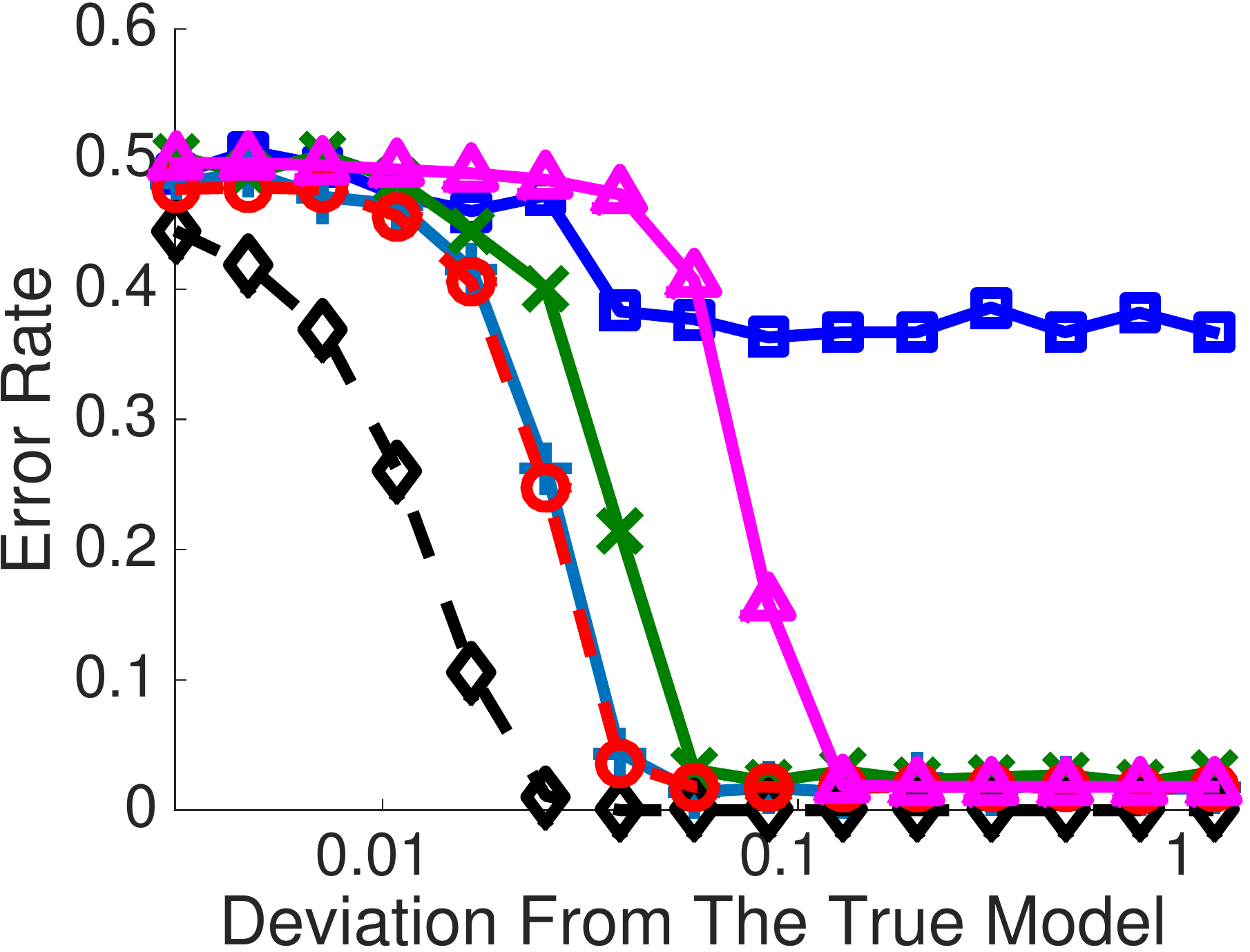} &
   \includegraphics[width=.2\textwidth, trim={0 1cm 0 0},clip]{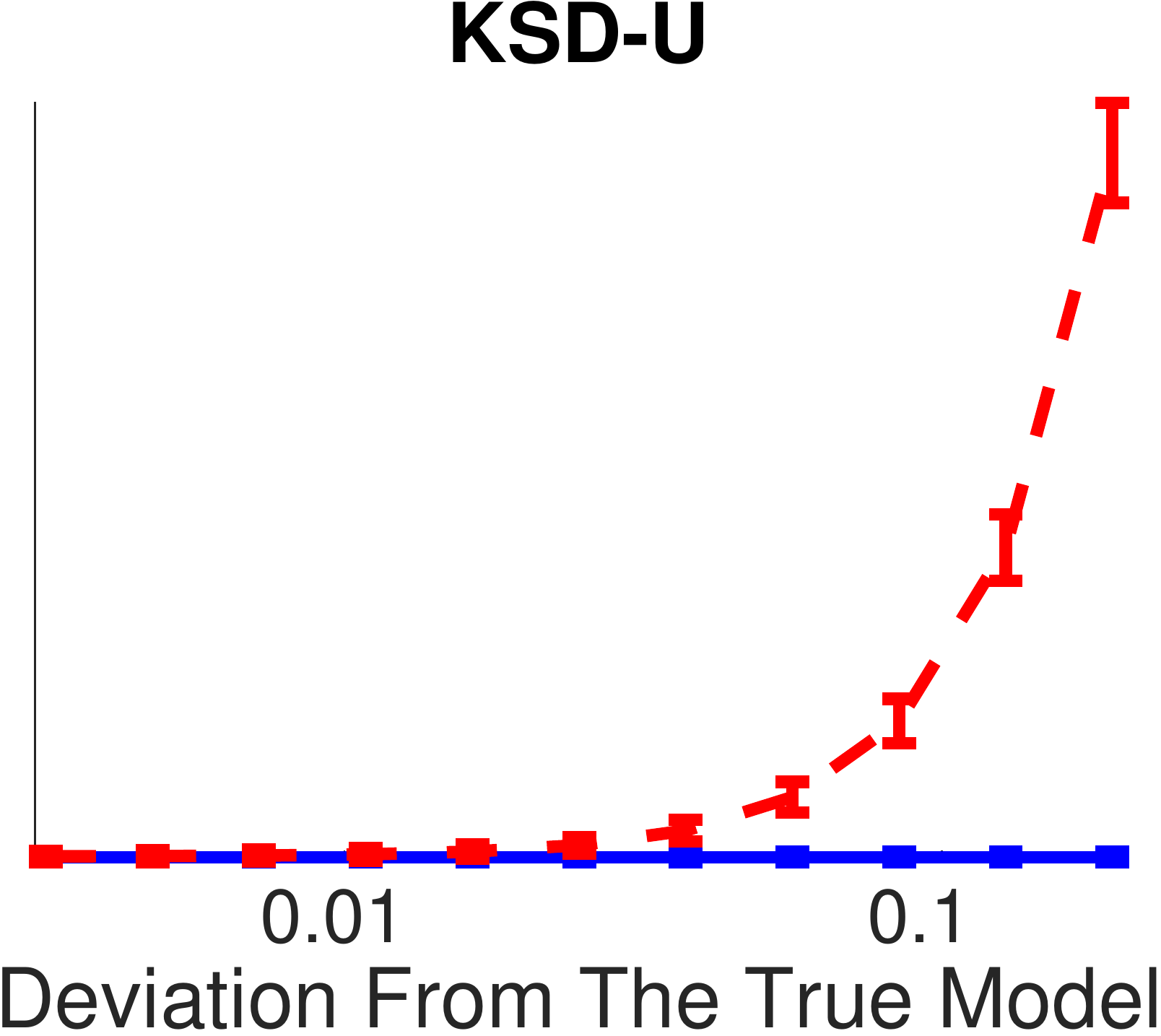}   &
   \includegraphics[width=.2\textwidth, trim={0 1cm 0 0},clip]{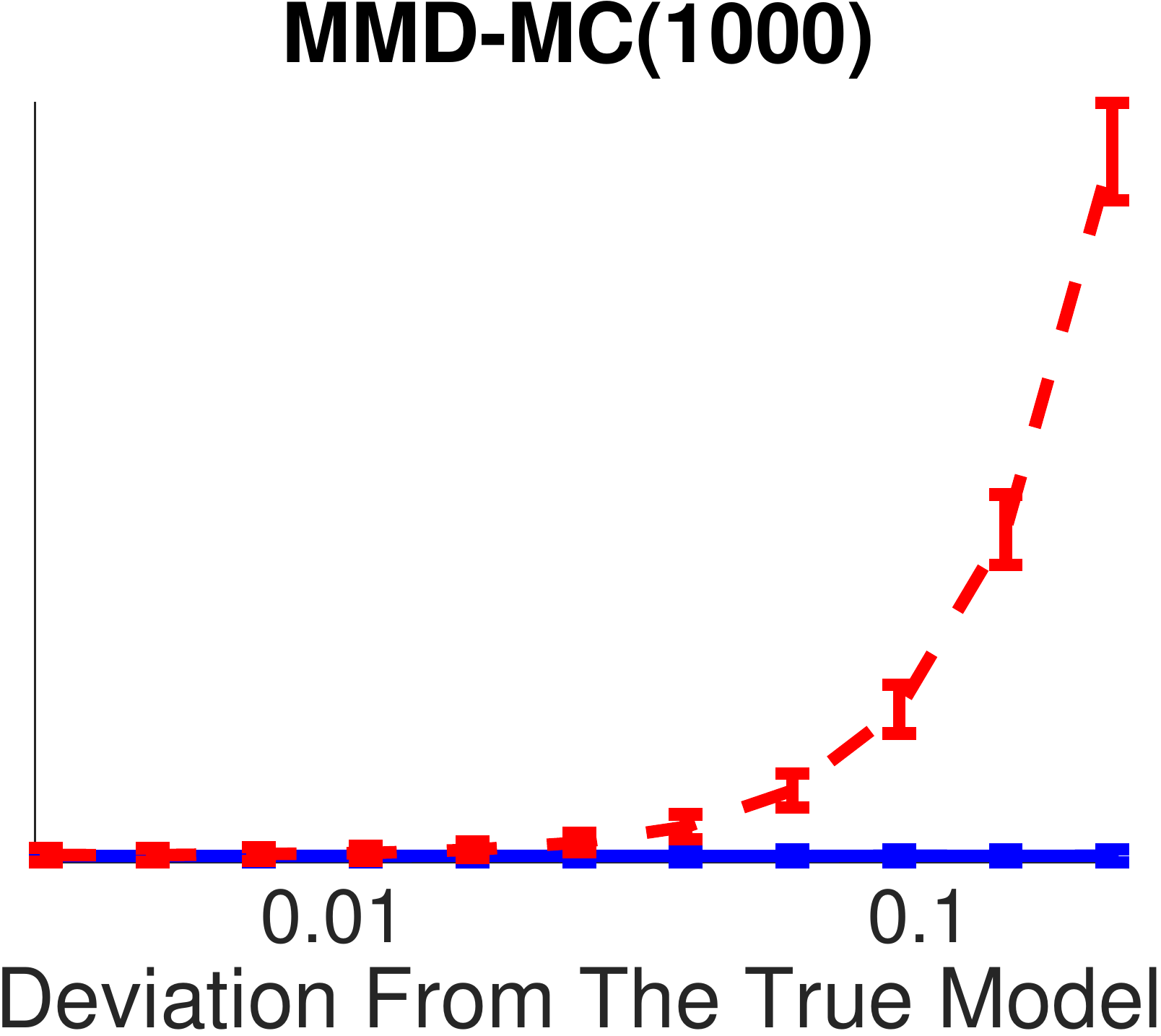}   &   
      \includegraphics[width=.2\textwidth, trim={0 1cm 0 0},clip]{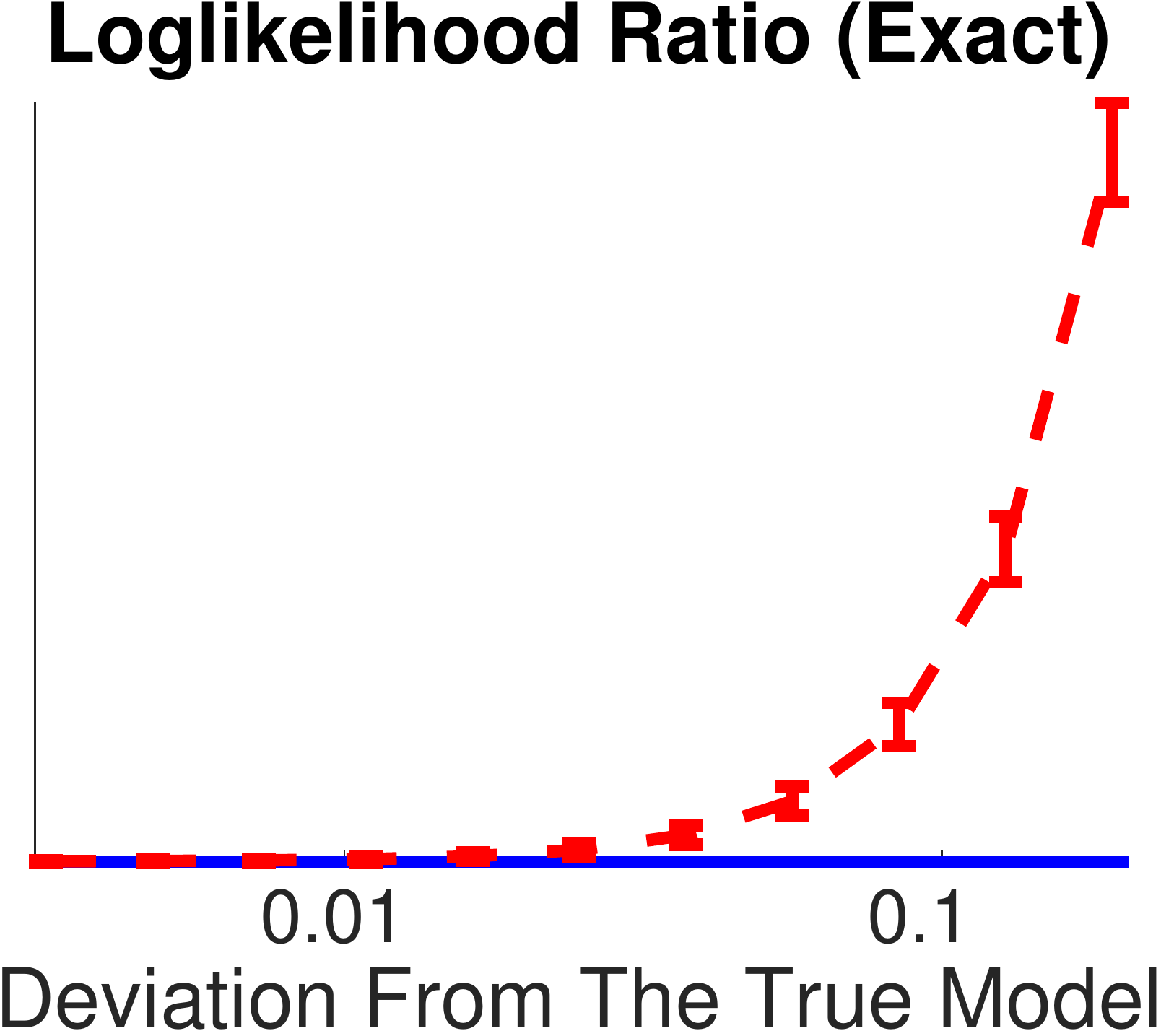}\\ [-.7em]      
    {~~~~~~~ {\fontfamily{phv}\selectfont \scriptsize Perturbation Magnitude $\sigma_{per}$ }}   &
{{\fontfamily{phv}\selectfont \scriptsize Perturbation Magnitude $\sigma_{per}$ }}  &
{ {\fontfamily{phv}\selectfont \scriptsize Perturbation Magnitude $\sigma_{per}$ }}       & 
{ {\fontfamily{phv}\selectfont \scriptsize Perturbation Magnitude $\sigma_{per}$ }} 
 \\         
      %  \vspace{1em}  \\
    \small  (a) & \small (c) & \small (e) & \small (g) \\
   \includegraphics[width=.30\textwidth, trim={0 1cm 0 0},clip]{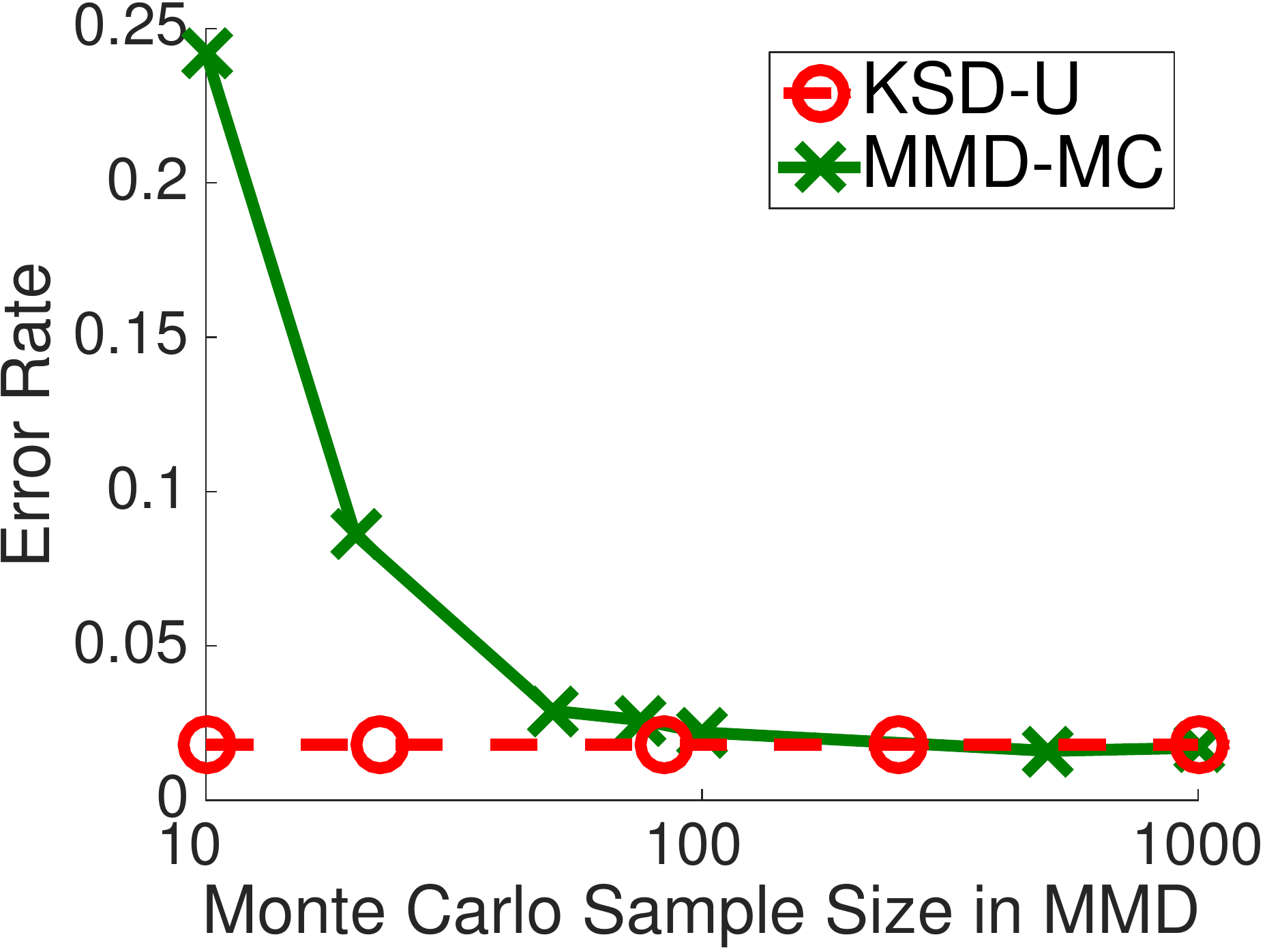}   &       
      \includegraphics[width=.2\textwidth, trim={0 1cm 0 0},clip]{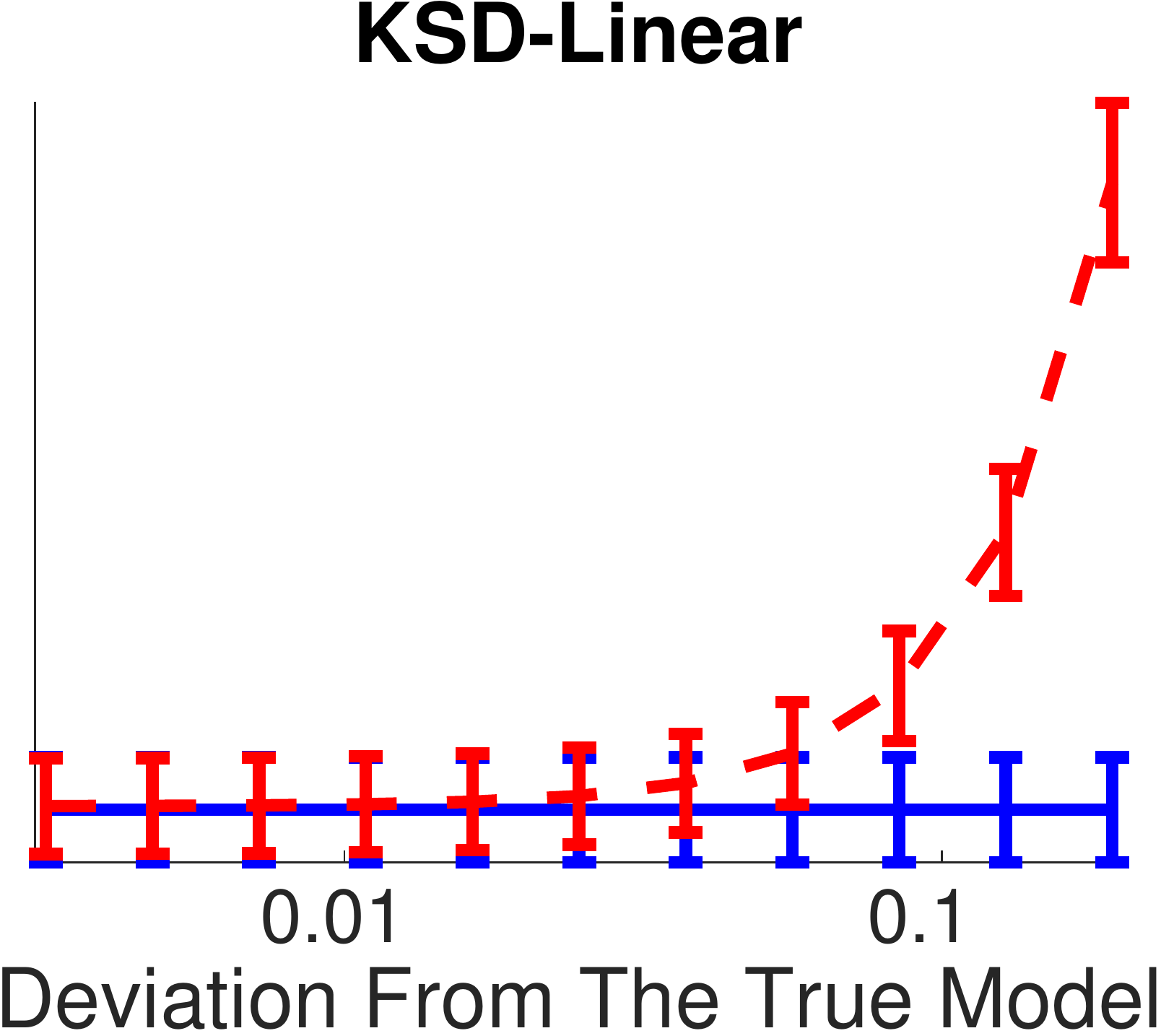}   &   
   \includegraphics[width=.2\textwidth, trim={0 1cm 0 0},clip]{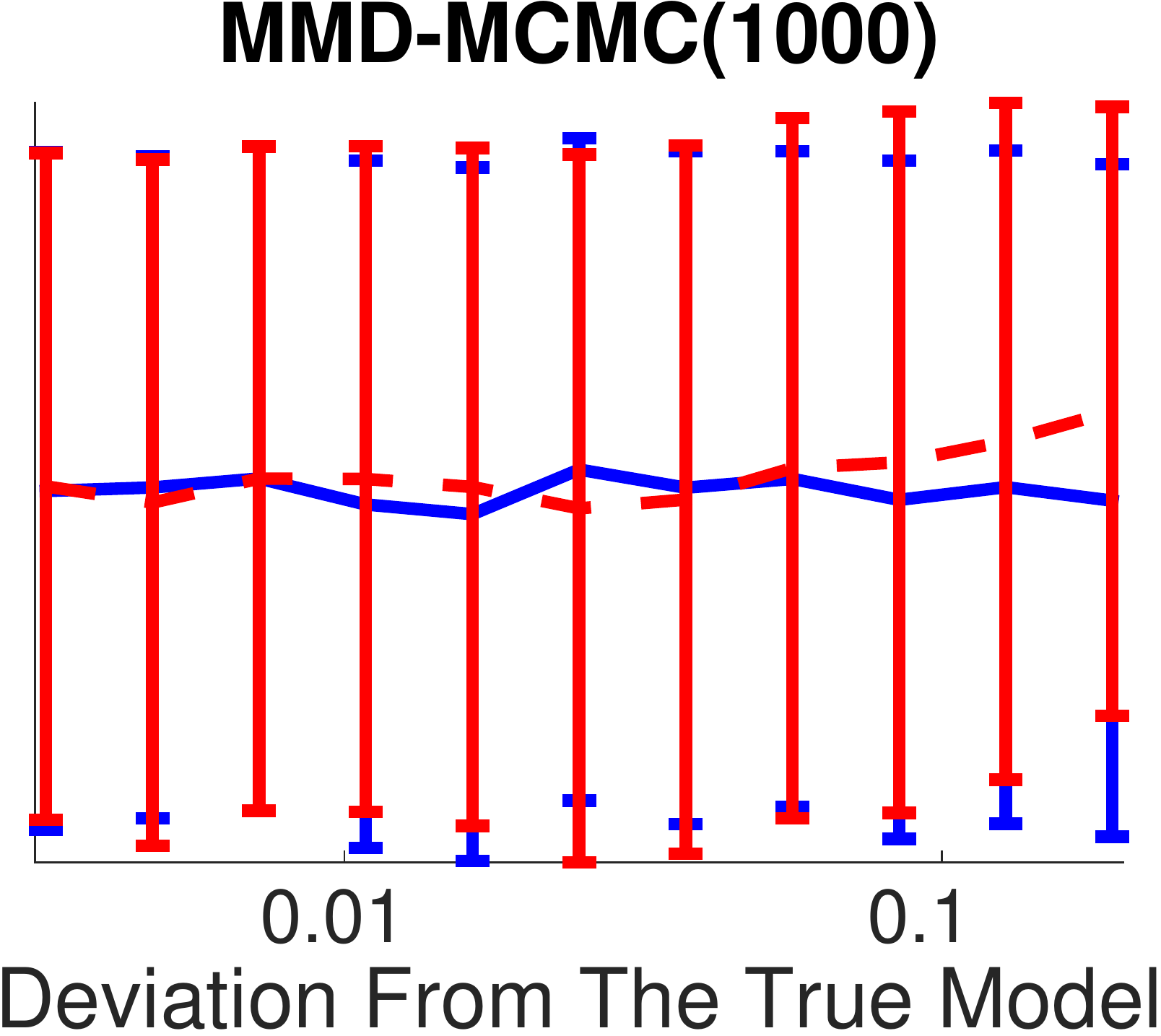}   &
   \includegraphics[width=.2\textwidth, trim={0 1cm 0 0},clip]{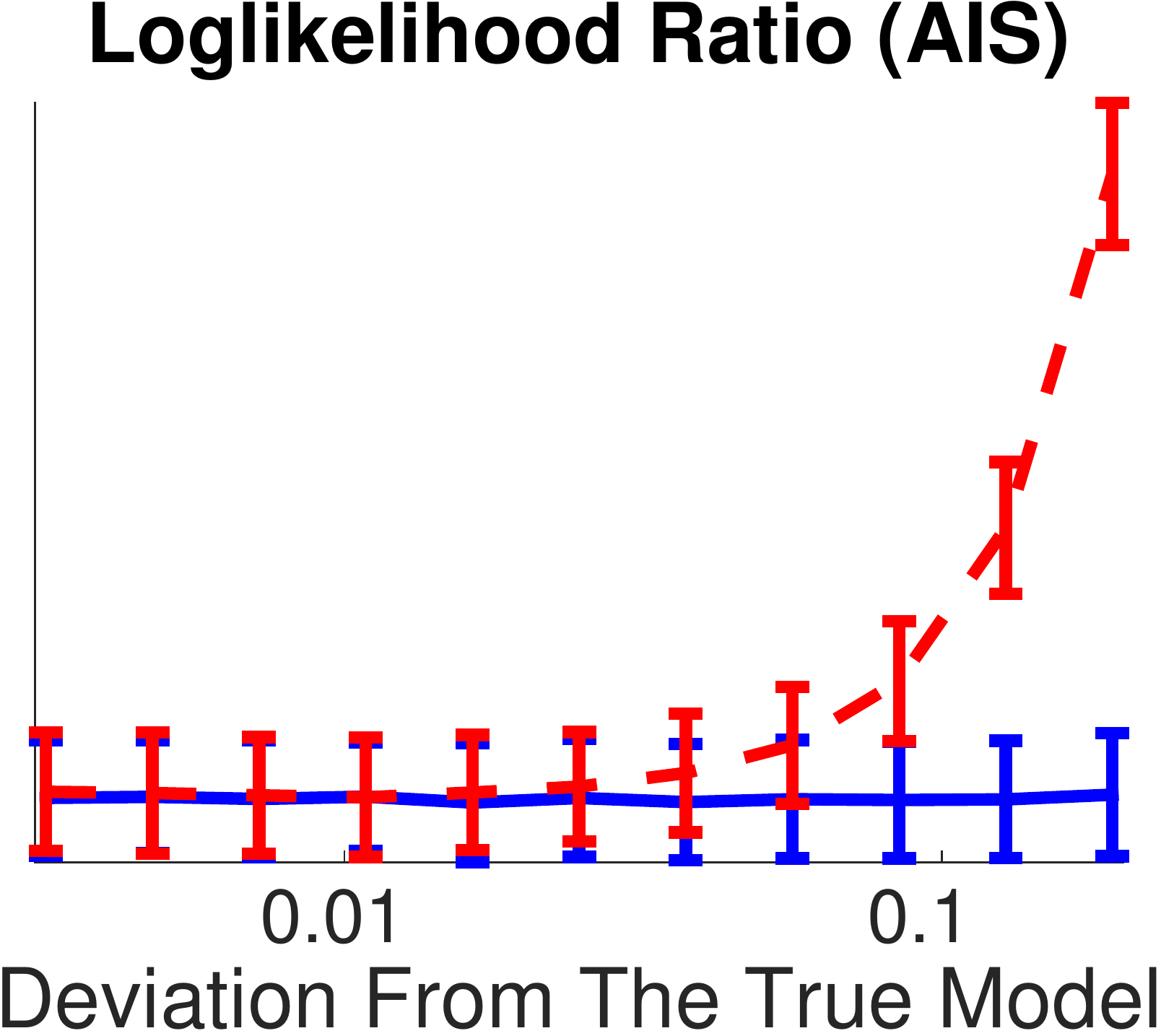}    \\
    [-.7em]      
    {~~~~~~~ {\fontfamily{phv}\selectfont \scriptsize Monte Carlo Sample Size $n'$ in MMD }}   &
{{\fontfamily{phv}\selectfont \scriptsize Perturbation Magnitude $\sigma_{per}$ }}  &
{ {\fontfamily{phv}\selectfont \scriptsize Perturbation Magnitude $\sigma_{per}$ }}       & 
{ {\fontfamily{phv}\selectfont \scriptsize Perturbation Magnitude $\sigma_{per}$ }} 
\\
\small  (b) & \small (d)& \small (f) & \small (h)
    \end{tabular}
    }\\
\scalebox{.95}{
\setlength{\unitlength}{\textwidth}    
   \begin{picture}(0,0)(0, 0.01)
   \put(-.387,.15){\includegraphics[width=.18\textwidth]{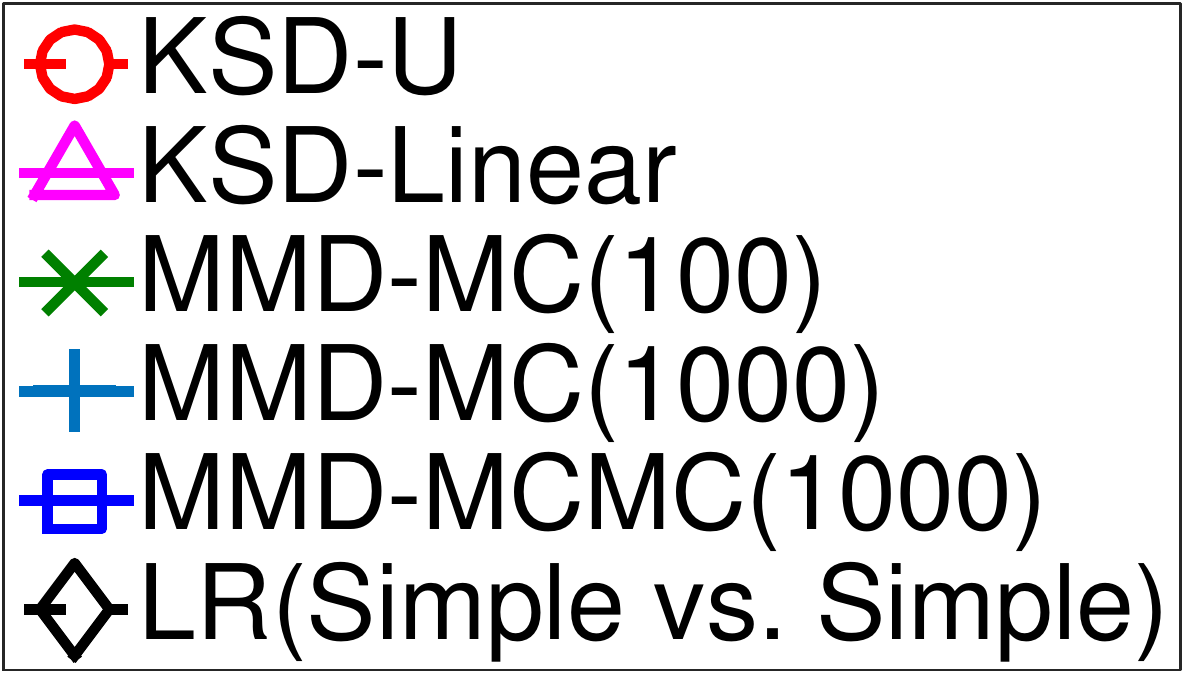} }      
    \put(-0.15,.4){\includegraphics[width=.07\textwidth]{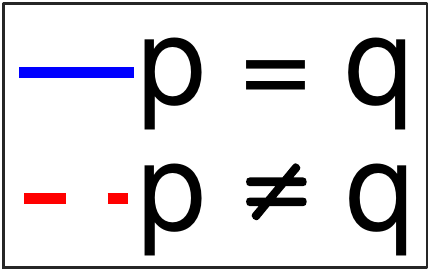}}       
    \put(-0.15,.14){\includegraphics[width=.07\textwidth]{figures/Legend_12.pdf}}
    \put(-.18,.12){\rotatebox{90}{\fontfamily{phv}\selectfont \scriptsize Discrepancy }}   
    \put(-.18,.36){\rotatebox{90}{\fontfamily{phv}\selectfont \scriptsize Discrepancy }}       
    \end{picture}
    }
   \caption{Results on Gaussian-Bernoulli RBM. (a) The error rate vs.  the perturbation magnitude $\sigma_{per}$. %between $p$ and $q$
   (b) The error rate of {\tt MMD-MC} vs. the size of the exact sample used. 
    (c) Different discrepancy measures between $p$ and $q$ under the null $p = q$ (blue solid lines) and the alternatives $p\neq q$ (dashed red lines); 
    the x-axes are the deviation $\sigma_{per}$ between $p$ and $q$ when $p \neq q$. We set $n = 100$ in all the cases.}
%    We use a significance level of $\alpha = 0.05$ and a sample size of $n = 100$ in all the cases.}
%    Results are averaged on $1000$ random trials.}
   %as the deviation increases; the blue lines are when the case $p = q$.  }
   \label{fig:gbrmhard}
\end{figure*}

\paragraph{Gaussian-Bernoulli Restricted Boltzmann Machine (RBM)}
\fullversion{
Gaussian-Bernoulli RBM is a type of hidden variable graphical models consist of a continuous observable variable $x \in \R^d$ and a binary hidden variable $h \in \{\pm 1\}^{d'}$, with joint probability
\begin{align*}
%p(x) = \sum_{h\in \{\pm 1\}^{d'}}p(x, h),  ~~~\text{where~~~~}  \\ 
p(x,h) \propto \exp( \frac{1}{2}x^\top B h + b^\top x + c^\top h - \frac{1}{2}|| x||^2_2). 
\end{align*}
The probability of the observable variable $x$ is $p(x) = \sum_{h\in \{\pm 1\}^{d'}}p(x, h).$
Equivalently, $p(x)$ can be viewed as a Gaussian mixture model with an exponential number of components, each corresponding to a configuration of an Ising model defined on $h$, that is,  %with each corresponds a state in $\{\pm 1\}^{d'}$. 
\begin{align*}
p(x) & \propto \sum_{h}  p(h)  \normal(x, ~ Bh + b, ~I), \\ 
p(h) & = \frac{1}{Z_h} \exp(\frac{1}{2}h^\top B^\top B h + (B^\top b + c)h), 
\end{align*}
where $p(h)$ is obtained by explicitly integrating out $x$ in $p(x,h)$.  
Because calculating the constant $Z_h$ of the Ising model is \#P-complete, and it is generally tractable to evaluate the likelihood $p(x)$. 
Nevertheless, one can show that its score function $\score_p$ can be easily calculated in a closed form, 
\begin{align*}
\score_p(x)= \nabla_x  \log p(x) =  b -x  + B\tanh(B^\top x + c),  
\end{align*}
where $\tanh(y) =  \frac{\exp(y)-\exp(-y)}{\exp(y)+\exp(-y)}.$
}
Gaussian-Bernoulli RBM is a hidden variable graphical models consist of a continuous observable variable $x \in \R^d$ and a binary hidden variable $h \in \{\pm 1\}^{d'}$, with joint probability
\begin{align*}
%p(x) = \sum_{h\in \{\pm 1\}^{d'}}p(x, h),  ~~~\text{where~~~~}  \\ 
p(x,h) = \frac{1}{Z} \exp( \frac{1}{2}x^\top B h + b^\top x + c^\top h - \frac{1}{2}|| x||^2_2), 
\end{align*}
where $Z$ is the normalization constant. 
The probability of the observable variable $x$ is $p(x) = \sum_{h\in \{\pm 1\}^{d'}}p(x, h),$ 
which is intractable to calculate due to the difficult constant term $Z$. 
%Because calculating the constant $Z_h$ of the Ising model is \#P-complete, and it is generally tractable to evaluate the likelihood $p(x)$. 
Nevertheless, one can show that its score function $\score_p$ can be easily calculated in a closed form, 
\begin{align*}
\score_p(x)=   b -x  + B~\phi(B^\top x + c),  ~~~~~~~ \phi(y) =  \frac{e^{2y}-1}{e^{2y}+1}. 
%\score_p(x)=   b -x  + B\tanh(B^\top x + c),  ~~~ \tanh(y) =  \frac{e^{2y}-1}{e^{2y}+1}
%\score_p(x)= \nabla_x  \log p(x) =  b -x  + B\tanh(B^\top x + c),  
%\score_p(x)= \nabla_x  \log p(x) =  b -x  + B \frac{\exp(B^\top x + c)-\exp(-B^\top x - c)}{\exp(B^\top x + c)+\exp(-B^\top x - c)}
%\tanh(B^\top x + c),  
%~~~~~~\text{where}~~~~~  \\
%\E[h | x], ~~~~ \E[h|x] = \frac{      }{  \exp(B_i^\top  x + c) } 
\end{align*}
%where 
%$\tanh(y) =  [{\exp(y)-\exp(-y)}]/[{\exp(y)+\exp(-y)}].$
%$\tanh(y) =  \frac{\exp(y)-\exp(-y)}{\exp(y)+\exp(-y)}.$

In our experiment, we simulate a true model $p(x)$ by drawing $b$ and $c$ from standard Gaussian and select $B$ uniformly randomly from $\{\pm 1\}$; we use $d=50$ observable variables and $d'=10$ hidden variables, so that it remains possible to exactly calculate $p(x)$ and draw exact samples using the brute-force algorithm. 
Similar to the case of 1D Gaussian mixture,  
we set $q(x)$ randomly with equal probability to be equal to either $p(x)$ or a perturbed version by adding Gaussian noise to $B$ with variance $\sigma_{per}^2$. %, each with half probability. 
%we set $q(x)$ to be either equal to $p(x)$ or a perturbed version by adding Gaussian noise to $B$ with a given magnitude $\sigma_{per}^2$, each with half probability. 
We report the the error rates of different tests in Figure~\ref{fig:gbrmhard}; the results are averaged on $1000$ random trials. 
%Therefore, a random guess would have a $0.5$ error rate on deciding if $p= q$. 

%We use a sample size of $n =100 $ and a significance level of $\alpha = 0.05$ for all the tests and report the error rates (the percentage of mistakes) of the different tests in Figure~\ref{fig:gbrmhard}(a). 
%We find that the
Figure~\ref{fig:gbrmhard}(a) shows that the oracle {\tt LR (simple vs. simple)} performs the best again as expected, followed by our {\tt KSD-U} method. 
%Note that because the case $p=q$ and $p\neq q$ both take fifty percent in our simulation, the error rate in the hardest case when $p$ is close $q$ is $0.5$. 
%\jnote{How is error rate defined here? An error is if you are under the alternative and you accept, right? But then the error when you have close to 0 deviation from the null should be always about $1-\alpha =.95$, but it seems the error rate is .5 in Figure 1a for small deviations.} 
The {\tt MMD-MCMC} breaks down because the MCMC sample is not representative of $q$, while the performance of {\tt MMD-MC} depends on the size of the exact sample: 
it performs worse than {\tt KSD-U} with {\tt MMD-MC(100)}, and is almost as good with {\tt MMD-MC(1000)}; see also Figure~\ref{fig:gbrmhard}(b). 
Again, we find that {\tt KSD-linear} generally performs much worse than {\tt KSD-U}, but it provides a computationally efficient $O(n)$ alternative to {\tt KSD-U} which has a $O(m n^2)$ complexity and {MMD} which costs $O(mnn')$. 
A trade-off between linear and quadratic complexity can be achieved using block averaging; see \citet{zaremba2013b}.   

Figure~\ref{fig:gbrmhard}(c)-(h) shows the different discrepancy measures under the case $p=q$ and $p\neq q$, respectively. 
Again, we can find that the exact likelihood ratio provides the best discrimination, while  {\tt MMD-MCMC} fails to distinguish the two cases at all. 
The AIS approximation performs reasonably well, but is worse than {\tt KSD-U} and {\tt MMD-MC(1000)} in this particular case. 
%We find that {\tt KSD-U} performs achieves the best type II error (power) under both cases. {\tt MMD-MC} also performs well, but with a slightly worse type II error (and better type I error) in both cases. As we expect, {\tt MMD-MCMC} performs well as {\tt MMD-MC} in the easy case, but completely breaks down when in the hard case; this is because the MCMC does not converge in the hard case, and hence MMD would the two distributions are different even if they are. 
%It is interesting to note that both {\tt KSD-U} and {\tt MMD-MC} seem to serve as better discrepancy measures than the exact likelihood in both cases. 
%AIS gives a close approximation of the likelihood in the easy case, but perform poorer in the hard case. 
%
%We find that {\tt KSD-linear} generally perform much worse than {\tt KSD-U}, but it provides a computational efficient $O(n)$ alternative for {\tt KSD-U} which has a $O(n^2)$ complexity. A trade-off between linear and quadratic complexity can be achieved using blocking averaging; see \citet{Gretton}.  

%Because it is difficult to calculate the normalization constant $Z_h$ in the Ising model $w(h)$, it is 
%Therefore, it is difficult to calculate the likelihood $p(x)$. 

\section{Conclusion and Future Directions}
\label{sec:con}
We propose a new computationally tractable discrepancy measure between complex probability models, and use it to derive a novel class of goodness-of-fit tests. 
We believe our discrepancy measure provides a new fundamental tool for analyzing and using complex probability models in statistics and machine learning. 
Future directions include extending our method to composite goodness-of-fit tests, in which we want to test if the observed data follows a given class of distributions, 
as well as understanding the theoretical discrimination power of KSD compared to the other classical goodness-of-fit tests, 
two sample tests (e.g., MMD with infinite exact Monte Carlo sample), and the method in \citet{gorham2015measuring}.  
%An immediate future direction is to extend our method to composite goodness-of-fit tests, in which we want to test if the observed data follows a given class of distributions. 
%Another important open problem is to
%understand the theoretical discrimination power of our method compared to the other classical goodness-of-fit tests, as well as 
%MMD (e.g., with infinite exact Monte Carlo sample), and the method in \citet{gorham2015measuring}.  
%MMD (e.g., with infinite exact Monte Carlo sample), and the method in \citet{gorham2015measuring}.  
%for measuring the difference between tw
%The work opens a new fundamental tool for analyzing complex probability models, providing new potentials for many different problems. 
%An immediate future direction is to extend our work on composite goodness-of-fit tests, in which we want to test if data follows a class of distribution. 
%\todo{Composite test: testing whether $p$ belongs to a parametric family $q(x;  \theta)$, we can first optimize $\theta$, and then do a test (how the optimization influence the test??)}

\paragraph{Acknowledgment}
This work is supported in part by NSF CRII 1565796. 
We thank Arthur Gretton and the anonymous reviewers for their valuable comments. 
%\section{todo list} 
%X. Some discussion on the efficiency of the test?
%X. Consistency of bootstrap under both degenerated and non-degenerate cases. 
%X. Clean up the zero boundary condition, call stein class instead?

%\bibliographystyle{unsrtnat}
\newpage
{\small
\bibliographystyle{icml2015mine}
\bibliography{bibrkhs_stein}
}
%\clearpage \newpage
%\appendix 

\myempty{
\clearpage\newpage 
\red{::::::::::: trash}
\paragraph{Probability Metrics}
Probability metrics or divergences quantify distances between different probability distributions, and are of fundamental importance in machine learning and statistics. 
The most common examples are Kullback-Leiber (KL) divergence, $\chi^2$ divergence, $f$-divergences and total variance distance. 
Unfortunately, computational tractability are often a critical issue when evaluating these metrics in real world applications, which is especially true in modern machine learning applications with complex probabilistic models , such as probabilistic graphical models and deep generative models, that involves high dimensional variables for which even evaluating the likelihood is challenging. 
In this work, we propose a new probability metrics motivated on stein' method, by considering Stein discrepancy in reproducing kernel Hilbert spaces (RKHS). 
Our new metric emphasizes computational tractability even for highly complex probabilistic graphical models. 
%and ease of computing. 

%Although there are numerous divergence  which has been an important theoretical tool for approximating distributions in theoretical analysis. It has been widely used for proving central limit theorems, providing error bounds, and proving concentration inequalities. This work exploit its new ponential to apply it real world applications involves in machine learning. 

\paragraph{Goodness-of-fit and Model Evaluation}
Our new metric finds a natural application for goodness fit tests or model evaluations, which 
quantifies how well a given model $q(x)$ fits a set of observations $\{x_i\}$. 
Modern machine learning and statistics increasingly involves complicated probabilistic models that are out of our current computational and analytical reach. 
In the recent growing field of deep learning, highly complicated probabilistic models are constructed from data using all sorts of heuristics. 
These models are learnt without statistical guarantees, and more problematically, it is even challenging to validate these models in unsupervised fashion, since the likelihood, as the typical goodness-of-fit measurement, is often intractable to calculate in complicated models. Often people use Markov chain Monte Carlo (MCMC) or deterministic approximations for approximating the likelihood, however, since the log-likelihood is fundamentally difficult to calculate (\# P-complete for graphical models), the quality of these approximations tends to be very bad, and often inherent certain biases due to the choice of the approximation method; e.g. typical MCMC methods tend to underestimate the likelihood due to the under-exploration of the probability field. 
A critical question is to find alternative measure of goodness-of-fit that are computationally tractable. 
%We achieve this goal using stein method. 

\emph{Our Applications:}

1. Unsupervised model evaluation and comparison for complicated probabilistic models, such as Markov random fields or deep generative models. 

2. Checking the convergence of MCMC. 

3. When learning under data streaming, checking the fitness of the current model against new data. 

\paragraph{Stein's Method}
Stein's method, first introduced by Charles Stein in 1972 \citep{stein1972} to obtain bounds between the distribution of sum of dependent random variables with a standard normal distribution. 
It has since then developed into a collection of powerful tools is a general method in probability theory to obtain bounds on the distance between two probability distributions, %with respect to a probability metrics, 
with widely applications, mostly in theoretical analysis, for proving central limit theorems, providing probability metric bounds, and recently for proving concentration inequalities \citep{barbour2005introduction, stein2005stein, chatterjee2007stein}.  
In this work, we demonstrate that Stein's idea can also be a profound tool to address problems related to high complex probabilistic models in modern machine learning applications. 
}

%\end{document}

\onecolumn
%\include{main_icml15}

%\twocolumn[
\icmltitle{Appendix for ``A Kernelized Stein Discrepancy for Goodness-of-fit  Tests"}

% unless you've provided the [accepted] option to the icml2015
%\icmlkeywords{boring formatting information, machine learning, ICML}

\vskip 0.3in
%]
%\onecolumn

%\cleardoublepage

%\usepackage{qiangstyle}
%\usepackage[toc]{appendix}
%\numberwithin{equation}{chapter}

%\input{newcommands.tex}

%\title{Appendix for ``A Kernelized Stein Discrepancy for Goodness of fit  tests and Model Evaluation"}
%\author{
%Qiang Liu  \\Dept. of Computer Science\\Univ. of California, Irvine\\\texttt{qliu1@uci.edu} \\\And
%}
% The \author macro works with any number of authors. There are two commands
% used to separate the names and addresses of multiple authors: \And and \AND.
%
% Using \And between authors leaves it to \LaTeX{} to determine where to break
% the lines. Using \AND forces a linebreak at that point. So, if \LaTeX{}
% puts 3 of 4 authors names on the first line, and the last on the second
% line, try using \AND instead of \And before the third author name.

%\newcommand{\fix}{\marginpar{FIX}}
%\newcommand{\new}{\marginpar{NEW}}

%\nipsfinalcopy % Uncomment for camera-ready version

%\begin{document}
%\include{main_icml15}
%\input{rkhs_stein_icml15}

\appendix
\numberwithin{equation}{section}

%\maketitle
%\renewcommand{\theequation}{\Alph{chapter}.\arabic{equation}}

\section{Proofs}

\begin{proof}[Proof of Theorem~\ref{thm:kxx}]
1) %is trivial. 
Denote by $\vv v(x,x') = k(x,x') \score_q(x') + \nabla_{x'} k(x,x') =  \stein_q k_x(x') $; applying Lemma~\ref{lem:basic2} on $k(x, \cdot)$ with fixed $x$, 
\begin{align*}
\S(p,q) 
& = \E_{x,x'\sim p} [(\score_q(x) - \score_p(x))^\top k(x,x') (\score_q(x') - \score_p(x')  )] \\
& = \E_{x,x'\sim p} [(\score_q(x) - \score_p(x))^\top \vv v(x,x')] 
\end{align*}
Because $k(\cdot,x')$ is in the Stein class of $p$ for any $x'$, we can show that $\nabla_{x'}k(\cdot, x')$ is also in the Stein class, since %for any fixed $x'$ and $x''$
$$
\int_x \nabla_x( p(x) \nabla_{x'} k(x, x'))  dx
= \nabla_{x'} \int_x \nabla_x( p(x) k(x, x'))  dx  = 0, 
%\stein_p \nabla_{x'}k(\cdot, x')   = 
%\stein_p \lim_{x'' \to x'}   \frac{k(\cdot, x'') -   k(\cdot, x')}{x'' - x} 
%=   \lim_{x'' \to x'}  \score_p(x) \frac{k(\cdot, x'') -   k(\cdot, x')}{x'' - x}  + \nabla_x \lim_{x'' \to x'}   \frac{k(\cdot, x'') -   k(\cdot, x')}{x'' - x'}   \\
% \lim_{x'' \to x'}  \frac{1}{x'' - x}  \score_p(x) \frac{k(\cdot, x'') -   k(\cdot, x')}{x'' - x}  + \nabla_x \lim_{x'' \to x'}   \frac{k(\cdot, x'') -   k(\cdot, x')}{x'' - x'}   
$$
and hence $\vv v(\cdot ,x')$ is also in the Stein class; 
%Denote by $v(x,x') =  \stein_q k_x(x') = k(x,x') \score_q(x') + \nabla_{x'} k(x,x')$. 
apply Lemma~\ref{lem:basic2} on $\vv v(\cdot, x')$ with fixed $x'$ gives
\begin{align*}
\S(p,q) 
& = \E_{x,x'\sim p} [(\score_q(x) - \score_p(x))^\top \vv v(x,x'))] \\
& = \E_{x,x'\sim p} [\score_q(x)^\top \vv v(x,x') + \trace( \nabla_x \vv v (x,x'))] 
\end{align*}
The result then follows by noting that $\nabla_x  \vv v(x,x')  =  \nabla_x k(x,x') \score_q(x')^\top + \nabla_{x'x'} k(x,x').$
\end{proof}

\begin{proof}[Proof of Theorem~\ref{thm:mercer22}]
%Because $\lambda_j  e_j(x) = \int k(x, x') e_j(x') dx',$the fact that $k(x,x')$ is smooth and belong to the Stein class of $p$ implies that $e_j$ is also smooth and belong to the Stein class $p$ for $\lambda_j \neq 0$. In addition, 
%$$ $k(x,x') = \sum_j \lambda_j  e_j(x) e_j(x')$,\footnote{because $k(x,x')$ is continuously differentiable, $e_j$ must also be so, since $e_j$ is in the RHKS of $k(x,x')$, and has a form of $e_j(x) = \sum_i a_i k(x,x_i)$; see \citet[][Corollary 4.36, p131]{steinwart2008support} and \citet[][Theorem 1]{zhou2008derivative}).}  
%we have
Note that 
\begin{align*}
\nabla_x k(x,x') = \sum_j \lambda_j   \nabla_x e_j(x) ~  e_j(x'), &&
\nabla_{x,x'} k(x,x') = \sum_j \lambda_j  \nabla_x e_j(x) ~ \nabla_{x'} e_j(x')^\top, 
\end{align*}
and hence 
\begin{align*}
&\!\!\!\!\!\!\!\! u_q(x,x')  \\
&~ =  \score_q(x)^\top k(x,x') \score_q(x') + \score_q(x)^\top \nabla_x'k(x,x') + \score_q(x')^\top \nabla_{x}k(x,x') + \trace(\nabla_{x,x'}k(x,x')\\
&~ =  \sum_j \lambda_j\big[\score_q(x)^\top e_j(x)  e_j(x') \score_q(x') +  \score_q(x)^\top  e_j(x) \nabla_{x'}e_j(x') 
+ \score_q(x')^\top \nabla_x e_j(x) e_j(x') +  \nabla_x e_j(x)^\top \nabla_{x'} e_j(x') \big] \\
%+ \\[-.5\baselineskip]
%& ~~~~~~~~~~~~~~~~~~~~~~~~~~~~~~~~~~~~~+ \score_q(x')^\top \nabla_x e_j(x) e_j(x') +  \nabla_x e_j(x)^\top \nabla_{x'} e_j(x') \big] \\
& ~=  \sum_j \lambda_j  \big [ \score_q(x) e_j(x) + \nabla_x e_j(x) \big]^\top  \big[\score_q(x') e_j(x') + \nabla_{x'} e_j(x')   \big]  \\
& ~ = \sum_j  \lambda_j [\stein_q e_j(x)]^\top [\stein_q e_j(x')]. 
\end{align*}
%\begin{align*}
%U_q(x,x') 
%&~ =  \score_q(x) k(x,x') \score_q(x')^\top + \score_q(x)\nabla_x'k(x,x')^\top + \score_q(x')  \nabla_{x}k(x,x')^\top + \nabla_{x,x'}k(x,x')\\
%&~ =  \sum_j \lambda_j\big[\score_q(x) e_j(x)  e_j(x') \score_q(x')^\top +  \score_q(x)  e_j(x) \nabla_{x'}e_j(x')^\top + \\
%& ~~~~~~~~~~~~~~~~~~~~~~~~~~~~~~~~~~~~~+ \score_q(x') \nabla_x e_j(x)^\top e_j(x') +  \nabla_x e_j(x) \nabla_{x'} e_j(x')^\top \big] \\
%& ~=  \sum_j \lambda_j  \big [ \score_q(x) e_j(x) + \nabla_x e_j(x) )  (\score_q(x') e_j(x') + \nabla_{x'} e_j(x') )^\top  \big]  \\
%& ~ = \sum_j  \lambda_j [\stein_q e_j(x)] [\stein_q e_j(x')]^\top. 
%\end{align*}
%and hence
Therefore, $u_q(x,x')$ is positive definite because $\lambda_j > 0$. 
In addition, 
%Therefore, 
\begin{align*}
\S(p,q)  
 & = \E_{x,x'}[u_q(x,x')]   \\
 &=  \sum_j  \lambda_j \E_x[\stein_q e_j(x)]^\top ~ \E_{x'}[\stein_q e_j(x')]    \\
& =   \sum_j  \lambda_j || \E_x[\stein_q e_j(x)] ||_2^2 .
% \lambda_j  \big(\E_{x} [ e_j(x)\score_q(x) + \nabla_x e_j(x)] \big)^2 \\
%& = \sum_{j} \lambda_j  || \E_{x} [ e_j(x)(\score_q(x) - \score_p(x)) ]  ||_2^2
\end{align*}
\fullversion{\red{
Similarly, we can derive $\S(p,q) = \sum_{j} \lambda_j  || \E_{x} [ e_j(x)(\score_q(x) - \score_p(x)) ]  ||_2^2$ from $\S(p,q) = \E_{x,x'}[(\score_q(x)- \score_p(x))^\top k(x,x') (\score_q(x') - \score_p(x'))]$, that is, 
\begin{align*}
\S(p,q)  
 & =  \E_{x,x'}[(\score_q(x)- \score_p(x))k(x,x') (\score_q(x') - \score_p(x'))] \\
  &=  \E_{x,x'}[\sum_j \lambda_j  (\score_q(x)- \score_p(x))^\top e_j(x) (\score_q(x') - \score_p(x') e_j(x'))] \\
    &= \sum_j  \lambda_j  \E_{x}[(\score_q(x)- \score_p(x))^\top e_j(x)]  ~ \E_{x'}[(\score_q(x') - \score_p(x') e_j(x'))]    \\   
& =   \sum_j  \lambda_j || \E_{x} [ e_j(x)(\score_q(x) - \score_p(x)) ]  ||_2^2. 
\end{align*}
}}
\end{proof}

\begin{proof}[Proof of Theorem~\ref{thm:rkhs}]
We first prove \eqref{equ:beta2} by applying the reproducing property $ k(x,x') = \la k(x,\cdot), k(x',\cdot)\ra_\H  $ on \eqref{equ:dpEdsq}: 
\begin{align*}
\S(p,q) 
& = \E_{x,x'\sim p}[(\score_q(x) - \score_p(x))^\top ~k(x,x')~ (\score_q(x') - \score_p(x'))] \\
& = \E_{x,x'\sim p}[(\score_q(x) - \score_p(x))^\top ~~\big\la k(x,\cdot),~ k(x,\cdot) \big\ra_\H ~ (\score_q(x') - \score_p(x'))] \\
%& = \la \E_{x \sim p}[(\score_q(x) - \score_p(x)) k(x,\cdot)]^\top,~  \E_{x'} (\score_q(x') - \score_p(x'))] \big\ra_\H \\
&= \sum_{\ell=1}^d \big\la \E_{x}[(\score_q^\ell(x) - \score_p^\ell(x)) k(x,\cdot)],~ \E_{x'}[k(x,\cdot)  (\score_q^\ell(x) - \score_p^\ell(x))] \big\ra_\H \\
& =  \sum_{\ell=1}^d \big\la  \vv \beta_\ell ,   \vv \beta_\ell \big \ra_{\H}  \\
%& =  \big\la \vv \beta ,  \vv \beta \big \ra_{\H^d}  \\
& =  ||\vv \beta ||_{\H^d}^2
%&= \E_{x,x'\sim p}[(\score_q(x) - \score_p(x))^\top \la k(x,\cdot), k(x', \cdot) \ra_\H (\score_q(x) - \score_p(x))] 
%&= \E_{x,x'\sim p}[\E_x (\score_q(x) - \score_p(x))^\top \la k(x,\cdot), k(x', \cdot) \ra_\H (\score_q(x) - \score_p(x))] 
\end{align*}
where we used the fact that $\vv \beta(x') = \E_{x\sim p}[\stein_q k_{x'}(x)] =  \E_{x\sim p} [ (\score_q (x) k(x, x')+  \nabla_{x} k(x, x')] 
= \E_{x}[(\score_q(x) - \score_p(x)) k(x,x')]$. 
In addition, 
\begin{align*}
\la \vv f ,  \vv \beta \ra_{\H^d}  
&= \sum_{\ell=1}^d  \la f_\ell,~  \E_{x\sim p} [ (\score_q^{\ell}(x) k(x, \cdot)+  \nabla_{x_\ell} k(x, \cdot)] \ra_\H  \\
&= \sum_{\ell=1}^d   \E_{x\sim p} [ (\score_q^{\ell}(x) \la f_\ell, k(x, \cdot)\ra_\H +  \la f_{\ell}, \nabla_{x_\ell} k(x, \cdot) \ra_\H ] \\
&= \sum_{\ell=1}^d   \E_{x\sim p} [ (\score_q^{\ell}(x) f_\ell(x)  +  \nabla_{x_\ell} f_\ell(x)  ]  \\
&=\E_{x\sim p} [\trace(\stein_{q} \vv f(x))], 
\end{align*}
where we used the fact that $\nabla_x f(x) = \la  f(\cdot), ~ \nabla_{x}k (x,\cdot) \ra_\H$; see \citep{zhou2008derivative, steinwart2008support}. 
The variational form \eqref{equ:good} then follows the fact that $||\vv \beta||_{\H^d} = \max_{\vv f\in \H^d}\big\{ \la \vv f, \vv \beta \ra_{\H^d}, ~~~.s.t. ~~ || \vv f||_{\H^d} \leq 1\big\}$. 

Finally, the $\vv \beta(\cdot ) = 
 \E_{x\sim p} [ (\score_q(x) k(x, \cdot)+  \nabla_{x} k(x, \cdot)] $ is in the Stein class of $p$ because 
 $k(x, \cdot)$ and $\nabla_x k(x,\cdot)$ are in the Stein class of $p$ for any fixed $x$ (see the proof of Theorem~\ref{thm:kxx}). % by the assumption. 
\end{proof}

%\red{[]}
%\begin{pro}label{pro:steinrkhsIS}
%If $k(x,x')$ is in the Stein class of $p$, so is any $f\in \H$.
%\end{pro}
\begin{proof}[Proof Proposition~\ref{pro:steinrkhsIS}]
For any $f\in \H$ with kernel $k(x,x')$, we have $f = \la f, ~ k (\cdot,x)\ra_\H$ and $\nabla_x f = \la f, ~ \nabla_x k(x,\cdot) \ra_\H$. Therefore,
\begin{align*}
\E_{x\sim p} [ \score_p(x) f(x) + \nabla_x f(x) ]  
&=  \E_{x\sim p} [     \score_p(x) \big \la f, ~ k(x,\cdot) \big\ra_\H +  \big\la f, ~ \nabla_x k(x,\cdot) \big \ra_\H ]    \\
&=\big \la f, ~ ~ \E_{x\sim p} [ \score_p(x)  k(x,\cdot)  +  \nabla_x k(x,\cdot)] \big \ra_\H \\
& = \big \la f, ~~ \E_{x\sim p} [\stein_p k_x(\cdot)] \ra_\H \\
& = 0,
\end{align*}
where the last step used the fact that $\E_{x\sim p} [\stein_p k_x(\cdot)]$ because $k_x(\cdot) = k(\cdot, x)$ is in the Stein class of $p$ for any fixed $x$. 
\end{proof}

\begin{proof}[Proof of Theorem~\ref{thm:uasym}]
Applying the standard asymptotic results of $U$-statistics in \citet[][Section 5.5]{serfling2009approximation}, we just need to check that $\sigma_u^2\neq 0$ when $p\neq q$ and $\sigma_u^2 = 0$ when $p =q$. 

We first note that we can show that  
$\E_{x' \sim p} [u_q(x,x')] =  \trace(\stein_q \vv \beta)$, %\sum_{\ell} ( \score_q^\ell(x)\vv \beta_\ell(x) + \nabla_{x_\ell} \vv \beta_\ell(x)) $, 
%$\E_{x' \sim p} [u_q(x,x')] = \sum_{\ell} ( \score_q^\ell(x)\vv \beta_\ell(x) + \nabla_{x_\ell} \vv \beta_\ell(x)) $, 
where $\vv \beta(x) = \E_{x' \sim p} [\stein_q k_x(x')]$ and is in the Stein class of $p$ (see the proof of Theorem~\ref{thm:kxx}). %\red{[XX]}
%$\E_{x'} [u_q(x,x')] = \score_q(x)^\top \vv g(x) + \trace(\nabla_x \vv g(x))$, where $g(x) = \E_{x'} [\stein_q k_x(x')]$. 
%
Therefore, when $p = q$, we have $\vv \beta(x) \equiv 0$ by Stein's identity, and hence $\sigma_u^2 = 0$. 
%$\E_{x'}[u_q(x,x')] = \E[x]
%$ for any fixed $x$, and hence $\sigma_u^2 = 0$.

Assume $\sigma_u^2 =0$ when $p\neq q$, we must have $\E_{x'\sim p} [u_q(x,x')] = c$, where $c$ is a constant. 
Therefore, 
$$
c = \E_{x\sim q} \big(\E_{x'\sim p} [u_q(x,x')] \big)  = \E_{x' \sim p} \big ( \E_{x\sim q} [u_q(x,x')] \big ). 
$$
Because we can show that $ \E_{x\sim q} [u_q(x,x')] = 0$ following the proof above for $p=q$, we must have $c = 0$, 
and hence 
$$
 \S(p,q)   = \E_{x\sim p}\big(\E_{x'\sim p} [u_q(x,x')] \big) = c = 0,
$$
 which contradicts with $p \neq q$. 
 
\myempty{
Assume $\sigma_u^2 =0$ when $p\neq q$, we must have $\E_{x'\sim p} [u_q(x,x')] = c$, where $c$ is a constant. 
Because $\score_q^\ell(x) = \nabla_{x_\ell}q(x) / q(x)$, we have
$$
\sum_{\ell} \nabla_{x_\ell} [q(x)\vv \beta_\ell(x)]  = q(x)  \sum_{\ell} ( \score_q^\ell(x)\vv \beta_\ell(x) + \nabla_{x_\ell} \vv \beta_\ell(x)) 
= q(x)  ~ \E_{x' \sim p} [u_q(x,x')] = c q(x), 
$$
%\E_{x'} [u_q(x,x')] = \sum_{\ell} ( \score_q^\ell(x)\vv g_\ell(x) + \nabla_{x_\ell} \vv g_\ell(x)) = 
%\frac{1}{q(x)}\sum_{\ell} \nabla_{x_\ell} [q(x)\vv g_\ell(x)] = c ,$$
% this is equivalent to 
%$$\E_{x'} [u_q(x,x')] = \sum_{\ell} ( \score_q^\ell(x)\vv g_\ell(x) + \nabla_{x_\ell} \vv g_\ell(x)) =  \frac{1}{q(x)}\sum_{\ell} \nabla_{x_\ell} [q(x)\vv g_\ell(x)] = c ,$$
 and therefore $\int \sum_{\ell} \nabla_{x_\ell} [q(x)\vv \beta_\ell(x)] dx = c \int q(x) dx = c$. Since $\vv \beta(x)$ is in the Stein class of $q$ by assumption, we should have $c = 0$ and hence $ \S(p,q)  = c =0$, which contradicts with $p \neq q$. 
 }
%note that $E_{x'}[u_q(x.x')] = (\score_q(x) -\score_p(x)) \E_{x'}[k(x,x')(s_q(x') - s_p(x'))]$. 
\end{proof}

%because the kernel $ \ind[ x = x']$ is not differentiable, and hence not be estimated using $U$-statistics as kernelized Stein discrepancy. 
%In additional, the variational form \eqref{equ:good} under RKHS does not hold for Fisher divergence neither, because kernel $\ind[ x = x']$ does not corresponding to a RKHS. 
%Instead, one can show that Fisher divergence can be viewed as a variational optimization in the union ball of %the intersection of $\mathcal L^2(p)$ and the Stein class of $q$. 
% $(\mathcal L^2(p))^d = \mathcal L^2(p) \times \cdots \times \mathcal L^2(p)$, 
% \begin{align}
% \label{equ:fff}
%\sqrt{\mathbb F(p,q)} 
%= \max_{\vv f \in (\mathcal L^2(p))^d } \bigg\{ \sum_{\ell = 1}^d  \E_p[f_{\ell}(x) (\score_q^\ell(x) - \score_p^\ell(x))] \notag \\
%~~~~ s.t.~~~ \E_p[||\vv f(x)||^2] \leq 1 \bigg\}. 
%\end{align}
 %equipped with norm, %$||\vv f ||_{(\mathcal L^2(p))^d} = \sum_{\ell\in [d]} \E_p []$
%under function set $\mathcal F = \{f \in  \mathcal L^2(p) \cap \mathcal{S}(q), ~~~ ||f||_{2,p} \leq 1\}$ space, that is, 
% Note that $\mathcal L^2(p)$ are larger than the Stein class and RKHS, including discontinuous, un-smooth functions.  
%Nevertheless, we can rewrite Fisher divergence into a maximum Stein discrepancy 
%in the intersection of $(\mathcal L^2(p))^d$ and the Stein class of $p$ if $\score_q - \score_p$ is in the Stein class of $p$. 

\begin{proof}[Proof of Theorem~\ref{thm:FS}]
\eqref{equ:SF} is  obtained by applying Cauchy-Schwarz inequality on \eqref{equ:dpEdsq}, 
\begin{align*}
\S(p,q) ^2 
&  = | \E_{xx'} [(\score_q(x) - \score_p(x))^\top k(x,x') (\score_q(x) - \score_p(x))] |^2  \\
&\leq   \E_{xx'} [ k(x,x')^2] \cdot  \ \E_{x,x'}  [  [(\score_q(x) - \score_p(x))^\top (\score_q(x') - \score_p(x'))]^2 ] \\  %k(x,x') (\score_q(x')- \score_p(x) \bigg]. 
&\leq   \E_{xx'} [ k(x,x')^2] \cdot  \ \E_{x,x'}  [  || \score_q(x) - \score_p(x)||_2^2 \cdot ||\score_q(x) - \score_p(x)||_2^2 ]\\  %k(x,x') (\score_q(x')- \score_p(x) \bigg]. 
&= \E_{xx'} [ k(x,x')^2] \cdot  \mathbb F(p,q)^2. 
\end{align*}

To prove \eqref{equ:FS}, we simply note that \eqref{equ:good} is equivalent to 
$$
\sqrt{\S(p,q)} = \max_{\vv f \in \H^d} \bigg\{  \E_p[ (\score_q(x) - \score_p(x))^\top  \vv f(x) ] ~~~~ s.t.~~~ ||\vv f||_{\H^d} \leq 1  \bigg\}.
%\sqrt{\S(p,q)} = \max_{\vv f_\ell \in \H} \bigg\{ \sum_{\ell = 1}^d  \E_p[\vv f_{\ell}(x) (\score_q^\ell(x) - \score_p^\ell(x))] ~~~~ s.t.~~~ ||\vv f_{\ell}||_{\H} \leq 1, \forall \ell \in [d] \bigg\}.
$$
Taking $\vv f  = (\score_q - \score_p)/|| \score_q(x) - \score_p(x) ||_{\H^d}$ then gives \eqref{equ:FS}.  
\end{proof}

\begin{pro}
\label{pro:FisherVar}
Let $\mathcal F(p) = \mathcal L^2(p)  \cap \mathcal S(p)$, where $ \mathcal S(p)$ represents the Stein class of $p$, then we have 
\begin{align*}
 \sqrt{\mathbb F(p,q)} \geq  \max_{\vv f \in \mathcal F(p)^d} \bigg\{ \E_{p}[\trace(\stein_q \vv f(x))]  ~~~~~~ s.t.~~~ \E_p[ || \vv f(x)||^2_2 ] \leq 1 \bigg\}.  
% \sqrt{\mathbb F(p,q)} \geq  & \max_{\vv f \in \mathcal F(p)^d} \bigg\{ \E_{p}[\trace(\stein_q \vv f(x))] \\& ~~~~~~~~~~~~~~~~~~~~~~~~~~~~~~~~~ s.t.~~~ \E_p[ || \vv f(x)||^2_2 ] \leq 1 \bigg\}.  
\end{align*}
and the equality holds when $\score_q - \score_p \in \mathcal F(p)^d$. 
\end{pro}
Note that $\mathcal L^2(p)$ is larger than the Stein class and RKHS, and includes discontinuous, non-smooth functions, and hence we need to ensure $\vv f$ is in the Stein class explicitly. 
\begin{proof}%[Proof of Proposition~\ref{pro:FisherVar}]
Denote by  $(\mathcal L^2(p))^d = \mathcal L^2(p) \times \cdots \times \mathcal L^2(p)$,  note that by the definition of $\mathbb F(p,q)$, we have
 \begin{align}
 \label{equ:fff}
\sqrt{\mathbb F(p,q)} 
= \max_{\vv f \in (\mathcal L^2(p))^d } \bigg\{ \sum_{\ell = 1}^d  \E_p[f_{\ell}(x) (\score_q^\ell(x) - \score_p^\ell(x))] 
~~~~ s.t.~~~ \E_p[||\vv f(x)||_2^2] \leq 1 \bigg\}. 
\end{align}
%Note that \eqref{equ:fff} 
Restricting the maximizing to $\mathcal F(p)^d$ and applying Lemma~\ref{lem:basic2}  would give the result. 
\end{proof}
%we can expect that Fisher divergence dominants kernelized Stein discrepancy: 

%\bibliographystyle{unsrtnat}
%\bibliography{bibrkhs_stein}
%\clearpage \newpage
%\appendix 
%\end{document}

\end{document}